%% file: iclr2025_conference.tex
\documentclass{article} % For LaTeX2e
\date{}
\input{math_commands.tex}

\include{macros_arXiv}

\usepackage[square,numbers]{natbib}
\setcitestyle{authoryear,open={(},close={)}} %Citation-related commands

\usepackage[margin = 1.3 in]{geometry}

\usepackage[colorinlistoftodos,bordercolor=orange,backgroundcolor=orange!20,linecolor=orange,textsize=scriptsize]{todonotes}

\usepackage{colortbl}

\usepackage{hyperref}
\usepackage{url}

\title{Error Feedback under $(L_0,L_1)$-Smoothness: \\ Normalization and Momentum}

\usepackage{authblk}

\author[1]{Sarit Khirirat}
\author[1]{Abdurakhmon Sadiev}
\author[1]{Artem Riabinin}
\author[2]{Eduard Gorbunov}
\author[1]{Peter Richt{\'a}rik}

\affil[1]{King Abdullah University of Science and Technology (KAUST)}
\affil[2]{Mohamed bin Zayed University of Artificial Intelligence (MBZUAI)}

\begin{document}

\maketitle

\begin{abstract}
We provide the first proof of convergence for normalized error feedback algorithms across a wide range of machine learning problems. 
Despite their popularity and efficiency in training deep neural networks, traditional analyses of error feedback algorithms rely on the smoothness assumption that does not capture the properties of objective functions in these  problems. 
Rather, these problems have recently been shown to satisfy generalized smoothness assumptions, and the theoretical understanding of error feedback algorithms under these assumptions remains largely unexplored. 
Moreover, to the best of our knowledge, all existing analyses under generalized smoothness either i) focus on single-node settings or ii) make unrealistically strong assumptions for distributed settings, such as requiring data heterogeneity, and almost surely bounded stochastic gradient noise variance. 
In this paper, we propose distributed error feedback algorithms that utilize normalization to achieve the $\cO(1/\sqrt{K})$ convergence rate for nonconvex problems under generalized smoothness. Our analyses apply for distributed settings without data heterogeneity conditions, and enable stepsize tuning that is independent of problem parameters. 
Additionally, we provide strong convergence guarantees of normalized error feedback algorithms for stochastic settings. 
Finally, we show that due to their larger allowable stepsizes, our new normalized error feedback algorithms outperform their non-normalized counterparts on various tasks, including the minimization of polynomial functions, logistic regression, and ResNet-20 training. 
%confirm that \algname{||EF21||}  achieves stronger convergence performance than EF21, due to its larger allowable stepsize ranges.
\end{abstract}

%\sarit{Check citep and citet.}

\tableofcontents

\section{Introduction}
\label{sec:intro}

%\sarit{TODO: Change terminology from \textcolor{red}{single-node}, distributed to single-node, multi-node...}
%\sarit{TODO: Change terminology from \textcolor{red}{generalized smoothness} to generalized smoothness...}

%\sarit{Remove the paragraph titles later.}
%\sarit{Adjust the format for theorems and lemmas.}

%\paragraph{Distributed training.}
Machine learning models achieve impressive prediction and classification power by employing sophisticated architectures, comprising vast numbers of model parameters, and requiring training on massive datasets.
%Machine learning has been used in a wide range of applications, such as the internet-of-things (IoT), health care, and networked autonomous systems. 
%Due to the continuously increasing sizes of modern models and training data,  \textcolor{red}{single-node} training is often infeasible. 
Distributed training has emerged as an important approach, where multiple machines with their own local training data collaborate to train a model efficiently within a reasonable time.
%, and under their available computational power, storage, and communication capacities.
%Instances of distributed training that have gained significant attention include data parallelism and federated learning. 
Many optimization algorithms can be easily adapted for distributed training frameworks. For example, stochastic gradient descent (\algname{SGD}) can be modified into distributed stochastic gradient descent within a data parallelism framework, and into federated averaging algorithms~\citep{mcmahan2017communication} in a federated learning framework.
%However, the communication cost of running these distributed algorithms is a major bottleneck that limits scalability to  large models. 
%For instance, distributed stochastic gradient descent for training the VGG-16 model requires the communication of $138.34$ million parameters, thus taking up more than $500$MB storage space and creating an unmanageable burden on the communication network of machines.  
However, the communication overhead of running these distributed algorithms poses a significant barrier to scaling up to large models. For example, training the VGG-16 model~\citep{simonyan2014very} using distributed stochastic gradient descent involves communicating $138.34$ million parameters, thus consuming over $500$MB of storage and posing an unmanageable burden on the communication network between machines.

%\begin{itemize}
%    \item The simplest gradient-based algorithm under data parallelism is distributed gradient descent, where each node in parallel computes and communicates its local gradient while the central server aggregates the local gradient and update the solution using the aggregated local gradients. 
%    \item Distributed gradient descent poses a significant challenge to the communication network. 
%    \item This method transmits dense gradient of $d$ dimensions to the master. The communication load is unacceptable in many practical scenarios. 
%    \item The communication bottleneck is a main factor that limits the scalability of distributed training.
%\end{itemize}

%\paragraph{Error feedback compression.}
One approach to mitigate the communication burden is to apply compression.
In this approach, the information, such as gradients or model parameters, is compressed using sparsifiers or quantizers to be transmitted with much lower communicated bits between machines. 
However, while this reduces communication overhead, too coarse compression often brings substantial challenges in maintaining high training performance due to information loss, and in extreme cases, it may potentially lead to divergence. 
Therefore, error feedback mechanisms have been developed to improve the convergence performance of compression algorithms, while ensuring high communication efficiency. 
Examples of error feedback mechanisms include \algname{EF14}~\citep{seide20141,stich2018sparsified,alistarh2018convergence,wu2018error,gorbunov2020linearly}, \algname{EF21}~\citep{richtarik2021ef21,fatkhullin2021ef21}, \algname{EF21-SGDM} ~\citep{fatkhullin2024momentum}, \algname{EF21-P}~\citep{gruntkowska2023ef21}, and \algname{EControl}~\citep{gao2023econtrol}. 
Several studies developing error feedback algorithms often assume  the smoothness of an objective function, i.e., 
%the objective function is smooth, i.e. 
%the smoonthess condition on the objective function, 
%While optimization algorithms with error feedback compression have been extensively studied, these algorithms assume the smoothness condition on the objective function, 
%i.e. 
its gradient is Lipschitz continuous.

%\begin{itemize}
%    \item To resolve this communication bottleneck, one approach is to apply compression on gradients.
 %   \item Error feedback mechanisms have been extensively studied under various settings, from stochastic settings to distributed and federated settings.
 %   \item Originally developed by~\citet{seide20141}, error feedback mechanisms include \algname{EF14}, \algname{EF21}, EF21-P, and EControl. 
 %   \item However, most communication-efficient learning methods assume smoothness of the objective function, i.e. the Lipschitz continuity of its gradient. 
%\end{itemize}

%\paragraph{Non-smooth regimes for modern learning models.}
%Although the smoothness condition is often imposed to analyze the convergence of optimization algorithms, 
However, many modern learning problems, such as distributionally robust optimization~\citep{jin2021non} and deep neural network training, are often non-smooth. 
For instance, the gradient of the loss computed for deep neural networks, such as LSTM~\citep{zhang2019gradient}, ResNet20~\citep{zhang2019gradient}, and transformer models~\citep{crawshaw2022robustness}, is not Lipschitz continuous.
These empirical findings highlight the need for a new smoothness assumption.
One such assumption is $(L_0,L_1)$-smoothness, originally introduced by~\citet{zhang2019gradient}, for twice differentiable functions, and later extended to differentiable functions by~\citet{chen2023generalized}.

%\paragraph{First-order algorithms under $(L_0,L_1)$-smoothness.}
To solve generalized smooth problems, clipping and normalization have been widely utilized in first-order algorithms. 
Gradient descent with gradient clipping was initially shown by~\citet{zhang2019gradient} to achieve lower iteration complexity, i.e., fewer iterations needed to attain a target solution accuracy,  than classical gradient descent.
Subsequent works have further refined the convergence theory of clipped gradient descent~\citep{koloskova2023revisiting}, and improved its convergence performance by employing momentum updates~\citep{zhang2020improved}, variance reduction techniques~\citep{reisizadeh2023variance}, and adaptive step sizes~\citep{wang2024provable,li2024convergence,takezawa2024polyak}.
Similar convergence results have been obtained for gradient descent using normalization~\citep{zhao2021convergence}, and its momentum variants~\citep{hubler2024parameter}, including generalized \algname{SignSGD}~\citep{crawshaw2022robustness}.
However, these first-order algorithms have mostly been explored in training on a single machine. 
To the best of our knowledge, distributed algorithms under generalized smoothness have been investigated in only a few works, e.g., by \citet{crawshaw2024federated,liu2022communication}.
Nonetheless, these works rely on assumptions limiting families of optimization problems, including data heterogeneity, almost sure variance bounds, and symmetric noise distributions around the mean assumptions. 
Furthermore, these first-order algorithms under generalized smoothness do not incorporate compression techniques to improve communication efficiency. 
These aspects motivate us to develop \emph{distributed communication-efficient algorithms for solving nonconvex generalized smooth problems}.

\subsection{Contributions}

In this paper, we  develop distributed error feedback algorithms for  communication-efficient optimization under nonconvex, generalized smooth regimes.  Our contributions are summarized below.

\paragraph{$\bullet$ Importance of normalization.}
Just as gradient clipping is crucial for gradient descent, we empirically demonstrate that normalization stabilizes the convergence of error feedback algorithms for minimizing nonconvex generalized smooth functions. In this paper, we introduce a variant of \algname{EF21}, a widely used error feedback algorithm by \citet{richtarik2021ef21}, which incorporates normalization to guarantee convergence for nonconvex, generalized smooth problems. 
In a single-node setting, this new method, which we call \algname{||EF21-GD||}, or more compactly as \algname{||EF21||}, provides larger stepsize, and faster convergence rate than its non-normalized counterpart \algname{EF21} for minimizing simple nonconvex polynomial functions that satisfy generalized smoothness, as shown by Figure~\ref{fig:exp1}.   

%while the original \algname{EF21} may diverge, \algname{||EF21||} still converges when minimizing simple nonconvex polynomial functions that satisfy \textcolor{red}{generalized smoothness} conditions (Assumption~\ref{assum:LzeroLoneSmooth}), as illustrated in Figure~\ref{fig:plot_divergenceEF21}.

%\eduard{TODO: once the proofs are generalized to the symmetric case, please remove the word ``asymmetrically'' from the above.}

%\eduard{I believe one can prove convergence of \algname{EF21} (at least for deterministic compressors such as TopK) if the stepsize is small enough (following the same proof as in Clip21; see also \citep{li2024convex}). Therefore, the statement "EF21 may diverge" and Figure~\ref{fig:plot_divergenceEF21} are not very accurate: \algname{EF21} may diverge if the stepsize is too large, but it converges with small enough stepsize (though, presumably much slower than \algname{||EF21||}).}

%\sarit{I will rewrite this text according to Figure 5. By stating that  \algname{||EF21||} allows larger stepsizes, and faster convergence than \algname{EF21}.}

\begin{figure}[h]
    \centering
 %   \begin{subfigure}{0.3\textwidth}
        \includegraphics[width=0.3\textwidth]{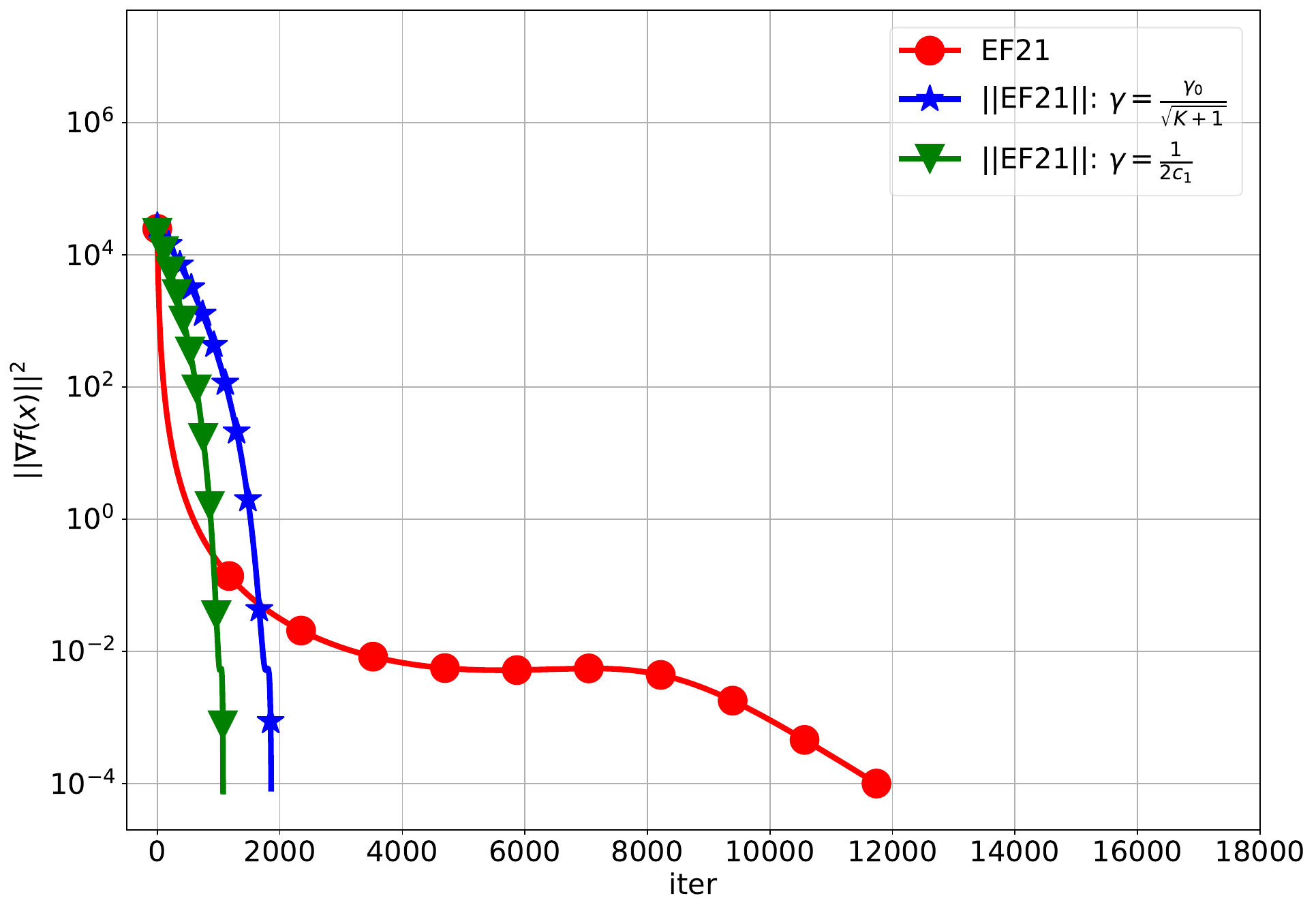}
 %       \caption{$L_0 = 4$, $L_1 = 1$, $K=2000$}
%    \end{subfigure}
    \hfill
 %   \begin{subfigure}{0.3\textwidth}
        \includegraphics[width=0.3\textwidth]{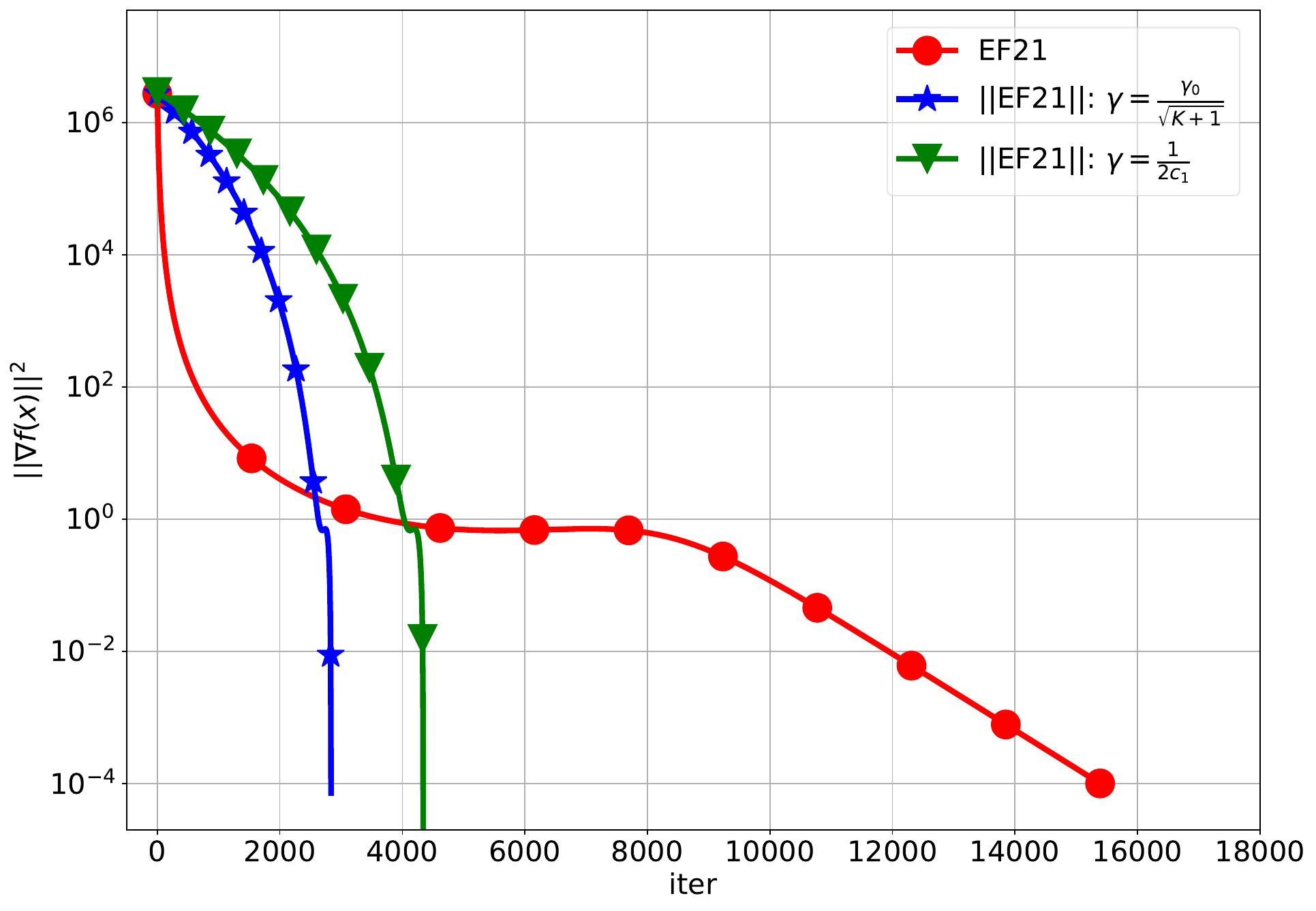}
 %       \caption{$L_0 = 4$, $L_1 = 4$, $K=5000$}
 %   \end{subfigure}
    \hfill
  %  \begin{subfigure}{0.3\textwidth}
        \includegraphics[width=0.3\textwidth]{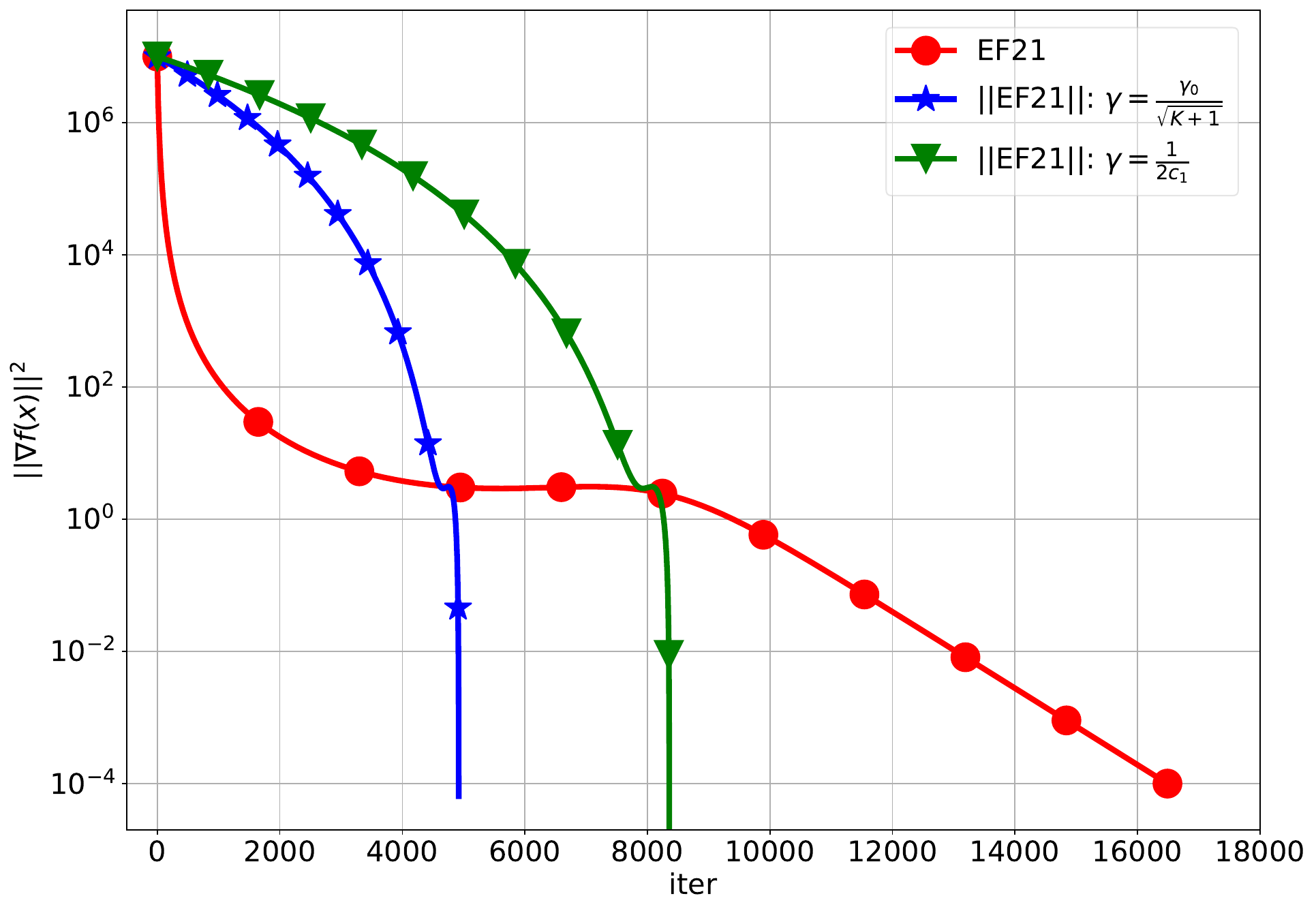}
   %     \caption{$L_0 = 4$, $L_1 = 8$, $K=16000$}
   % \end{subfigure}
    \caption{The minimization of polynomial functions using \algname{EF21} with $\gamma = \frac{1}{L + L \sqrt{\frac{\beta}{\theta}}}$, and \algname{||EF21||}  with $\gamma = \frac{\gamma_0}{\sqrt{K+1}}$, $\gamma_0 = 1$ (blue line) and $\gamma = \frac{1}{2c_1}$ (green line). Here, we ran both algorithms for (1) $L_0 = 4$, $L_1 = 1$, and $K=2,000$ (left), (2) $L_0 = 4$, $L_1 = 4$, and $K=5,000$ (middle), and (3) $L_0 = 4$, $L_1 = 8$, and $K=16,000$ (right). }
    \label{fig:exp1}
\end{figure}

% \begin{figure}[h]
%     \centering
%     \includegraphics[width=0.4\textwidth]{plot_divergenceEF21.pdf}
%     \caption{The minimization of nonconvex polynomial functions in \eqref{eqn:simple_polynomials} using \algname{EF21} and \algname{||EF21||} . Here, the functions satisfy \textcolor{red}{generalized smoothness} with $L_0=4$ and $L_1=4$, and we set $K=5000$. For \algname{EF21},  $\gamma = \frac{1}{L_0 + L_0 \sqrt{\frac{\beta}{\theta}}}$ according to Theorem 1 of \citet{richtarik2021ef21} (red line). For \algname{||EF21||}, we choose $\gamma = \frac{\gamma_0}{\sqrt{K+1}}$ according to Theorem~\ref{thm:ef21} with $\gamma_0=1$ (blue line), and also $\gamma = \frac{1}{2c_1}$ (green line). 
%     \eduard{It is not a completely fair experiment: Theorem 1 of \citet{richtarik2021ef21} requires smoothness of the objective and does not imply convergence for $\gamma = \frac{1}{L_0 + L_0 \sqrt{\frac{\beta}{\theta}}}$. The problem from \eqref{eqn:simple_polynomials} is actually $L$-smooth on any compact set (but the smoothness constant depends on the diameter of the set). Therefore, it would be more fair if we compute the theoretical $L$ and then run \algname{EF21} with such $L$.}}
%     \sarit{Ok. I will replace Figure 1 with Figure 5, and rewrite the statement in the first bullet.}
%     \label{fig:plot_divergenceEF21}
% \end{figure}

\paragraph{$\bullet$ Convergence of normalized error feedback algorithms.}
We establish an $\mathcal{O}(1/\sqrt{K})$ convergence rate in the gradient norm for \algname{||EF21||} on nonconvex generalized smooth problems. 
\algname{||EF21||} achieves the same  rate as  \algname{EF21} under $L$-smoothness by \cite{richtarik2021ef21}. Our results are derived under standard assumptions, i.e., generalized smoothness and the existence of lower bounds on the objective function, and are applicable in distributed settings regardless of any data heterogeneity degree, unlike the results by \citet{crawshaw2024federated,liu2022communication}. Additionally, our stepsize rules for \algname{||EF21||} ensure convergence without requiring knowledge of the generalized smoothness constants $L_0$ or $L_1$, in contrast to \citet{richtarik2021ef21}, where the stepsize depends on the smoothness constant $L$ (which is often inaccessible).

%\eduard{TODO: once the proofs are generalized to the symmetric case, please remove the word ``asymmetrically'' from the above.}

\paragraph{$\bullet$ Extension to stochastic settings.}
Furthermore, we propose a variant of \algname{EF21-SGDM}, an error feedback algorithm with momentum updates by~\citet{fatkhullin2024momentum}, that employs normalization for solving nonconvex, stochastic optimization under generalized smoothness. 
Specifically, we prove that \algname{||EF21-SGDM||}  with suitable stepsize choices attains the same $\cO(1/K^{1/4})$ convergence rate in the gradient norm as  \algname{EF21-SGDM}.

\paragraph{$\bullet$ Numerical evaluation.} 
We implemented \algname{||EF21||} using the stepsize rules derived from our theory, and compared its performance against  \algname{EF21}. Both algorithms were evaluated on three learning tasks: minimizing nonconvex polynomial functions, solving logistic regression with a nonconvex regularizer, and training ResNet-20 on the CIFAR-10 dataset. Thanks to its larger stepsizes, \algname{||EF21||} outperforms \algname{EF21}, in terms of both convergence speed and solution accuracy across these tasks.

\definecolor{LightCyan}{rgb}{0.9,0.9,0.9}
\begin{table}[h]
\begin{center}
\resizebox{\textwidth}{!}{%
\begin{tabular}{ccccc}
\bf Methods & \bf Complexity & \bf Smoothness & \bf Variance bound & \bf Normalization \\ \hline 
\begin{tabular}{c} \algname{EF21} \\ \citet{richtarik2021ef21} \end{tabular} & $\cO(1/\epsilon^2)$ & $L$ & No & No \\ 
\begin{tabular}{c} \algname{EF21-SGDM}  \\ \citet{fatkhullin2024momentum} \end{tabular} &  $\cO( {1}/{\epsilon^4})$  & $L$ & expectation & No \\
 \rowcolor{LightCyan} \begin{tabular}{c} \algname{||EF21||} \\  {\bf NEW} (Alg.~\ref{alg:normalized_ef21}) \end{tabular}
&  $\cO(1/\epsilon^2)$ & $(L_0,L_1)$ & No & Yes \\
 \rowcolor{LightCyan} \begin{tabular}{c} \algname{||EF21-SGDM||}  \\ {\bf NEW} (Alg.~\ref{alg:normalized_ef21_sgdm}) \end{tabular} & $\cO(1/\epsilon^4)$ & $(L_0,L_1)$ & Expectation & Yes \\ \hline 
\end{tabular}
}
\end{center}

\caption{Comparisons of complexities and assumptions between known and our results for \algname{EF21} variants. The  complexity is defined by the iteration count $K$ required by the algorithms to attain $\underset{k=0,1,\ldots,K}{\min}\Exp{\norm{\nabla f(x^k)}} \leq \epsilon$. $(L_0,L_1)$-smoothness refers to generalized smoothness in Assumption~\ref{assum:LzeroLoneSmooth}. The variance bound in expectation is defined in Assumption~\ref{assum:bounded_variance}.}
\label{sample-table}
\end{table}

\section{Related Works}
\label{sec:related_work}

\paragraph{Error feedback.}
Error feedback mechanisms have been utilized in various algorithms with communication compression, leading to significant improvements in solution accuracy, while reducing communication. 
As the first version of these mechanisms, \algname{EF14} was introduced by~\citet{seide20141}, and later analyzed for first-order algorithms in both single-node \citep{stich2018sparsified,karimireddy2019error} and distributed settings \citep{alistarh2018convergence,wu2018error,tang2019doublesqueeze,basu2019qsparse,gorbunov2020linearly,li2020acceleration,qian2021error,tang2021errorcompensatedx}. 
Next, \algname{EF21} is another error feedback variant proposed by \citet{richtarik2021ef21}, which offers strong convergence guarantees for distributed gradient algorithms with any contractive compressors, without requiring bounded gradient norm or bounded data heterogeneity assumptions. \algname{EF21} can also be adapted for  stochastic optimization through sufficiently large mini-batches \citep{fatkhullin2021ef21} or momentum updates \citep{fatkhullin2024momentum}. More recently, \algname{EControl} was developed by \citet{gao2023econtrol}  to guarantee provably superior complexity results for distributed stochastic optimization compared to prior error feedback mechanisms. %\citep{fatkhullin2024momentum}.
To the best of our knowledge, these existing works on error feedback have focused solely on optimization under traditional \( L \)-smoothness. In this paper, we introduce a normalized variant of the \algname{EF21} methods \citep{richtarik2021ef21} for solving nonconvex generalized smooth  problems. 
In particular, we prove that  \algname{||EF21||} under generalized smoothness achieves the same $\cO(1/\sqrt{K})$ rate as \algname{EF21} under traditional smoothness, and demonstrate in experiments that \algname{||EF21||} permits larger step sizes, and thus attains faster convergence than  \algname{EF21}.

%We demonstrate that to solve these problems, the \algname{||EF21||} method permits larger step sizes, thus leading to faster convergence rates than the original EF21.

%\begin{itemize}
%    \item I can look at \algname{EControl} papers to see the structure of what to talk about. 
%    \item \algname{EF14} was first analyzed for \textcolor{red}{single-node} training  in ... 
%    \item Extensions to distributed training were first conducted under the assumption of homogeneous (IID) data. 
%    \item \algname{EF21} -> EF21-SGD (large minibatch sizes) -> \algname{EF21-SGDM}  (minibatch size can be one). 
%    \item EF21-P 
%    \item \algname{EControl} that outperforms EF21, \algname{EF21-SGDM} . 
%    \item Despite various settings, they are limited to the traditional smoothness condition, which is not useful for deep neural network training. The stepsize usually requires us to know the Lipschitz constant of the gradients, which is hard to obtain. 
%    \item Unlike existing literature, we aim to prove the iteration complexity of communication-efficient optimization algorithms under non-uniform smoothnes, where the stepsize does not require the knowledge of smoothness parameters.  
%\end{itemize}

\paragraph{Non-smoothness assumptions.}
Empirical findings suggest that the traditional smoothness used for analyzing optimization algorithms does not capture the properties of objective functions in many machine learning problems, especially deep neural network training problems. 
This motivates researchers to consider different assumptions to replace this traditional smoothness condition. 
First introduced by~\citet{zhang2019gradient}, the $(L_0,L_1)$-smoothness condition on a twice differentiable function $f(x)$ is defined by $\norm{ \nabla^2 f(x)} \leq L_0+L_1 \norm{\nabla f(x)}$ for $x\in\R^d$. This $(L_0,L_1)$-smoothness has been extended to
 differentiable functions without assuming the existence of the Hessian.  
For instance, the smoothness with a differentiable function $\ell(x)$~\citep{li2024convex}, and symmetric generalized smoothness~\citep{chen2023generalized} cover the $(L_0,L_1)$-smoothness when the Hessian exists, and includes many important machine learning problems, such as phase retrieval problems~\citep{chen2023generalized}, and distributionally robust optimization~\citep{levy2020large}. Other classes of non-smoothness assumptions, which are not related to the generalized smoothness but capture other optimization problems, include H{\"o}lder's continuity of the gradient \citep{devolder2014first},  the relative smoothness~\citep{bauschke2017descent}, and the polynomial growth of the gradient norm~\citep{mai2021stability}.
In this paper, we impose the generalized smoothness condition to establish the convergence of \algname{||EF21||} for solving deterministic and stochastic optimization.

%\eduard{TODO: once the proofs are generalized to the symmetric case, please remove the word ``asymmetric'' from the above.}

\paragraph{Gradient clipping and normalization.} Clipping and normalization are commonly employed in gradient-based methods for solving generalized smooth problems. Clipped (stochastic) gradient descent has been studied for both nonconvex and convex problems under  $(L_0,L_1)$-smoothness conditions by \citet{zhang2019gradient,koloskova2023revisiting}. Extensions to clipped gradient algorithms have been proposed, including momentum updates \citep{zhang2020improved}, variance reduction methods \citep{reisizadeh2023variance}, and adaptive step sizes \citep{wang2024provable,li2024convergence,takezawa2024polyak,gorbunov2024methods}. Comparable complexities have been achieved for normalized gradient descent \citep{zhao2021convergence}, and its momentum-based variants \citep{hubler2024parameter}, including generalized \algname{SignSGD} \citep{crawshaw2022robustness}.
Convergence properties of gradient-based algorithms have also been explored under more generalized forms of non-uniform smoothness, extending beyond the $(L_0,L_1)$-smoothness by \citet{zhang2019gradient} to cover a wider range of optimization problems. For example, variants of (stochastic) gradient descent have been analyzed under $\alpha$-symmetric generalized smoothness by \citet{chen2023generalized}, and under $\ell$-smoothness involving certain differentiable functions $\ell(\cdot)$ by \citet{li2024convex,li2024convergence}. 
However, the majority of these analyses focus on the single-node setting. To the best of our knowledge, only a limited number of works, such as those by \citet{crawshaw2024federated,liu2022communication}, have examined federated averaging algorithms for nonconvex problems under generalized smoothness. These works, however, often rely on restrictive assumptions, including data heterogeneity, almost sure variance bounds, and symmetric noise distributions centered around their means.
In this paper, we develop distributed error feedback algorithms, which eliminate the need for the restrictive assumptions mentioned above, and rely on standard assumptions on objective functions and compressors.

\section{Preliminaries}
\paragraph{Notations.}
We use \([n]\) to denote the set \(\{1, 2, \ldots, n\}\), and \(\Exp{u}\) to represent the expectation of a random variable \(u\). Additionally, \(\norm{\cdot}\) indicates the Euclidean norm for vectors or the spectral norm for matrices, and \(\norm{\cdot}_1\) is the $\ell_1$-norm for vectors, while \(\langle x, y \rangle\) denotes the inner product between \(x\) and \(y\) in \(\mathbb{R}^d\). Lastly, for a square matrix \(A \in \mathbb{R}^{d \times d}\), \(\lambda_{\min}(A)\) refers to its minimum eigenvalue, and \(I \in \mathbb{R}^{d \times d}\) is the identity matrix.

\paragraph{Problem formulation.} In this paper, we focus on the following distributed optimization problem: 
\begin{eqnarray}\label{eqn:Problem}
 \underset{x\in\R^d}{\text{min}} \ \left\{f(x) \eqdef \frac{1}{n}\sum_{i=1}^n f_i(x)\right\},
\end{eqnarray}
where $n$ refers to the number of clients, and $f_i(x)$ is the loss of a model parameterized by vector $x\in \R^d$ over its local data $\cD_i$ owned by client $i\in [n]$.

\paragraph{Assumptions.} To facilitate our convergence analysis, we make standard assumptions on objective functions and compression operators.
%Specifically, Assumptions~\ref{assum:lowerbound_whole_f} and~\ref{assum:lowerbound_f_i} establish lower bounds for $f(x)$ and each $f_i(x)$, respectively, Assumption~\ref{assum:LzeroLoneSmooth} defines the \textcolor{red}{generalized smoothness} of $f_i(x)$, and  Assumption~\ref{assum:contractive_comp} enforces the $\alpha$-contractive property of the compressors.

\begin{assumption}[Lower Boundedness of $f$]\label{assum:lowerbound_whole_f}
The function $f$ is bounded from below, i.e., $$f^{\inf} = \inf_{x\in\R^d} f(x)>-\infty.$$
\end{assumption}

\begin{assumption}[Lower Boundedness of $f_i$]\label{assum:lowerbound_f_i}
For each $i\in [n]$, the function $f_i$ is bounded from below, i.e., $$f^{\inf}_{i} \eqdef \inf_{x\in\R^d} f_i(x)>-\infty.$$
\end{assumption}

Assumptions~\ref{assum:lowerbound_whole_f} and~\ref{assum:lowerbound_f_i} are standard for analyzing optimization algorithms for unconstrained optimization.

%\sarit{I work for Assumption 3.}

%\eduard{TODO: replace the assumption below with symmetric version}

%\sarit{The point $u_\theta = \theta x + (1-\theta)y$. This is from Definition 3  of \citet{chen2023generalized} with $\alpha=1$. }

\begin{assumption}[Generalized Smoothness of $f_i$]\label{assum:LzeroLoneSmooth}
A function $f_i(x)$ is symmetrically generalized smooth if there exists $L_0,L_1>0$ such that for $u_\theta = \theta x + (1-\theta)y$, and for all $x,y\in\R^d$,
\begin{eqnarray}\label{eqn:LzeroLoneSmooth}
	\norm{ \nabla f_i(x) - \nabla f_i(y) } \leq \left( L_0+L_1 \sup_{\theta\in[0,1]}\norm{\nabla f_i(u_\theta)} \right) \norm{x-y}.
\end{eqnarray}	
\end{assumption}

Assumption~\ref{assum:LzeroLoneSmooth} refers to symmetric generalized smoothness defined by~\citet{chen2023generalized}, which covers asymmetric generalized smoothness~\citep{koloskova2023revisiting,chen2023generalized}, and the original $(L_0,L_1)$-smoothness by~\cite{zhang2019gradient}.
%generalizes the traditional smoothness condition, and is equivalent to the original $(L_0,L_1)$-smoothness proposed by~\cite{zhang2019gradient}, when the Hessian exists \citep{chen2023generalized}.
Moreover, Assumption~\ref{assum:LzeroLoneSmooth} covers the functions with unbounded classical smoothness constant, e.g., exponential function. 
Additionally, Assumption~\ref{assum:LzeroLoneSmooth} with $L_1=0$ reduces to the traditional $L_0$-smoothness~\citep{nesterov2018lectures,beck2017first}, under which the convergence of optimization algorithms has been extensively studied.

\begin{assumption}[Contractive Compressor]\label{assum:contractive_comp}
  An operator $\cC^k:\R^d\rightarrow\R^d$ is an $\alpha$-contractive compressor if there exists $\alpha\in(0,1]$ such that for $k \geq 0$ and $v\in\R^d$,
\begin{eqnarray} \label{eqn:contractive_comp}
	\Exp{\sqnorm{\cC^k(v)-v}} \leq (1-\alpha) \sqnorm{v}.
\end{eqnarray}	
\end{assumption}

Furthermore, compressors defined by Assumption~\ref{assum:contractive_comp} cover top-$k$ sparsifiers~\citep{alistarh2018convergence,stich2018sparsified}, low-rank approximation~\citep{vogels2019powersgd,safaryan2021fednl}, and various other compressors described by~\citet{safaryan2022uncertainty,JMLR:v24:21-1548,demidovich2023guide}.

\begin{assumption}[Bounded Variance]\label{assum:bounded_variance}
A stochastic gradient $\nabla f_i(x;\xi_i)$ with its sample $\xi_i\sim \cD_i$ is an unbiased estimator of $\nabla f_i(x)$ with bounded variance, i.e., 
\begin{eqnarray}\label{eqn:bounded_variance}
    \Exp{\nabla f_i(x;\xi_i)} = \nabla f_i(x), \quad \text{and} \quad \Exp{\sqnorm{\nabla f_i(x;\xi_i) - \nabla f_i(x)}} \leq \sigma^2,
\end{eqnarray}
for all $x\in \R^d$.
\end{assumption}

Assumption~\ref{assum:bounded_variance} is standard for stochastic optimization \citep{nemirovski2009robust, ghadimi2012optimal, ghadimi2013stochastic} that is only imposed on each local stochastic gradient, and it does not imply data heterogeneity, i.e., the bounded difference between each component function $f_i(x)$ and the global function $f(x)$.

\begin{algorithm}
\caption{Normalized Error Feedback (\algname{||EF21||})}
\label{alg:normalized_ef21}
\begin{algorithmic}[1]
\STATE \textbf{Input:} Stepsize $\gamma_k>0$ for $k=0,1,\ldots$; starting points $x^0,\revision{g_i^{-1}} \in \R^d$ for $i\in \{1,2,\ldots,n\}$; and $\alpha$-contractive compressors $\cC^k:\R^d\rightarrow\R^d$ for $k=0,1,\ldots$.
\FOR{each iteration $k = 0, 1, \dots, K$}
    \FOR{each client $i = 1, 2, \dots, n$ in parallel}
        \STATE Compute local gradient $\nabla f_i(x^k)$
        \STATE Transmit $\Delta_i^k = \cC^k(\nabla f_i(x^k) - \revision{g_i^{k-1}})$
        \STATE Update  \revision{$g_i^k = g_i^{k-1} +\Delta_i^k $}
    \ENDFOR
    \STATE Central server computes \revision{$g^{k} = \frac{1}{n}\sum_{i=1}^n g_i^k$ via  $g_i^k = g_i^{k-1} +\Delta_i^k$}
    \STATE Central server updates \revision{$x^{k+1} = x^k - \gamma_k \frac{g^k}{\norm{g^k}}$}
\ENDFOR
\STATE \textbf{Output:} $x^{K+1}$
\end{algorithmic}
\end{algorithm}

\section{Normalized Error Feedback (\algname{\large ||EF21||})}
\label{sec:ef21}

For nonconvex deterministic optimization under generalized smoothness,  we  develop a distributed error feedback algorithm. One challenge is that the generalized smoothness parameter scales with the gradient norm $\norm{\nabla f(x^k)}$. To resolve this issue, we apply gradient normalization to the algorithms. 
In particular, we consider \algname{||EF21||}, the normalized version of \algname{EF21}~\citep{richtarik2021ef21} that updates the next iterates $x^{k+1}$ using the \algname{||EF21||} update. 
The full description of \algname{||EF21||} can be found in Algorithm~\ref{alg:normalized_ef21}.

%We first develop an error feedback algorithm for minimizing 

%The main algorithm, which we present as Algorithm~\ref{alg:normalized_ef21}, is \algname{||EF21||} to solve the general finite-sum problem in~\eqref{eqn:Problem} under generalized smoothness. 

%the gradient descent update uses $v^k$, rather than $v^k/\norm{v^k}$ at the central server, and recovers distributed gradient descent when the compressor is an identity operator, i.e. $\cC^k(\cdot)=I$. 

%\sarit{Plan to include \algname{||EF21||} -- version 2, where the stepsize is $\gamma = \frac{\gamma_0}{\sqrt{K+1}}$ for $\gamma>0$ and $K\in\sN$.}

Our new method \algname{||EF21||}, just like \algname{EF21}~\citep{richtarik2021ef21} under traditional smoothness, enjoys the $\cO(1/\sqrt{K})$ convergence in the gradient norm under generalized smoothness, as shown below.

\begin{theorem}[Convergence of \algname{||EF21||}]\label{thm:ef21}
  Consider Problem~\eqref{eqn:Problem}, where Assumption~\ref{assum:lowerbound_whole_f} (lower bound on $f$), Assumption~\ref{assum:lowerbound_f_i} (lower bound on $f_i$), Assumption~\ref{assum:LzeroLoneSmooth} (generalized smoothness of $f_i$), and Assumption~\ref{assum:contractive_comp} (contractive compressor) hold.
  Then, the iterates $\{x^k\}$ generated by \algname{||EF21||} (Algorithm~\ref{alg:normalized_ef21}) with 
  \begin{eqnarray*}
      \gamma_k = \frac{\gamma_0}{\sqrt{K+1}} 
  \end{eqnarray*}
  for $K \geq 0$ and $\gamma_0>0$ satify   
\begin{eqnarray*}
    \underset{k=0,1,\ldots,K}{\min} \Exp{\norm{\nabla f(x^k)}}
   \leq  \frac{V^0 \exp(8c_1L_1 \exp(L_1\gamma_0)\gamma_0^2)}{\gamma_0 \sqrt{K+1}} + B\frac{\gamma_0 \exp(L_1\gamma_0)}{\sqrt{K+1}}, 
\end{eqnarray*}
where $V^k \eqdef f(x^k)-f^{\inf} + \frac{2\gamma_k}{1-\sqrt{1-\alpha}} \frac{1}{n}\sum_{i=1}^n \norm{\nabla f_i(x^k)-g_i^k}$,
$B= 2c_0  +   \frac{8L_1c_1}{n}\sum_{i=1}^n (f^{\inf}-f^{\inf}_{i})$, and  $c_i = \left(\frac{1}{2} + 2\frac{\sqrt{1-\alpha}}{1-\sqrt{1-\alpha}} \right)L_i$ for $i=0,1$. %$b_1 = 8c_1L_1$, $b_2 = 1$, and $b_3 = 2c_0 +  \frac{8c_1L_1}{n}\sum_{i=1}^n (f^{\inf}-f^{\inf}_{i})$, 
\end{theorem}

%\eduard{I suggest to write everywhere "generalized smoothness" or "$(L_0,L_1)$-smoothness" instead of "\textcolor{red}{generalized smoothness}" because, in the end, we will have all the results under symmetric $(L_0,L_1)$ (so, for simplicity, we can just omit the word "asymmetric" everywhere)}

Theorem~\ref{thm:ef21} establishes the $\cO(1/\sqrt{K})$ convergence in the expectation of gradient norms for \algname{||EF21||} on nonconvex deterministic problems under generalized smoothness.
This rate is the same as Theorem 1 of \citet{richtarik2021ef21} for \algname{EF21} under traditional smoothness, and does not depend on data heterogeneity conditions in contrast to \citet{crawshaw2024federated,liu2022communication}. 
Also, our stepsize depends on any positive constant $\gamma_0$, and total iteration count $K$, without needing to know smoothness constants $L_0,L_1$ in contrast to \citet{richtarik2021ef21}.
%
%Our proof utilizes a Lyapunov function $V^k = f(x^k)-f^{\inf} + \frac{2\gamma_k}{1-\sqrt{1-\alpha}} \frac{1}{n}\sum_{i=1}^n \norm{\nabla f_i(x^k)-v_i^k}$, which draws an  inspiration from the convergence analysis for \algname{EF21} under  $L$-smoothness by~\citet{richtarik2021ef21}.
%\sarit{TODO: Fix the proof. Because the discussions are not true. }
%
%\paragraph{The convergence when the stepsize is inversely proportional to $L_1$.}
%We can determine the stepsize $\gamma_0$ that minimizes the convergence bound when the exponential term is dominant or negligible. On the one hand, when the exponential term is dominant, we choose 
Additionally, if we choose $\gamma_0 = 1/(8c L_1)$, then our convergence bound from Theorem~\ref{thm:ef21} becomes
\begin{eqnarray*}
  \underset{k=0,1,\ldots,K}{\min} \Exp{\norm{\nabla f(x^k)}}
   &\leq&  \frac{32c L_1 V^0 +L_0/L_1 + 2L_1 \delta^{\inf} }{\sqrt{K+1}}, 
\end{eqnarray*}
where $c=\frac{1}{2} + 2\frac{\sqrt{1-\alpha}}{1-\sqrt{1-\alpha}}$, and $\delta^{\inf} = \frac{1}{n}\sum_{i=1}^n (f^{\inf} -f_i^{\inf} )$.

%On the other hand, when the exponential term is negligible, we select 
%\begin{eqnarray*}
%    \gamma_0 = \underset{\gamma_0 > 0}{\min} \frac{V_0}{\gamma_0} + \gamma_0 b = \sqrt{\frac{V_0}{b}},
%\end{eqnarray*}
%which gives 
%\begin{eqnarray*}
%    \underset{k=0,1,\ldots,K}{\min} \Exp{\norm{\nabla f(x^k)}}
%   \leq  \frac{\sqrt{b V_0}}{\sqrt{K+1}}\left[ \exp\left( \frac{8c_1L_1 V_0}{b} \right) + 1  \right].
%\end{eqnarray*}

\paragraph{Comparisons between \algname{||EF21||} and \algname{EF21} under traditional smoothness.}
For nonconvex, traditional smooth problems,  \algname{||EF21||} from Theorem~\ref{thm:ef21} with $L_1=0$  achieves the same $\cO(1/\sqrt{K})$ rate in the expectation of gradient norms as \algname{EF21} analyzed by~\citet{richtarik2021ef21}, but with a larger convergence factor. We prove this by assuming \revision{$\nabla f_i(x^0)=g_i^0$} for all $i$. That is, Theorem~\ref{thm:ef21} with $L_0=L$, $L_1=0$,  $\gamma_0 =\sqrt{(f(x^0)-f^{\inf})/(2b)}$, and $b=\frac{L}{2} + 2\frac{\sqrt{1-\alpha} L}{1-\sqrt{1-\alpha}}$ implies that \algname{||EF21||} achieves
\begin{eqnarray*}
    \underset{k=0,1,\ldots,K}{\min} \Exp{\norm{\nabla f(x^k)}}
   &\leq&  \frac{1}{\sqrt{K+1}}\left[ \frac{f(x^0)- f^{\inf}}{\gamma_0} + 2b\gamma_0\right]\\
  %  &\leq&   2\sqrt{2c} \sqrt{\frac{f(x^0)-f^{\inf}}{K+1}} \\
   & \leq &  2 \sqrt{ L \frac{(1+3\sqrt{1-\alpha})(1+\sqrt{1-\alpha})}{\alpha} }  \sqrt{\frac{f(x^0)-f^{\inf}}{K+1}} \\
   & \overset{\alpha \geq 0}{\leq} & 4\sqrt{2}  \sqrt{\frac{L}{\alpha}} \sqrt{\frac{f(x^0)-f^{\inf}}{K+1} }.
\end{eqnarray*}
On the other hand, \algname{EF21} attains from Theorem 1 of \cite{richtarik2021ef21} with $L_i=\tilde L = L$ (i.e., $f_i(x)$ has the same smoothness constant as $f(x)$), and $\hat x^K$ being chosen from the iterates $x^0,x^1,\ldots,x^K$ uniformly at random
\begin{eqnarray*}
     \underset{k=0,1,\ldots,K}{\min} \Exp{\norm{\nabla f(x^k)}} 
     & \leq &   \Exp{\norm{\nabla f(\hat x^K)}} \\
   & \leq & \sqrt{  \Exp{ \sqnorm{\nabla f(\hat x^K)}} } \\
   & \leq & \sqrt{ 2L ( 1+ \sqrt{\beta/\theta}) \frac{f(x^0)-f^{\inf}}{K+1}   } \\
   & \overset{\sqrt{\beta/\theta} \leq 2/\alpha-1}{\leq} & 2 \sqrt{\frac{L}{\alpha}} \sqrt{\frac{f(x^0)-f^{\inf}}{K+1} }.
\end{eqnarray*}
In conclusion, the convergence bound of \algname{||EF21||} is slower by a factor of $2\sqrt{2}$ than the original \algname{EF21} for nonconvex, $L$-smooth problems.

While \algname{||EF21||} can handle nonconvex problems under generalized smoothness, the algorithm is limited to deterministic settings, where each node computes its full local gradient. In the following section, we demonstrate how to integrate normalization into \algname{EF21-SGDM}~\citep{fatkhullin2024momentum}, an error feedback algorithm that allows each node to compute its local stochastic gradient, for solving nonconvex stochastic problems.

\begin{algorithm}
\caption{Normalized Error Feedback with Stochastic Gradients \& Momentum (\algname{||EF21-SGDM||})}
\label{alg:normalized_ef21_sgdm}
\begin{algorithmic}[1]
\STATE \textbf{Input:} Stepsizes $\gamma_k>0$ and $\eta_k\in[0,1]$ for $k=0,1,\ldots$; starting points $x^0,g_i^{-1} \in \R^d$ for $i\in \{1,2,\ldots,n\}$, and \revision{$v_i^{-1} = \nabla f_i(x_i^0;\xi_{i}^0)$ with independent random samples $\xi_{i}$ for $i\in \{1,2,\ldots,n\}$}; $\alpha$-contractive compressors $\cC^k:\R^d\rightarrow\R^d$ for $k=0,1,\ldots$ 
\FOR{each iteration $k = 0, 1, \dots, K$}
    \FOR{each client $i = 1, 2, \dots, n$ in parallel}
        \STATE Compute a local stochastic gradient $\nabla f_i(x^k;\xi_i^k)$
        \STATE Update a momentum estimator $v_i^k = (1-\eta_k)v_i^{k-1} + \eta_k \nabla f_i(x^{k};\xi_i^k)$
        \STATE Transmit $\Delta_i^k = \cC^k( v_i^k - g_i^{k-1})$
        \STATE Update  $g_i^k = g_i^{k-1} +\Delta_i^k $
    \ENDFOR
    \STATE Central server computes $g^{k} = \revision{\frac{1}{n}}\sum_{i=1}^n g_i^k$ via  $g_i^k = g_i^{k-1} +\Delta_i^k$
    \STATE Central server updates $x^{k+1} = x^k - \gamma_k \frac{g^k}{\norm{g^k}}$
\ENDFOR
\STATE \textbf{Output:} $x^{K+1}$
\end{algorithmic}
\end{algorithm}

\section{Normalized Error Feedback with Stochastic Gradients \& Momentum (\algname{\large||EF21-SGDM||})}
\label{sec:ef21_sgdm}

%\sarit{EF21 for deterministic settings. Now, the extension to \algname{EF21} for stochastic settings is to consider \algname{EF21-SGDM} . Why? Because unlike EF21-SGD, \algname{EF21-SGDM}  leverages momentum updates to allow for formining minibatch size of $\cO(1)$.}

%\sarit{Rewrite the introductory text. Key messages are \algname{EF21-SGDM}  are for stochastic optimization. But we need the normalization update to deal with the gradient norm due to generalized smoothness. }

Having established the convergence of \algname{||EF21||} for deterministic optimization, we will next develop a distributed error feedback algorithm that incorporate stochastic gradients and normalization to accommodate generalized smoothness conditions. 
In particular, we focus on \algname{||EF21-SGDM||}  (Algorithm~\ref{alg:normalized_ef21_sgdm}), the normalized version of \algname{EF21-SGDM} due to \citet{fatkhullin2024momentum}. We also note that  \algname{||EF21-SGDM||}  recovers many optimization algorithms of interest in the special cases. For instance, it  reduces to   
\begin{itemize}
\item  normalized version of \algname{EF21}~\citep{richtarik2021ef21}, which we call \algname{||EF21||}, when we let $\eta_k=1$ and $\nabla f_i(x^k;\xi_i^k)=\nabla f_i(x^k)$, 
\item normalized version of \algname{EF21-SGD}~\citep{fatkhullin2021ef21}, which we call \algname{||EF21-SGD||}, when we let $\eta_k=1$, and 
\item  normalized version of \algname{SGDM}~\citep{cutkosky2020momentum}, which we call \algname{||SGDM||}\footnote{This method is also known as \algname{NSGD-M}.}, when we let $\eta_k=1-\beta_k$ and $\cC^k(\cdot)$ is the identity compressor/mapping. 
\end{itemize}
%Second, when the central server rather updates $x^{k+1}=x^k-\gamma_k g^k$ in Step 10 of Algorithm~\ref{alg:normalized_ef21_sgdm}, \algname{||EF21-SGDM||}  recovers EF21-SGD~\citep{fatkhullin2021ef21} when we let $\eta_k=1$, EF21~\cite{richtarik2021ef21} when we let $\eta_k=1$ and $\nabla f_i(x^k;\xi_i^k)=\nabla f_i(x^k)$ 

%Next, we turn our attention to stochastic settings, where error feedback algorithms use stochastic gradients, rather than full gradients, to compute the iterates. For solving nonconvex stochastic optimization under \textcolor{red}{generalized smoothness}, we consider the normalized version of \algname{EF21-SGDM} ~\citep{fatkhullin2024momentum}, which is an \algname{EF21} variant with a momentum update.

%For solving nonconvex, stochastic, generalized smooth problems, the natural extension of \algname{||EF21||} is to compute the mini-batch gradient $\nabla f_i(x;\xi_i)$, rather than the full gradient $\nabla f_i(x)$, at each node $i$.  

%Next, we proceed to develop error feedback algorithms in the nonconvex, stochastic, generalized smooth regimes. 
%In particular, we consider the normalized version of \algname{EF21-SGDM} ~\citep{fatkhullin2024momentum}, which  

%\sarit{Plan to include \algname{||EF21||}-SGD, where the stepsize is $\gamma = \frac{\gamma_0}{\sqrt{K+1}}$ for $\gamma>0$ and $K\in\sN$. ABD shows the result.}

In the next theorem, we demonstrate that \algname{||EF21-SGDM||}  attains the same \(\mathcal{O}(1/K^{1/4})\) convergence rate as both \algname{EF21-SGDM}  and \algname{||SGDM||}. 

%\sarit{Abdurakhmon will type his latest convergence proof for \algname{||EF21-SGDM||} . The results use the variance bound in expectation, anddo not require data heterogeneity. The only issue is that at initialization the mini-batch size must be large enough. }

%\sarit{We can use the big $\cO$ notation for showing the result. Why? Because (1) the parameter-agnostic paper presented the results like this, and (2) we can compare only the rate which is the same as the parameter-agnostic paper. }

%\sarit{Abdurakhmon will add the conditions of the mini-batch size in the theorem below.}
%\sarit{Ok. Fixed.}

\begin{theorem}[Convergence of \algname{||EF21-SGDM||}]\label{thm:ef21_sgdm}
 Consider Problem~\eqref{eqn:Problem}, where Assumption~\ref{assum:lowerbound_whole_f} (lower bound on $f$), Assumption~\ref{assum:lowerbound_f_i} (lower bound on $f_i$), Assumption~\ref{assum:LzeroLoneSmooth} (generalized smoothness of $f_i$), Assumption~\ref{assum:contractive_comp} (contractive compressor), and Assumption~\ref{assum:bounded_variance} (bounded variance) hold.
 If \revision{$g_i^{-1} = 0$ for $i\in \{1,\ldots,n\}$ and}
 \begin{eqnarray*}
      \gamma_k\equiv \gamma &=& \frac{\gamma_0}{(K+1)^{\nicefrac{3}{4}}},\text{ with } \revision{0 < \gamma_0 \leq \frac{1}{16L_1}\min\left\{(K+1)^{\nicefrac{1}{2}}C_\alpha, 1\right\}}, \quad \text{and} \\
      \quad \eta_k \equiv \eta &=& \frac{1}{(K+1)^{\nicefrac{1}{2}}},
  \end{eqnarray*}
  \revision{where $C_\alpha \eqdef 1 - \sqrt{1-\alpha}$,} then the iterates $\{x^k\}$ generated by \algname{||EF21-SGDM||}  (Algorithm~\ref{alg:normalized_ef21_sgdm})
  %with the mini-batch size at the starting point $B^{\rm init} = \sqrt{K+1}$, and with stepsizes
  satisfy for $K \geq 0$ 
  %and $0< \gamma_0 \leq \frac{1}{8L_1\sqrt{1+\nicefrac{\sqrt{1-\alpha}}{\alpha}}}$ satisfy   
 \begin{eqnarray*}
   \underset{k=0,1,\ldots,K}{\min} \Exp{\norm{\nabla f(x^k)}}
   &\leq& \revision{\gO\left(\frac{\nicefrac{\delta^0}{\gamma_0}+\nicefrac{\sigma}{\sqrt{n}} +\gamma_0(L_0 +  L_1^2\delta^{\inf})}{(K+1)^{\nicefrac{1}{4}}}\right)}\\
   && \revision{+\gO\left(\frac{\sqrt{1-\alpha}}{\alpha}\left(\frac{\sigma}{(K+1)^{\nicefrac{1}{2}}} + \frac{\gamma_0(L_0 +  L_1^2\delta^{\inf})}{(K+1)^{\nicefrac{3}{4}}}\right)\right),}
\end{eqnarray*}      
where \revision{$\delta^0 \eqdef f(x^0) -f^{\inf}$, and $\delta^{\inf} \eqdef \frac{1}{n}\sum^n_{i=1}(f^{\inf} - f^{\inf}_i)$.}
%\begin{equation*}
%    \delta^{\inf} = \frac{1}{n}\sum^n_{i=1}(f^{\inf} - f^{\inf}_i);\quad V^0 = f(x^0) -f^{\inf} +\frac{2\gamma}{(1-\sqrt{1-\alpha})n}\sum^n_{i=1}\norm{v^0_i-g^0_i}.
%\end{equation*}
% \begin{equation*}
%     v^0_i = \frac{1}{B^{\text{init}}}\sum^{B^{\text{init}}}_{j=1} \nabla f_i(x^{0}_i;\xi^0_{i,j}),
% \end{equation*}
% where $\xi^0_{i,j}$ are i.i.d., $j \in B^{\text{init}}$, and $ B^{\text{init}} = \sqrt{K+1}$.
\end{theorem}

From Theorem~\ref{thm:ef21_sgdm},  \algname{||EF21-SGDM||}  under generalized smoothness achieves the  $\mathcal{O}(1/K^{1/4})$ convergence rate in the expectation of gradient norms. 
This rate is the same as that of \algname{EF21-SGDM}, previously analyzed under traditional smoothness by \citet[Theorem 3]{fatkhullin2024momentum}. 
The result holds regardless of the data heterogeneity degree and \revision{the mini-batch size.} 
% Additionally, one possible for the stepsize $\gamma_0>0$ satisfying the condition from Theorem~\ref{thm:ef21_sgdm} is $\gamma_0 \leq \nicefrac{1}{(9L_1\sqrt{1+B(\alpha)})}$ with $B(\alpha)=\nicefrac{\sqrt{1-\alpha}}{\alpha}$. 
\revision{We also notice} that the stepsize \(\gamma_0\) for \algname{||EF21-SGDM||} , unlike in the case of \algname{||EF21||}, depends on the generalized smoothness constant \(L_1\), and the compression parameter \(\alpha\). \revision{However, the considered choice of stepsizes is agnostic to $\sigma$ and $L_0$.} 

%\eduard{When $n = 1$, $\alpha = 1$, do we get the rate as in \citep{hubler2024parameter}? If yes, then let us mention it.}
%\sarit{Should be the comparison for \algname{EF21-SGDM}  case below.}

Furthermore, Theorem~\ref{thm:ef21_sgdm} with \(\alpha = 1\) (i.e., \(\mathcal{C}^k\) is the identity compressor) implies the convergence bound of the distributed version of normalized SGD with momentum (\algname{||SGDM||}) \citep{cutkosky2020momentum} using $\beta=1-\eta$: 
 \begin{eqnarray}\label{eqn:result_ef21_sgdm_no_comp}
    \underset{k=0,1,\ldots,K}{\min} \Exp{\norm{\nabla f(x^k)}}
    &\leq& \gO\left(\frac{\nicefrac{(f(x^0)-f^{\inf})}{\gamma_0}+\nicefrac{\sigma}{\sqrt{n}} +\gamma_0L_0 + \gamma_0 L_1^2\delta^{\inf} }{(K+1)^{\nicefrac{1}{4}}}\right).
 \end{eqnarray} 
For the single-node  \algname{SGDM}, where \( n = 1 \) and \( \delta^{\inf} = 0 \), our convergence bound in \eqref{eqn:result_ef21_sgdm_no_comp} with \revision{$\gamma_0 = \Theta(\nicefrac{1}{L_1})$} achieves the $\gO\left(\frac{ L_1(f(x^0)-f^{\inf})+ \sigma + \nicefrac{L_0}{L_1} }{(K+1)^{\nicefrac{1}{4}}}\right)$ convergence, which matches the rate obtained by \citet[Corollary 3]{hubler2024parameter}. Unlike the earlier results for single-node  \algname{SGDM}, our result \revision{holds for the} multi-node \revision{regime}. The bound in \eqref{eqn:result_ef21_sgdm_no_comp} for multi-node  \algname{SGDM} includes the \(\nicefrac{\sigma}{\sqrt{n}}\)-term indicating a \(\sqrt{n}\)-fold reduction in the influence of stochastic variance noise \(\sigma\), and the \(\gamma_0 L_1^2 \delta^{\inf}\)-term accounting for the effect of data heterogeneity.

\section{Experiments}

We evaluate the performance of \algname{||EF21||}, and compare it against \algname{EF21}~\citep{richtarik2021ef21}. We test these algorithms for three nonconvex problems that satisfy generalized smoothness: the problem of minimizing  polynomial functions, the logistic regression problem with a nonconvex regularization term over synthetic  and benchmark datasets from LIBSVM~\citep{chang2011libsvm}, and the training of the ResNet-20~\citep{7780459} model over the CIFAR10~\citep{Krizhevsky2009LearningML} dataset\footnote{We implemented \algnamesmall{EF21} and  \algnamesmall{||EF21||} on training the ResNet-20 model by using PyTorch. Our source codes can be found in \href{https://anonymous.4open.science/r/error-feedback-generalized-smoothness-9DDC}{the link to error-feedback-generalized-smoothness-paper}.}. For all experiments, we use a top-$k$ sparsifier, which is a $\frac{k}{d}$-contractive compressor.  

%\eduard{TODO: remove the word "asymmetric" from the above}

%\subsection{Minimization of  Polynomial Functions}

\subsection{Logistic Regression with a Nonconvex Regularizer}

%\sarit{We try to show that \algname{||EF21||} allows for larger stepsizes, leading to faster convergecen than  \algname{EF21}.}

%\sarit{Artem 1. Fixed. Artem is right. It must not be $\lambda/n$.}

%In the second synthetic experiment, we extend the comparison of original \algname{EF21} and  \algname{||EF21||} to a multi-node scenario for minimizing nonconvex $(L_0, L_1)$-smooth problems. To adress this, 
First, we consider a logistic regression problem with a nonconvex regularizer, i.e., Problem~\eqref{eqn:Problem} with 
\begin{align*}
f_i(x)=\log(1+\exp(-b_i a_i^T x)) + \lambda \sum_{j=1}^d \frac{x_j^2}{1+x_j^2},
\end{align*}
%
%\begin{align*}
%\min _{x \in \mathbb{R}^d}\biggl\{f(x)\eqdef\frac{1}{n}\sum_{i=1}^n f_i(x)\eqdef \frac{1}{n}\sum_{i=1}^n \underbrace{\log(1+\exp(-b_i a_i^T x))}_{=:\tilde f_i(x)} + \underbrace{\lambda \sum_{j=1}^d \frac{x_j^2}{1+x_j^2}}_{=:h(x)}\biggl\},
%\end{align*}
where $a_i \in \R^d$ is the $i^{\rm th}$ feature vector of data matrix $A\in\R^{n\times d}$ with its class label $b_i \in \{-1,1\}$, and $\lambda>0$ is a regularization parameter. Here, $f(x)$ is nonconvex, and $L$-smooth with $L=\sqnorm{A}/(4n) +2\lambda$. Also, each $f_i(x)$ is $\hat L_i$-smooth with $\hat L_i = \sqnorm{a_i}/4+2\lambda$, and generalized smooth with $L_0=2\lambda+\lambda \sqrt{d} \max_i \norm{a_i}$ and $L_1 = \max_i \norm{a_i}$. The derivations of smoothness parameters can be found in Appendix~\ref{app:smooth_logistic}.

In these experiments, we initialized \( x^0 \in \R^d \), where each coordinate was drawn from a standard normal distribution \(\mathcal{N}(0,1)\), and set \(\lambda = 0.1\). %, and used a top-$k$ compressor (with $\alpha=k/d$). 
Here, the condition $\lambda >  \lambda_{\min}\left(A^{\top} A\right)/(2n)$  ensures that $f(x)$ is nonconvex.
We ran  \algname{||EF21||} and \algname{EF21} on the following datasets: (1) two from LIBSVM~\citep{chang2011libsvm}: \texttt{Breast Cancer} (\(n = 683\), \(d = 10\), and scaled to $[-1,1]$), and \texttt{a1a} (\(n = 1605\), \(d = 123\)); and (2) a synthetically generated dataset (\(n = 20\), \(d = 10\)), where the data matrix \( A \in \R^{n \times d} \) had entries drawn from \(\mathcal{N}(0,1)\), and the class label \( b_i \) was set to either \(-1\) or \(1\) with equal probability.
For \algname{EF21}, we selected the stepsize $\gamma_k  = {1}/{\left( L + \tilde L \sqrt{{\beta}/{\theta}} \right)}$ with $\tilde L = \sqrt{ \sum_{i=1}^n  \hat L_i^2/n}$, $\theta=1-\sqrt{1-\alpha}$, and $\beta={(1-\alpha)}/{(1-\sqrt{1-\alpha})}$, as given by \citet[Theorem 1]{richtarik2021ef21}.
For \algname{||EF21||}, we chose  $\gamma_k=\gamma_0/\sqrt{K+1}$ with $\gamma_0>0$ from Theorem~\ref{thm:ef21}, by setting $\gamma_0=1$, $K = 100$ for the generated data and \texttt{Breast Cancer}, and $K=400$ for \texttt{a1a}.
We choose $\gamma_0=1$, because \algname{||EF21||} with $\gamma_0 \in [1,10]$ converges faster than that with small values of $\gamma_0$ (e.g. $0.1$), when we run the algorithm on a single node ($n=1$) for minimizing polynomial function and solving logistic regression.
%\algname{||EF21||}, when run on a single node ($n=1$) converged faster with $\gamma_0 \in [1,10]$ than very small values of $\gamma_0$ (e.g. $0.1$). This is also observed
%Conversely, small values of $\gamma_0$, such as $0.1$, required significantly more iterations to reach convergence. 
We determine $K$ as the smallest number of iterations required to achieve the desired accuracy by performing a grid search with a stepsize of $50$.
%We chose $\gamma_0=1$ and $K=100$, because we observed that \algname{||EF21||} in a single-node setting ($n=1$) with $\gamma_k=\gamma_0/\sqrt{K+1}$ and very small values of $\gamma_0$, e.g. $0.1$, requires significantly higher iterations to converge.
%For logistic regression problems, as we observed by running \algname{||EF21||} in a single-node setting ($n=1$), the algorithm with $\gamma_k=\gamma_0/\sqrt{K+1}$ and very small values of $\gamma_0$, e.g. $0.1$, requires significantly higher iterations to converge. Choosing $\gamma_0 \in [1,10]$ tends to enable \algname{||EF21||} attain fast convergence speed. We also select $K$ as the smallest number of iterations required to achieve the desired accuracy by performing a grid search with the stepsize of $50$.

%As we observed in 

%In this problem, as in the synthetic experiment for a single-node case, we observed that for small values of $\gamma_0$, such as $0.1$, the  \algname{||EF21||} requires significantly more iterations to converge compared to choosing $\gamma_0$ values between $1$ and $10$. Therefore, in all experiments, we set $\gamma_0 = 1$ and vary only $K$ to achieve convergence (we select $K$ as the smallest number of iterations required to achieve the desired accuracy by performing a grid search with the stepsize of $50$).

%Now, let us compare the performance of original \algname{EF21} and  \algname{||EF21||} for different choices of $A$ and $y$. 

Figure~\ref{fig:exp2} shows that \algname{||EF21||} outperforms the traditional \algname{EF21} on all evaluated datasets, achieving faster convergence and higher solution accuracy. This improvement results from the fact that the theoretical stepsize for \algname{||EF21||}, as derived in Theorem~\ref{thm:ef21}, is larger than the stepsize for  \algname{EF21} outlined by \citet[Theorem 1]{richtarik2021ef21}.

%converges faster than \algname{EF21} 
%we can again observe that  \algname{||EF21||} converges faster to achieve the desired accuracy, $\norm{\nabla f(x)}^2 < \epsilon$, $\epsilon = 10^{-3}$, compared to the original  \algname{EF21}.
\begin{figure}[h]
    \centering
%    \begin{subfigure}{\textwidth}
        \includegraphics[width=0.3\textwidth]{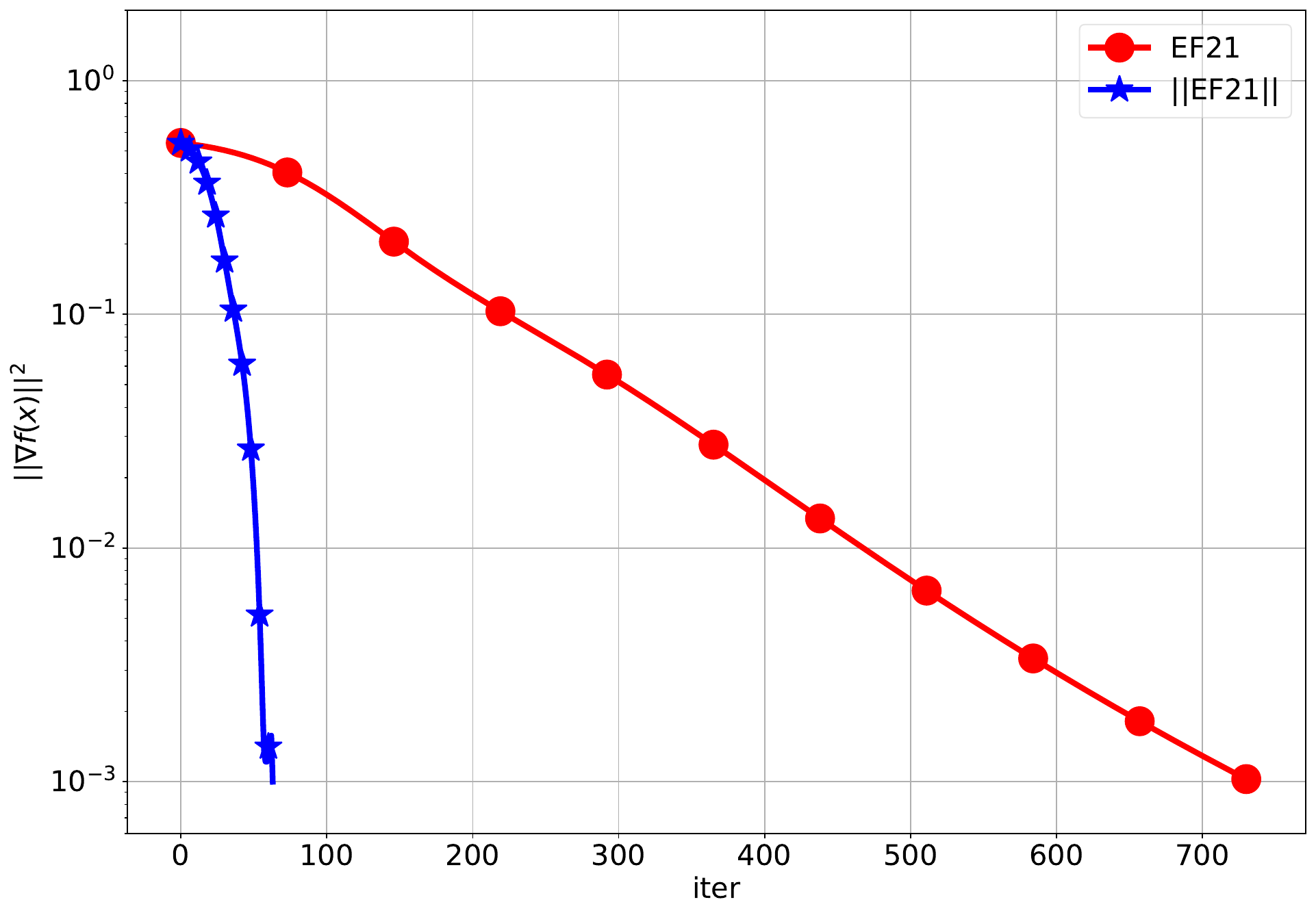}
 %       \caption{Custom, $K=100$}
 %   \end{subfigure}
    \hfill
 %   \begin{subfigure}{\textwidth}
        \includegraphics[width=0.3\textwidth]{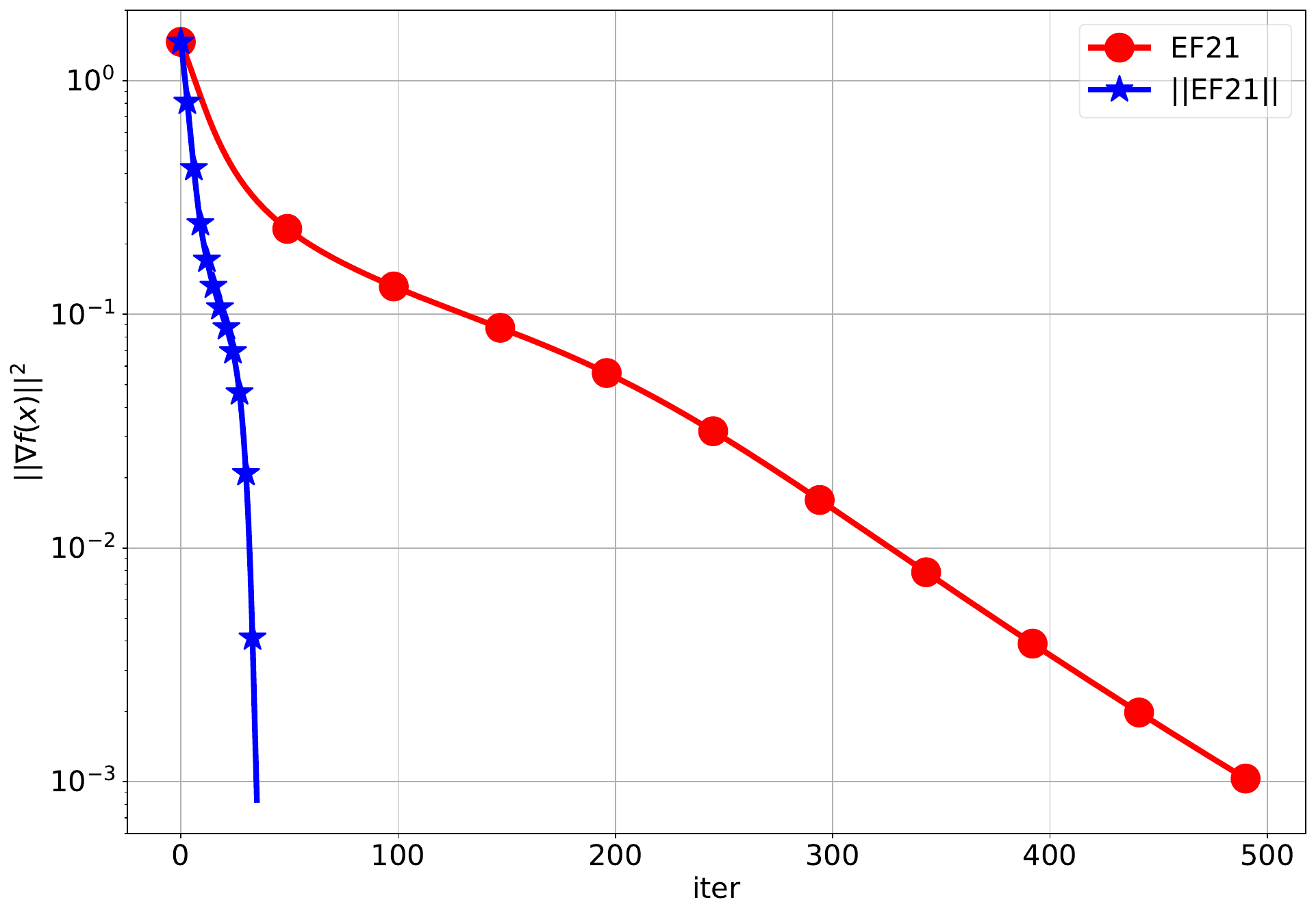}
  %      \caption{Breast Cancer, $K=100$}
  %  \end{subfigure}
    \hfill
   % \begin{subfigure}{\textwidth}
        \includegraphics[width=0.3\textwidth]{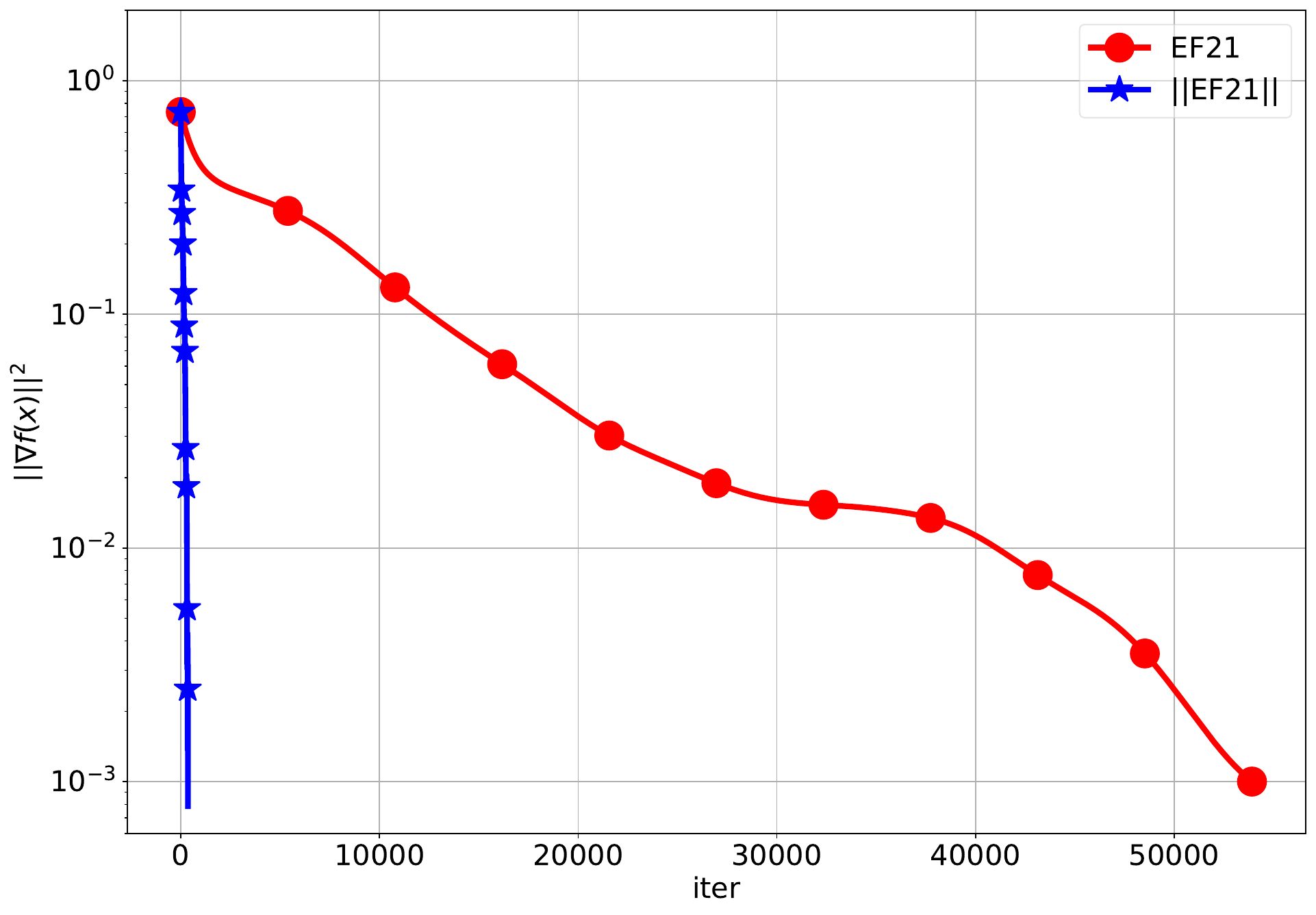}
    %    \caption{A1A, $K=400$}
    %\end{subfigure}
    \caption{Logistic regression with a nonconvex regularizer using normalized \algname{||EF21||} and  \algname{EF21}. We reported $\sqnorm{\nabla f(x^k)}$ with respect to iteration count $k$. We used the constant stepsize $\gamma = \frac{1}{L + \tilde L \sqrt{\frac{\beta}{\theta}}}$ for \algname{EF21}, and $\gamma = \frac{\gamma_0}{\sqrt{K+1}}$, $\gamma_0 = 1$ for \algname{||EF21||}. Here, $K=100$ for our generated data (left), and \texttt{Breast Cancer} (middle), while $K=400$ for \texttt{a1a}  (right). }
    \label{fig:exp2}
\end{figure}

\subsection{ResNet20 Training Over CIFAR-10}\label{subsec:resnet_20}

%\sarit{Under the same stepsize,  \algname{||EF21||} is better than EF21.}

Next, we trained the ResNet20~\citep{7780459} model on the CIFAR-10~\citep{Krizhevsky2009LearningML} dataset, which was demonstrated  empirically by \citet{zhang2019gradient}  to satisfy the $(L_0,L_1)$-smoothness condition. 
%ResNet20 contains $269,722$, while CIFAR-10 contains $60,000$ images. 
%To compare the performance between \algname{EF21} and \algname{||EF21||},
In these experiments, we used a top-$k$ compressor over $50,000$ training images, with evaluation on $10,000$ test images.
%we used a top-$k$ compressor, trained ResNet20 with $50,000$ training images, and evaluated the model with $10,000$ test images. 
The dataset was evenly distributed  among $5$ clients, each using a mini-batch size of $128$. Both algorithms were run for $100$ epochs with a constant stepsize $\gamma=5$. Here, one epoch refers to a full pass through the entire dataset processed by all clients.  

%\sarit{May need to remove one of the main plots below. }

From Figure%~\ref{fig:plot_exp_DL} and
~\ref{fig:plot_exp_DL}, under the same constant stepsize and the  top-$k$ sparsifier with $k=0.01d$, \algname{||EF21||} outperforms \algname{EF21}, in terms of  convergence speed (in gradient norms and losses) and accuracy, relative to the number of bits communicated from each client to the server. Specifically, \algname{||EF21||} achieved accuracy gains of up to $10$\% over \algname{EF21}. 

\begin{figure}[h]
   \centering
   \includegraphics[width=\textwidth]{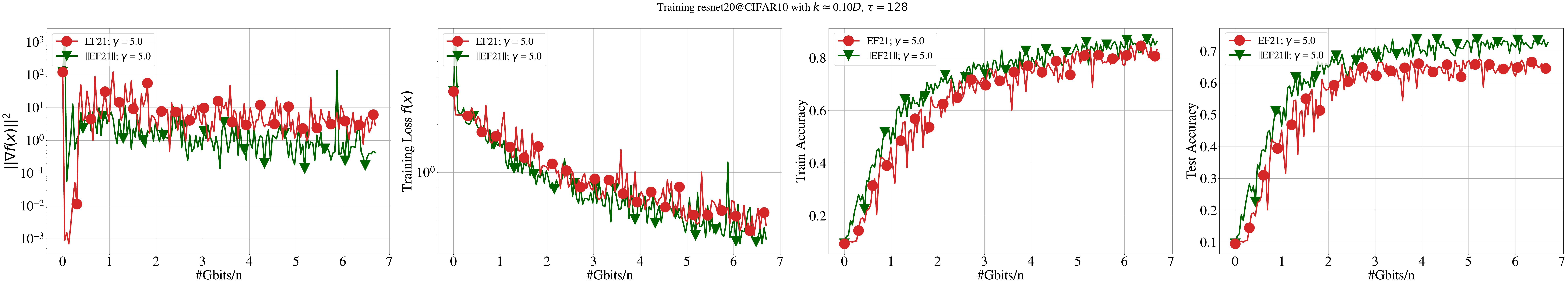}
    \caption{ResNet20 training on CIFAR-10 by using \algname{EF21} and \algname{||EF21||} under the same stepsize $\gamma=5$ and $k=0.1d$ for a  top-$k$ sparsifier.}
    %\caption{Image Classification experiment: Comparing \algname{EF21} and  \algname{||EF21||} under the same stepsize $\gamma=5$ during training ResNet20 on CIFAR-10. Using top-k compressor with $k = 0.1d$.}
   \label{fig:plot_exp_DL}
\end{figure}

\section{Conclusion and Future Works}

%\paragraph{Conclusion.}

In this paper, we have demonstrated that normalization can be effectively combined with \algname{EF21} to develop distributed error feedback algorithms for solving nonconvex optimization problems under generalized smoothness conditions. Specifically, \algname{||EF21||} and \algname{||EF21-SGDM||}  achieve convergence rates of $\mathcal{O}(1/K^{1/2})$ in deterministic settings and $\mathcal{O}(1/K^{1/4})$ in stochastic settings, respectively. These convergence rates match those of the vanilla \algname{EF21} and \algname{EF21-SGDM}  algorithms. Unlike previous works on distributed algorithms under generalized smoothness, our analysis does not assume data heterogeneity or impose smoothness-dependent restrictions on the stepsize \revision{(in the deterministic case)}. Finally, our experiments confirm that \algname{||EF21||} exhibits stronger convergence performance compared to the original \algname{EF21}, due to its larger allowable stepsizes.

%\paragraph{Future works.}  
Our work implies many promising research directions. One interesting direction is to extend our convergence results for \algname{||EF21||} and \algname{||EF21-SGDM||}  to accommodate decreasing or adaptive stepsize schedules, as the constant stepsizes required by our current analysis can become impractically small when the total number of iterations is large. 
% In particular, applying appropriate decreasing stepsizes to \algname{EF21-SGDM}  could overcome its current theoretical requirement of a sufficiently large mini-batch size for the stochastic gradient at initialization. 
Another important direction is the development of distributed and federated algorithms that leverage clipping or normalization for minimizing nonconvex generalized smooth functions. %These algorithms could be further enhanced by incorporating additional mechanisms, such as adding differentially private noise or applying communication compression, to tackle broader challenges in distributed training.

\section*{Acknowledgement}

The research reported in this publication was supported by funding from King Abdullah University of Science and Technology (KAUST): i) KAUST Baseline Research Scheme, ii) Center of Excellence for Generative AI, under award number 5940, iii) SDAIA-KAUST Center of Excellence in Artificial Intelligence and Data Science.

%\bibliography{iclr2025_conference}
%\bibliography{iclr2025_conference}

%\bibliographystyle{plain} % We choose the "plain" reference style
\bibliography{iclr2025_conference} 

\newpage

\appendix

\newpage 
\section{Lemmas}
%\abdrk{Change statements and proofs of all lemmas in this section, according to the Symmetric generalized smoothness  }
%\sarit{Done with all the proofs.}
In this section, we introduce useful lemmas  for our analysis. 
Lemmas~\ref{lemma:LzeroLoneSmooth} and~\ref{lemma:bound_avg_norm_grad_multinodeef21} introduce inequalities by generalized smoothness, while Lemmas~\ref{lemma:descent_ineq_generalized_hubler} and~\ref{lemma:conv_v2} present the descent inequality and convergence rate, respectively, when the normalized gradient descent update is applied.

\begin{lemma}\label{lemma:LzeroLoneSmooth}
Let each $f_i(x)$ be generalized smooth with  parameters $L_0,L_1>0$, and lower bounded by $f^{\inf}_{i}$, and  let $f(x)=\frac{1}{n}\sum_{i=1}^n f_i(x)$. Then, for any $x,y\in\R^d$
\begin{eqnarray}
&&\norm{\nabla f_i(x)-\nabla f_i(y)} \leq (L_0+L_1\norm{\nabla f_i(y)})\exp\left(L_1\norm{x-y}\right)\norm{x-y}, \label{eqn:LzeroLoneSmooth_ineq_0}\\
&&f_i(y) \leq f_i(x) + \inp{\nabla f_i(x)}{y-x} + \frac{L_0+L_1\norm{\nabla f_i(x)}}{2}\exp\left(L_1\norm{x-y}\right)\sqnorm{y-x}, \label{eqn:LzeroLoneSmooth_ineq_1}\\
  &&\frac{\sqnorm{\nabla f_i(x)}}{4(L_0+L_1\norm{\nabla f_i(x)})} \leq f_i(x)-f^{\inf}_{i}, \ \text{and}
  \label{eqn:LzeroLoneSmooth_ineq_2} \\
  &&f(y)  \leq  f(x) + \inp{\nabla f(x)}{y-x} + \frac{L_0 + \frac{L_1}{n}\sum_{i=1}^n \norm{\nabla f_i(x)}}{2}\exp\left(L_1\norm{x-y}\right)\sqnorm{y-x}.  \label{eqn:LzeroLoneSmooth_ineq_3} 
\end{eqnarray}	 	
\end{lemma}
\begin{proof}
The first and second statements are derived in \citet[Proposition 3.2]{chen2023generalized}\revision{. Next, the third inequality follows from \citep[Lemma 2.2]{gorbunov2024methods}. Finally, averaging \eqref{eqn:LzeroLoneSmooth_ineq_1} for $i=1,\ldots,n$ and taking into account that $f(x) = \frac{1}{n}\sum_{i=1}^n f_i(x)$, we get \eqref{eqn:LzeroLoneSmooth_ineq_3}.} 
\end{proof}

\begin{lemma}\label{lemma:bound_avg_norm_grad_multinodeef21}
Let $f_i(x)$ be generalized smooth with parameters $L_0,L_1>0$, and lower bounded by $f^{\inf}_{i}$, and let $f(x)$ be lower bounded by $f^{\inf}$. Then, for any $x\in\R^d$
\begin{eqnarray}
  \frac{1}{n}\sum_{i=1}^n \norm{\nabla f_i(x)}  \leq 8L_1 (f(x)-f^{\inf}) + \frac{8L_1}{n}\sum_{i=1}^n (f^{\inf}-f^{\inf}_{i}) + L_0/L_1. \label{eq:sum_of_grad_norms_bound}
\end{eqnarray}
\end{lemma}
\begin{proof}
  By the $(L_0,L_1)$-smoothness of $f_i(x)$, 
\begin{eqnarray*}
    4(f_i(x)-f^{\inf}_{i}) \overset{\eqref{eqn:LzeroLoneSmooth_ineq_2}}{\geq} \frac{\sqnorm{\nabla f_i(x)}}{L_0+L_1\norm{\nabla f_i(x)}} \geq \begin{cases} 
    \frac{\sqnorm{\nabla f_i(x)}}{2L_0} \quad \text{if } \norm{\nabla f_i(x)} \leq \frac{L_0}{L_1} \\ 
    \frac{\norm{\nabla f_i(x)}}{2L_1} \quad \text{otherwise}.
    \end{cases}
\end{eqnarray*}  
This condition \revision{implies}
% can be equivalently expressed as
\begin{eqnarray*}
    \norm{\nabla f_i(x)} 
    & \leq & \max( 8L_1 (f_i(x)-f^{\inf}_{i}), L_0/L_1) \\
    & \leq &  8L_1 (f_i(x)-f^{\inf}_{i}) + L_0/L_1 \\
    & \leq &  8L_1 (f_i(x)-f^{\inf}) + 8L_1(f^{\inf}-f^{\inf}_{i})  + L_0/L_1.
\end{eqnarray*}
Finally, by the fact that $f(x)=\frac{1}{n}\sum_{i=1}^n f_i(x)$, 
\begin{eqnarray*}
    \frac{1}{n}\sum_{i=1}^n \norm{\nabla f_i(x)}  \leq 8L_1 (f(x)-f^{\inf}) + \frac{8L_1}{n}\sum_{i=1}^n (f^{\inf}-f^{\inf}_{i}) + L_0/L_1.
\end{eqnarray*}    
\end{proof}

\begin{lemma}\label{lemma:descent_ineq_generalized_hubler}
 \revision{Let} $f(x) = \frac{1}{n}\sum_{i=1}^n f_i(x)$, where each $f_i(x)$ is  generalized smooth with parameters $L_0,L_1 >0$. 
Let $x^{k+1}=x^k - \frac{\gamma_k}{\norm{v^k}} v^k$ for $\gamma_k>0$. Then,  
\begin{eqnarray*}
    f(x^{k+1}) 
  &\leq& f(x^k) - \gamma_k \norm{\nabla f(x^k)} + 2\gamma_k \norm{\nabla f(x^k)-v^k}  \\ 
  &&+\frac{\gamma_k^2}{2}\exp\left(\gamma_k L_1\right) \left(L_0 + \frac{L_1}{n}\sum_{i=1}^n \norm{\nabla f_i(x^k)}\right). 
\end{eqnarray*}   
\end{lemma}
\begin{proof}
Let each $f_i(x)$ be  generalized smooth with $L_0,L_1>0$, and $f(x)=\frac{1}{n}\sum_{i=1}^n f_i(x)$. By \eqref{eqn:LzeroLoneSmooth_ineq_3} of Lemma~\ref{lemma:LzeroLoneSmooth}, and by the fact that $x^{k+1}=x^k - \frac{\gamma_k}{\norm{v^k}} v^k$ for $\gamma_k>0$,
%\begin{eqnarray*}
%    f(y)  \leq  f(x) + \inp{\nabla f(x)}{y-x} + \frac{L_0 + \frac{L_1}{n}\sum_{i=1}^n \norm{\nabla f_i(x)}}{2}\exp\left(L_1\norm{x-y}\right)\sqnorm{y-x}.
%\end{eqnarray*}
\begin{eqnarray*}
    f(x^{k+1}) 
    & \leq & f(x^k) - \frac{\gamma_k}{\norm{v^k}} \langle \nabla f(x^k) , v^k \rangle + \frac{\gamma_k^2}{2}\exp(\gamma_k L_1) \left( L_0 + \frac{L_1}{n}\sum_{i=1}^n \norm{\nabla f_i(x^k)} \right) \\
    & = &  f(x^k) - \frac{\gamma_k}{\norm{v^k}} \langle \nabla f(x^k) - v^k , v^k \rangle - \gamma_k \norm{v^k} \\
    && +  \frac{\gamma_k^2}{2}\exp(\gamma_k L_1) \left( L_0 + \frac{L_1}{n}\sum_{i=1}^n \norm{\nabla f_i(x^k)} \right) \\
    & \leq & f(x^k) + \gamma_k \norm{\nabla f(x^k)-v^k}- \gamma_k \norm{v^k} \\
    && +  \frac{\gamma_k^2}{2}\exp(\gamma_k L_1) \left( L_0 + \frac{L_1}{n}\sum_{i=1}^n \norm{\nabla f_i(x^k)} \right),
\end{eqnarray*}
where we reach the last inequality by Cauchy-Schwarz inequality. 
Next, since 
\begin{eqnarray*}
    - \norm{v^k} \overset{\text{triangle ineq.}}{\leq} - \norm{\nabla f(x^k)} + \norm{\nabla f(x^k)-v^k},
\end{eqnarray*}
we get  
\begin{eqnarray*}
    f(x^{k+1}) 
  &\leq& f(x^k) - \gamma_k \norm{\nabla f(x^k)} + 2\gamma_k \norm{\nabla f(x^k)-v^k} \\
  && + \frac{\gamma_k^2}{2}\exp(\gamma_k L_1) \left( L_0 + \frac{L_1}{n}\sum_{i=1}^n \norm{\nabla f_i(x^k)} \right).
\end{eqnarray*}

\end{proof}

%Lemma~\ref{lemma:descent_ineq_generalized_hubler} provides the descent inequality in the objective function for the normalized gradient descent iteration. Also, when we let $n=1$, this lemma recovers Lemma 12 of \citet{hubler2024parameter}.

%Finally, we provide Lemma~\ref{lemma:conv_v2} that shows how to derive the convergence rate from the descent inequalities. 

%\begin{lemma}\label{lemma:conv_v1}
%Let $V^k,W^k$ be non-negative sequences that satisfy

%\end{lemma}
%\begin{proof}
    
%\end{proof}

\begin{lemma}\label{lemma:conv_v2}
Let \revision{$\{V^k\}_{k\geq 0}, \{W^k\}_{k\geq 0}$} be non-negative sequences satisfying
\begin{eqnarray*}
    V^{k+1} \leq (1+b_1\exp(L_1\gamma)\gamma^2)V^k -b_2 \gamma W^k + b_3 \exp(L_1\gamma) \gamma^2,
\end{eqnarray*}
for $\gamma,b_1,b_2,b_3>0$. Then, 
\begin{eqnarray*}
    \underset{k=0,1,\ldots,K}{\min} W^k
   \leq  \frac{V^0 \exp(b_1\exp(L_1\gamma)\gamma^2 (K+1))}{b_2\gamma (K+1)} + \frac{b_3}{b_2}\exp(L_1\gamma)\gamma.
\end{eqnarray*} 
\end{lemma}
\begin{proof}
Define $\beta_k = \frac{\beta_{k-1}}{1+b_1\exp(L_1\gamma)\gamma^2}$ for $k=0,1,\ldots$ and $\beta_{-1}=1$. Then, we can show that  $\beta_k=\frac{1}{(1+b_1\exp(L_1\gamma)\gamma^2)^{k+1}}$ for $k=0,1,\ldots$, and that 
\begin{eqnarray*}
    \beta_k V^{k+1} 
    & \leq & (1+b_1\exp(L_1\gamma)\gamma^2)\beta_kV^k -b_2 \gamma \beta_k W^k + b_3 \exp(L_1\gamma) \gamma^2\beta_k \\ 
    & = & \beta_{k-1 } V^k - b_2\gamma \beta_k W^k + b_3 \exp(L_1\gamma)\gamma^2 \beta_k.
\end{eqnarray*}
Therefore, 
\begin{eqnarray*}
    \underset{k=0,1,\ldots,K}{\min} W^k
    & \leq & \frac{1}{\sum_{k=0}^K \beta_k} \sum_{k=0}^K \beta_k W^k \\ 
    & \leq & \frac{\sum_{k=0}^K (\beta_{k-1}V^k - \beta_k V^{k+1})}{b_2\gamma \sum_{k=0}^K \beta_k} + \frac{b_3}{b_2}\exp(L_1\gamma)\gamma  \\
    & = & \frac{\beta_{-1}V^0 - \beta_K V^{k+1}}{b_2\gamma \sum_{k=0}^K \beta_k} + \frac{b_3}{b_2}\exp(L_1\gamma)\gamma.
\end{eqnarray*}
By the fact that $\beta_{-1}=1$, $\beta_K >0$, and $V^{k+1} \geq 0$,
\begin{eqnarray*}
    \underset{k=0,1,\ldots,K}{\min} W^k
    \leq \frac{V^0}{b_2\gamma \sum_{k=0}^K \beta_k} + \frac{b_3}{b_2}\exp(L_1\gamma)\gamma.
\end{eqnarray*}
Next, since
\begin{eqnarray*}
    \sum_{k=0}^K \beta_k 
     \geq  (K+1) \underset{k=0,1,\ldots,K}{\min} \beta_k = \frac{K+1}{(1+b_1\exp(L_1\gamma)\gamma^2)^{K+1}},
\end{eqnarray*}
we have 
\begin{eqnarray*}
    \underset{k=0,1,\ldots,K}{\min} W^k
    & \leq & \frac{V^0 (1+b_1\exp(L_1\gamma)\gamma^2)^{K+1}}{b_2\gamma (K+1)} + \frac{b_3}{b_2}\exp(L_1\gamma)\gamma  \\
    & \overset{1+x \leq \exp(x)}{\leq} & \frac{V^0 \exp(b_1\exp(L_1\gamma)\gamma^2 (K+1))}{b_2\gamma (K+1)} + \frac{b_3}{b_2}\exp(L_1\gamma)\gamma.
\end{eqnarray*} 
    
\end{proof}

\clearpage

\section{Convergence Proof for \algname{\large||EF21||} (Theorem~\ref{thm:ef21})}

In this section, we derive the convergence rate results of \algname{||EF21||}. \revision{We start with the following lemma technical lemma.}
\begin{lemma}\label{lemma:ef21_bound_for_diff}
    \revision{Let Assumptions~\ref{assum:LzeroLoneSmooth} and \ref{assum:contractive_comp} hold. Then, the iterates $\{x^k\}$ generated by \algname{||EF21||} (Algorithm~\ref{alg:normalized_ef21}) satisfy
    \begin{eqnarray}
    \Exp{\norm{\nabla f_i(x^{k+1}) - g_i^{k+1}}} 
&\leq&   \sqrt{1-\alpha}\Exp{\norm{\nabla f_i(x^k) - g_i^k}} \notag \\
 && + \sqrt{1-\alpha}\exp(L_1\gamma_k)\gamma_k(L_0 + L_1\Exp{\norm{\nabla f_i(x^k)}}). \label{eqn:bound_variance_multief21_di_2}
\end{eqnarray}
}
\end{lemma}
\begin{proof}
    From the definition of the Euclidean norm, and by taking the expectation conditioned on \revision{$x^{k+1},g_i^{k}$, and by the update of $g_i^k$} from Algorithm~\ref{alg:normalized_ef21}  
\revision{    
\begin{eqnarray*}
    && \ExpCond{\norm{\nabla f_i(x^{k+1}) - g_i^{k+1}}}{x^{k+1},g_i^k} \\
  && =    \ExpCond{ \norm{\nabla f_i(x^{k+1}) - g_i^k - \cC^k(\nabla f_i(x^{k+1}) - g_i^k )} }{x^{k+1},g_i^k} \\
    && \leq  \sqrt{ \ExpCond{ \sqnorm{\nabla f_i(x^{k+1}) - g_i^k - \cC(\nabla f_i(x^{k+1}) - g_i^k )} }{x^{k+1},g_i^k}},
\end{eqnarray*}}
where we use the concavity of the square root function, and Jensen's inequality for the concave function, i.e., $\Exp{f(x)} \leq f(\Exp{x})$ if $f(x)$ is concave. By the $\alpha$-contractive property of compressors in~\eqref{eqn:contractive_comp}, by the fact that \revision{ 
$\norm{\nabla f_i(x^{k+1}) - g_i^k}$ is a constant conditioned on $x^{k+1},g_i^k$,} and then by the triangle inequality, we have
\revision{
\begin{eqnarray*}
 \ExpCond{\norm{\nabla f_i(x^{k+1}) - g_i^{k+1}}}{x^{k+1},g_i^k} 
    & \leq &  \sqrt{(1-\alpha) \ExpCond{ \sqnorm{\nabla f_i(x^{k+1}) - g_i^k} }{x^{k+1},g_i^k}  } \\
    & = &  \sqrt{1-\alpha}  \norm{\nabla f_i(x^{k+1}) - g_i^k} \\
 & \leq &      \sqrt{1-\alpha}\norm{\nabla f_i(x^k) - g_i^k} + \sqrt{1-\alpha}\norm{\nabla f_i(x^{k+1})-\nabla f_i(x^k)}. \\  
%     &\leq &  (1-\alpha)\norm{\nabla f(x^{k+1}) - v^k } \\
%     & \overset{\text{Triangle ineq.}}{\leq}&      (1-\alpha)\norm{\nabla f(x^k) - v^k} + (1-\alpha)\norm{\nabla f(x^{k+1})-\nabla f(x^k)} \\
 %    & \overset{\eqref{eqn:LzeroLoneSmooth}}{\leq}&  (1-\alpha)\norm{\nabla f(x^k) - v^k} + (1-\alpha)(L_0 + L_1\norm{\nabla f(x^k)})\norm{ x^{k+1} - x^k}.
\end{eqnarray*}
}
%where we reach the second inequality by the fact that the square root function is concave, and the last inequality by the fact that $\norm{\nabla f_i(x^{k+1}) - v_i^k}$ is a constant conditioned on $x^{k+1},v_i^k$. 

By the generalized smoothness of $f_i(x)$ in \eqref{eqn:LzeroLoneSmooth}, and by the fact that \revision{$x^{k+1}=x^k-\gamma_k \frac{g^k}{\norm{g^k}}$},
\revision{
\begin{eqnarray*}
    \ExpCond{\norm{\nabla f_i(x^{k+1}) - g_i^{k+1}}}{x^{k+1},g_i^k}
   % & \overset{\eqref{eqn:LzeroLoneSmooth}}{\leq}&  \sqrt{1-\alpha}\norm{\nabla f_i(x^k) - v_i^k} + \sqrt{1-\alpha}(L_0 + L_1\norm{\nabla f_i(x^k)})\norm{ x^{k+1} - x^k} \\
& \leq&  \sqrt{1-\alpha}\norm{\nabla f_i(x^k) - g_i^k} \\
 && + \sqrt{1-\alpha}(L_0 + L_1\norm{\nabla f_i(x^k)})\exp(L_1\gamma_k)\gamma_k.
\end{eqnarray*}
}

Let $\gamma_k>0$ be constants conditioned on \revision{$x^{k+1},g_i^k$}. Then, by the tower property, i.e., \revision{
\begin{eqnarray*}
\Exp{\norm{\nabla f_i(x^{k+1}) - g_i^{k+1}}} =  \Exp{ \ExpCond{\norm{\nabla f_i(x^{k+1}) - g_i^{k+1}}}{x^{k+1},g_i^k} },
\end{eqnarray*}
we have 
\begin{eqnarray}
    \Exp{\norm{\nabla f_i(x^{k+1}) - g_i^{k+1}}} 
%& = &    \Exp{ \ExpCond{\norm{\nabla f_i(x^{k+1}) - v_i^{k+1}}}{x^{k+1},v_i^k} } \notag \\
 &\leq&   \sqrt{1-\alpha}\Exp{\norm{\nabla f_i(x^k) - g_i^k}} \notag \\
 && + \sqrt{1-\alpha}\exp(L_1\gamma_k)\gamma_k(L_0 + L_1\Exp{\norm{\nabla f_i(x^k)}}). \notag
\end{eqnarray}
}
\revision{This concludes the proof.}
\end{proof}

\revision{Next,} we present the following descent lemma for \algname{||EF21||}. 

\begin{lemma}\label{lemma:di_multinodeef21}
  % Consider Problem~(\ref{eqn:Problem}), where Assumption~\ref{assum:lowerbound_whole_f} (lower bound on $f$), Assumption~\ref{assum:lowerbound_f_i} (lower bound on $f_i$), Assumption~\ref{assum:LzeroLoneSmooth} (generalized smooth of $f_i$), and Assumption~\ref{assum:contractive_comp} ($\alpha$-contractive property of $\cC^k$) hold.
  \revision{Let Assumptions~\ref{assum:lowerbound_whole_f}-\ref{assum:contractive_comp} hold.} 
  Then, the iterates $\{x^k\}$ generated by \algname{||EF21||} (Algorithm~\ref{alg:normalized_ef21}) satisfy 
\begin{eqnarray*}
\Exp{V^{k+1}} 
 \leq  \Exp{V^k}  +  c_1 \gamma_k^2\frac{1}{n}\sum_{i=1}^n \Exp{\norm{\nabla f_i(x^k)}} -   \gamma_k  \Exp{ \norm{\nabla f(x^k)} }  + c_0\gamma_k^2,
\end{eqnarray*}
where $V^k \eqdef f(x^k)-f^{\inf} + \frac{2\gamma_k}{1-\sqrt{1-\alpha}} \frac{1}{n}\sum_{i=1}^n \norm{\nabla f_i(x^k)-\revision{g_i^k}}$, and  $c_i = \frac{L_i}{2} + 2\frac{\sqrt{1-\alpha} L_i}{1-\sqrt{1-\alpha}}$ for $i=0,1$.  
\end{lemma}
\begin{proof}
\revision{For brevity, let $A_k = \frac{2\gamma_k}{1-\sqrt{1-\alpha}}$. Then, we have} $V^k \eqdef f(x^k)-f^{\inf} + A_k \frac{1}{n}\sum_{i=1}^n \norm{\nabla f_i(x^k)-v_i^k}$, \revision{and} from Lemma~\ref{lemma:descent_ineq_generalized_hubler}, we derive
\begin{eqnarray*}
\Exp{V^{k+1}} 
%& = & \Exp{f(x^{k+1})-f^{\inf}} + A_{k+1} \frac{1}{n}\sum_{i=1}^n \Exp{\norm{\nabla f_i(x^{k+1})-v_i^{k+1}}} \\
& \leq & \Exp{f(x^k)-f^{\inf}} - \gamma_k \Exp{\norm{\nabla f(x^k)}}  \\
&& + \exp(L_1\gamma_k)\gamma_k^2 \frac{L_1}{2n}\sum_{i=1}^n \Exp{\norm{\nabla f_i(x^k)}}+ \exp(L_1\gamma_k)\gamma_k^2\frac{L_0}{2} \\
&& + 2\gamma_k \Exp{\norm{\nabla f(x^k)- \revision{g^k}}} + A_{k+1} \frac{1}{n}\sum_{i=1}^n \Exp{\norm{\nabla f_i(x^{k+1}) - \revision{g_i^{k+1}}}}.
\end{eqnarray*}
\revision{Identities} $\nabla f(x^k) = \frac{1}{n}\sum_{i=1}^n \nabla f_i(x^k)$ and \revision{$g^k = \frac{1}{n}\sum_{i=1}^n g_i^k$} and the triangle inequality \revision{imply}
\begin{eqnarray*}
\Exp{V^{k+1}} 
%& = & \Exp{f(x^{k+1})-f^{\inf}} + A_{k+1} \frac{1}{n}\sum_{i=1}^n \Exp{\norm{\nabla f_i(x^{k+1})-v_i^{k+1}}} \\
& \leq & \Exp{f(x^k)-f^{\inf}} - \gamma_k \Exp{\norm{\nabla f(x^k)}} \\
&& + \exp(L_1\gamma_k)\gamma_k^2 \frac{L_1}{2n}\sum_{i=1}^n \Exp{\norm{\nabla f_i(x^k)}}+ \exp(L_1\gamma_k)\gamma_k^2\frac{L_0}{2} \\
&& \revision{+ 2\gamma_k \frac{1}{n}\sum_{i=1}^n \Exp{\norm{\nabla f_i(x^{k}) - g_i^{k}}} +  A_{k+1} \frac{1}{n}\sum_{i=1}^n \Exp{\norm{\nabla f_i(x^{k+1}) - g_i^{k+1}}}.}
\end{eqnarray*}
Next, \revision{we apply} \eqref{eqn:bound_variance_multief21_di_2}:
\begin{eqnarray*}
\Exp{V^{k+1}}  & \leq  & \Exp{f(x^k) - f^{\inf}} - \gamma_k \Exp{\norm{\nabla f(x^k)}} + \left( \frac{\gamma_k^2 }{2} + A_{k+1} \sqrt{1-\alpha}  \gamma_k \right)\exp(L_1\gamma_k)  L_0   \\
&& + \left( \frac{\gamma_k^2}{2} + A_{k+1} \sqrt{1-\alpha} \gamma_k \right) \exp(L_1\gamma_k) L_1 \frac{1}{n}\sum_{i=1}^n \Exp{\norm{\nabla f_i(x^k)}} \\
&&  \revision{+  \left( 2\gamma_k + A_{k+1} \sqrt{1-\alpha} \right) \frac{1}{n}\sum_{i=1}^n \Exp{\norm{\nabla f_i(x^k) - g_i^k}}.}
\end{eqnarray*}

If $A_k = \frac{2\gamma_k}{1-\sqrt{1-\alpha}}$, and $\gamma_k$ satisfies $\gamma_{k+1} \leq \gamma_k$, then 
\begin{eqnarray*}
    2\gamma_k + A_{k+1} \sqrt{1-\alpha} \leq 2\gamma_k + A_k\sqrt{1-\alpha} = A_k. 
\end{eqnarray*}
Therefore, 
%\begin{eqnarray*}
%\Exp{V^{k+1}} 
% \leq  \Exp{V^k}  +  \frac{2\sqrt{1-\alpha} L_1}{1-\sqrt{1-\alpha}} \gamma_k^2\frac{1}{n}\sum_{i=1}^n \Exp{\norm{\nabla f_i(x^k)}} -  \left( \gamma_k  - \frac{\gamma_k^2 L_1}{2} \right) \Exp{ \norm{\nabla f(x^k)} }  + c_0\gamma_k^2,
%\end{eqnarray*}
%where $c_i = \frac{L_i}{2} + 2\frac{\sqrt{1-\alpha} L_i}{1-\sqrt{1-\alpha}}$ for $i=0,1$. %Next, by the fact that $f(x)=\frac{1}{n}\sum_{i=1}^n f_i(x)$, and by the triangle inequality,    
\begin{eqnarray*}
\Exp{V^{k+1}} 
 &\leq&  \Exp{V^k}  +  c_1 \exp(L_1\gamma_k)\gamma_k^2\frac{1}{n}\sum_{i=1}^n \Exp{\norm{\nabla f_i(x^k)}} \\
 && -   \gamma_k  \Exp{ \norm{\nabla f(x^k)} }  + c_0\exp(L_1\gamma_k)\gamma_k^2,
\end{eqnarray*}
where $c_i = \frac{L_i}{2} + 2\frac{\sqrt{1-\alpha} L_i}{1-\sqrt{1-\alpha}}$ for $i=0,1$.
\end{proof}

\subsection{\revision{Proof of} Theorem~\ref{thm:ef21}}
Now, we are ready to prove Theorem~\ref{thm:ef21}. 
From Lemma~\ref{lemma:di_multinodeef21} and~\ref{lemma:bound_avg_norm_grad_multinodeef21}, and by the fact that $c_1 L_0/L_1 = c_0$\revision{, we have}
\begin{eqnarray*}
\Exp{V^{k+1}} 
 &\leq &  \Exp{V^k}  + 8c_1 L_1 \exp(L_1\gamma_k) \gamma_k^2 \Exp{f(x^k)-f^{\inf}} \\
 && -  \gamma_k  \Exp{ \norm{\nabla f(x^k)} }  + B\exp(L_1\gamma_k)\gamma_k^2,
\end{eqnarray*}
where $B = 2c_0 +  \frac{8c_1L_1}{n}\sum_{i=1}^n (f^{\inf}-f^{\inf}_{i})$. \revision{Using} the fact that $f(x^k)-f^{\inf} \leq V^k$, \revision{we derive}
\begin{eqnarray*}
\Exp{V^{k+1}} 
 \leq  (1+8c_1 L_1\exp(L_1\gamma_k)\gamma_k^2)\Exp{V^k}  -  \gamma_k  \Exp{ \norm{\nabla f(x^k)} }  + B\exp(L_1\gamma_k)\gamma_k^2.
\end{eqnarray*}
Applying Lemma~\ref{lemma:conv_v2} with $V^k = \Exp{V^k}$, $W^k = \Exp{\norm{\nabla f(x^k)}}$, $b_1 = 8c_1L_1$, $b_2 = 1$, and $b_3 = B$, \revision{we get}
\begin{eqnarray*}
    \underset{k=0,1,\ldots,K}{\min} W^k
   \leq  \frac{V^0 \exp(b_1\exp(L_1\gamma)\gamma^2 (K+1))}{b_2\gamma (K+1)} + \frac{b_3}{b_2}\exp(L_1\gamma)\gamma.
\end{eqnarray*} 

Finally,  if $\gamma = \frac{\gamma_0}{\sqrt{K+1}}$ with $\gamma_0 >0$, then $\exp(L_1\gamma_k) \leq \exp(L_1\gamma_0)$, and thus
\begin{eqnarray*}
    \underset{k=0,1,\ldots,K}{\min} W^k
   \leq  \frac{V^0 \exp(b_1 \exp(L_1\gamma_0)\gamma_0^2)}{b_2\gamma_0 \sqrt{K+1}} + \frac{b_3}{b_2}\frac{\gamma_0 \exp(L_1\gamma_0)}{\sqrt{K+1}}.
\end{eqnarray*}

\clearpage

\section{Convergence of \algname{\large||EF21||} for a Single-node Case}
In this section, we provide the convergence of \algname{||EF21||} for a single-node case. 
In particular, the algorithm enjoys the $\cO(1/K)$ convergence up to the \revision{error of} $\frac{c_0\gamma}{1-c_1\exp(L_1\gamma)\gamma}$. In contrast to Theorem~\ref{thm:ef21} for multi-node \algname{||EF21||}, the next result for single-node \algname{||EF21||} applies for any $\gamma_k = \gamma \in (0,1/(\beta c_1))$ with $\beta \geq 2$,  $c_1= \frac{L_1}{2} + 2\frac{\sqrt{1-\alpha} L_1}{1-\sqrt{1-\alpha}}$, and $\alpha \in (0,1]$. 

\begin{theorem}\label{thm:normalized_ef21_singlenode}
% Consider the problem of minimizing $f(x)$, which satisfies Assumption~\ref{assum:lowerbound_whole_f} (lower bound on $f$), and Assumption~\ref{assum:LzeroLoneSmooth} (generalized smoothness of $f$). Further, let Assumption~\ref{assum:contractive_comp} (contractive compressor) hold.
\revision{Let Assumptions~\ref{assum:lowerbound_whole_f}-\ref{assum:contractive_comp} hold.} 
Then, the iterates $\{x^k\}$ generated by \algname{||EF21||} (Algorithm~\ref{alg:normalized_ef21}) with $n=1$, $\gamma_k=\gamma = 1/(\beta c_1)$ and $\beta \geq 2$ satisfy   
\begin{eqnarray*}
     \underset{k=0,1,\ldots,K}{\min} \Exp{\norm{\nabla f(x^k)}} \leq \frac{\Exp{V^0}-\Exp{V^{K+1}}}{\gamma(1-c_1\exp(L_1\gamma)\gamma)(K+1)} + \frac{c_0\gamma}{1-c_1\exp(L_1\gamma)\gamma},
\end{eqnarray*}
where $V^k = f(x^k)-f^{\inf} + \frac{2\gamma}{1-\sqrt{1-\alpha}} \norm{\nabla f(x^k)- \revision{g^k}}$, and $c_i = \frac{L_i}{2} + 2\frac{\sqrt{1-\alpha} L_i}{1-\sqrt{1-\alpha}}$ for $i=0,1$.  
\end{theorem}
\begin{proof}
\revision{In the single-node case, Lemma~\ref{lemma:ef21_bound_for_diff} implies}
\begin{eqnarray}
    \Exp{\norm{\nabla f(x^{k+1}) - \revision{g^{k+1}} }} 
%& = &    \Exp{ \ExpCond{\norm{\nabla f_i(x^{k+1}) - v_i^{k+1}}}{x^{k+1},v_i^k} } \notag \\
 &\leq&   \sqrt{1-\alpha}\Exp{\norm{\nabla f(x^k) - \revision{g^k} }} \notag \\
 && + \sqrt{1-\alpha}\exp(L_1\gamma_k)\gamma_k(L_0 + L_1\Exp{\norm{\nabla f(x^k)}}). \label{eqn:bound_variance_ef21_di_2}
\end{eqnarray}
% \paragraph{Step 2) Bound $V^k \eqdef f(x^k)-f^{\inf} + A_k \norm{\nabla f(x^k)-v^k}$ for some $A_k>0$.} Denote $V^k \eqdef f(x^k)-f^{\inf} + A_k \norm{\nabla f(x^k)-v^k}$ for some constants $A_k>0$. Then, from the definition of $V^{k+1}$, from  Lemma~\ref{lemma:descent_ineq_generalized_hubler} with $n=1$, and by the fact $f(x)$ is  generalized smooth, 
\revision{Next, for brevity, let $A_k = \frac{2\gamma_k}{1-\sqrt{1-\alpha}}$. Then, we have $V^k \eqdef f(x^k)-f^{\inf} + A_k \frac{1}{n}\sum_{i=1}^n \norm{\nabla f_i(x^k)-g_i^k}$, \revision{and} from Lemma~\ref{lemma:descent_ineq_generalized_hubler}, we derive}
\begin{eqnarray*}
\Exp{V^{k+1}} 
& \leq & \Exp{f(x^k) - f^{\inf}}  - \left( \gamma_k  - \frac{\gamma_k^2 L_1}{2}\exp(L_1\gamma_k) \right)\Exp{\norm{\nabla f(x^k)}}+ \frac{\gamma_k^2 L_0}{2}\exp(L_1\gamma_k)  \\
&& \revision{+  2\gamma_k \Exp{\norm{\nabla f(x^k) - g^k}} + A_{k+1} \Exp{\norm{\nabla f(x^{k+1}) - g^{k+1}}} } \\
& \overset{\eqref{eqn:bound_variance_ef21_di_2}}{\leq} & \Exp{f(x^k) - f^{\inf}}  +  \revision{\left( 2\gamma_k + A_{k+1} \sqrt{1-\alpha} \right) \Exp{\norm{\nabla f(x^k) - g^k}}} \\
&& - \left( \gamma_k  - \frac{\gamma_k^2 L_1}{2}\exp(L_1\gamma_k) -  A_{k+1} \sqrt{1-\alpha}L_1 \gamma_k \exp(L_1\gamma_k) \right)\Exp{\norm{\nabla f(x^k)}} \\
&& + \frac{\gamma_k^2 L_0}{2} \exp(L_1\gamma_k) + A_{k+1} \sqrt{1-\alpha} L_0 \gamma_k \exp(L_1\gamma_k). 
\end{eqnarray*}

If $A_k = \frac{2\gamma_k}{1-\sqrt{1-\alpha}}$ and $\gamma_k$ satisfies $\gamma_{k+1} \leq \gamma_k$, then 
\begin{eqnarray*}
    2\gamma_k + A_{k+1} \sqrt{1-\alpha} \leq 2\gamma_k + A_k\sqrt{1-\alpha} = A_k. 
\end{eqnarray*}
Therefore, 
\begin{eqnarray*}
\Exp{V^{k+1}} 
 \leq  \Exp{V^k}  - \left( \gamma_k  - c_1\exp(L_1\gamma_k)\gamma_k^2 \right)\Exp{ \norm{\nabla f(x^k)} } + c_0\exp(L_1\gamma_k)\gamma_k^2,
\end{eqnarray*}
where $c_i = \frac{L_i}{2}  + 2\frac{\sqrt{1-\alpha} L_i}{1-\sqrt{1-\alpha}}$ for $i=0,1$.

\revision{Finally, taking} $\gamma_k = \gamma = \nicefrac{1}{(\beta c_1)}$ for $\beta \geq 2$, \revision{we get} $c_1 \exp(L_1\gamma) \gamma = 
\nicefrac{\exp(L_1/(\beta c_1))}{\beta} \leq \exp(2/\beta)/\beta \leq 0.7 < 1$, and 
\begin{eqnarray*}
\Exp{V^{k+1}} 
 & \leq &   \Exp{V^k}  - \gamma \left( 1  - c_1 \exp(L_1\gamma)\gamma \right)\Exp{ \norm{\nabla f(x^k)} } + c_0\gamma^2.
\end{eqnarray*}
\revision{Rearranging} the terms, \revision{we derive}
\begin{eqnarray*}
    \underset{k= 0,1,\ldots,K}{\min} \Exp{\norm{\nabla f(x^k)}} 
    & \leq & \frac{1}{K+1} \sum_{k=0}^K \Exp{\norm{\nabla f(x^k)}} \\
    & \leq & \frac{\Exp{V^0}-\Exp{V^{K+1}}}{\gamma(1-c_1\exp(L_1\gamma)\gamma)(K+1)} + \frac{c_0\gamma}{1-c_1\exp(L_1\gamma)\gamma}.
\end{eqnarray*}
\revision{Noticing that} $V^k \geq 0$, we complete the proof. 
\end{proof}

\newpage 
\section{Convergence of \algname{\large||EF21-SGDM||}  (Theorem~\ref{thm:ef21_sgdm})}

%\abdrk{I need to change all statements and proofs, since we changed main Assumption.}
%\sarit{I will read the proof in this section.}

In this section, we derive the convergence rate results of \algname{||EF21-SGDM||} . We first introduce auxiliary lemmas in Section~\ref{app:sec:lemma_ef21_sgdm}, and later prove the convergence theorem (Theorem~\ref{thm:ef21_sgdm}) in Section~\ref{app:sec:ef21_sgdm}. 

%\begin{theorem}
    
%\end{theorem}
%\begin{proof}
    
%\end{proof}

\subsection{Auxiliary Lemmas}\label{app:sec:lemma_ef21_sgdm}

Now, we provide useful lemmas for analyzing \algname{||EF21-SGDM||}. First, Lemma~\ref{lem:descent_lemma_ef21_sgdm} shows the descent inequality of the normalized gradient descent update under Assumption~\ref{assum:LzeroLoneSmooth} (generalized smoothness of $f_i$). 
Second, Lemmas~\ref{lem:compression_error_ef21_sgdm} and \ref{lem:momentum_error_ef21_sgdm} provide the upper-bound of the Euclidean distance between $v_i^{k}$ and $g_i^k$, and of the Euclidean distance between $v_i^k$ and $\nabla f_i(x^k)$, respectively. 
% Third, Lemma~\ref{lem:convergence_bound_ef21_sgdm} presents the convergence rate from the recursion of the non-negative sequences $r^k,s^k$.

\begin{lemma}
\label{lem:descent_lemma_ef21_sgdm}
    Consider the iterates $\{x^k\}$ generated by  Algorithm~\ref{alg:normalized_ef21_sgdm}. If 
 Assumption~\ref{assum:LzeroLoneSmooth} holds, then for any $\gamma_k >0, \eta_k \in [0, 1]$,
    \begin{eqnarray}
        f(x^{k+1}) &\leq& f(x^k) -\gamma_k \norm{\nabla f(x^k)} +2\gamma_k\norm{\nabla f(x^k) -v^k} +2\gamma_k \norm{v^k -g^k} \notag\\
        && \revision{+ L_0\gamma_k^2\exp(\gamma_k L_1) + 4L_1^2\gamma_k^2\exp(\gamma_k L_1)\left(f(x^k) - f^{\inf}\right)} \notag\\
        && \revision{+ \frac{4L_1^2\gamma_k^2\exp(\gamma_k L_1)}{n}\sum_{i=1}^n\left(f^{\inf} - f_i^{\inf}\right).} \notag
        % &&+ \frac{\gamma^2_k}{2}\exp\left(\gamma_kL_1\right) \left(L_0+ \frac{L_1}{n}\sum^n_{i=1}\norm{\nabla f_i (x^k)}\right). \notag
    \end{eqnarray}
\end{lemma}
\begin{proof}
    Applying the triangle inequality \revision{in} Lemma~\ref{lemma:descent_ineq_generalized_hubler}\revision{, i.e., $\norm{\nabla f(x^k) -g^k} \leq \norm{\nabla f(x^k) -v^k} + \norm{v^k - g^k}$, we get
    \begin{eqnarray*}
        f(x^{k+1}) &\leq& f(x^k) -\gamma_k \norm{\nabla f(x^k)} +2\gamma_k\norm{\nabla f(x^k) -v^k} +2\gamma_k \norm{v^k -g^k} \\
        &&+ \frac{\gamma^2_k}{2}\exp\left(\gamma_kL_1\right) \left(L_0+ \frac{L_1}{n}\sum^n_{i=1}\norm{\nabla f_i (x^k)}\right) \\
        &\overset{\eqref{eq:sum_of_grad_norms_bound}}{\leq}& f(x^k) -\gamma_k \norm{\nabla f(x^k)} +2\gamma_k\norm{\nabla f(x^k) -v^k} +2\gamma_k \norm{v^k -g^k} \\
        && + L_0\gamma_k^2\exp(\gamma_k L_1) + 4L_1^2\gamma_k^2\exp(\gamma_k L_1)\left(f(x^k) - f^{\inf}\right)\\
        && + \frac{4L_1^2\gamma_k^2\exp(\gamma_k L_1)}{n}\sum_{i=1}^n\left(f^{\inf} - f_i^{\inf}\right), 
    \end{eqnarray*}
    which concludes the proof.}
\end{proof}

\begin{lemma}
\label{lem:compression_error_ef21_sgdm}
Consider the iterates $\{x^k\}$ generated by Algorithm~\ref{alg:normalized_ef21_sgdm}. If Assumptions~\ref{assum:LzeroLoneSmooth},~\ref{assum:contractive_comp}, and~\ref{assum:bounded_variance} hold, then for $\gamma_k>0, \eta_k \in [0,1]$, and $k \geq 0$,
    \begin{align*}
        \frac{1}{n}\sum^n_{i=1}\Exp{\norm{v^{k+1}_i - g^{k+1}_i}} \leq& \frac{\sqrt{1-\alpha}}{n}\sum^n_{i=1}\Exp{\norm{v^{k}_i - g^{k}_i}} + \frac{\sqrt{1-\alpha}\eta_{k+1}}{n}\sum^n_{i=1}\Exp{\norm{v^{k}_i -\nabla f_i(x^k)}} \notag\\
        & \revision{+8L_1^2\sqrt{1-\alpha}\eta_{k+1}\gamma_k\exp\left(\gamma_k L_1\right) \Exp{f(x^k) - f^{\inf}}} \\
        & \revision{+\frac{8L_1^2\sqrt{1-\alpha}\eta_{k+1}\gamma_k\exp\left(\gamma_k L_1\right)}{n} \sum_{i=1}^n (f^{\inf}-f^{\inf}_{i})} \\
        & \revision{+ 2L_0\sqrt{1-\alpha}\eta_{k+1}\gamma_k\exp\left(\gamma_k L_1\right) + \sqrt{1-\alpha} \eta_{k+1} \sigma.}
        % & +\sqrt{1-\alpha}\eta_{k+1}\gamma_k\exp\left(\gamma_k L_1\right) \left(L_0 +  L_1\frac{1}{n}\sum^n_{i=1}\Exp{\norm{\nabla f_i(x^k)}} \right)\notag \\
        % & + \sqrt{1-\alpha} \eta_k \sigma. \notag
    \end{align*}
\end{lemma}
\begin{proof}
    Taking conditional expectation \revision{with fixed} $\gF_{k+1} = \{v^{k+1}_i, x^{k+1}, g^{k}_i\}$, using the \revision{concavity} of the squared root of the function, and applying the definition of $g_i^k$ in Algorithm~\ref{alg:normalized_ef21_sgdm}, we have 
\begin{eqnarray*}
        \ExpCond{\norm{v^{k+1}_i- g^{k+1}_i}}{\gF_{k+1}} &\leq& \sqrt{\ExpCond{\sqnorm{v^{k+1}_i- g^{k+1}_i}}{\gF_{k+1}}}\\
        &=&\sqrt{\ExpCond{\sqnorm{v^{k+1}_i- g^{k}_i - \gC^k\left(v^{k+1}_i- g^{k}_i\right) }}{\gF_{k+1}}}\\
        &\overset{\eqref{eqn:contractive_comp}}{\leq}&\sqrt{\ExpCond{(1-\alpha)\sqnorm{v^{k+1}_i- g^{k}_i  }}{\gF_{k+1}}}.
\end{eqnarray*}
Next, let $\gamma_k=\gamma>0$, and $\eta_k = \eta \in [0,1]$. By the fact that $v_i^{k+1},g_i^k$ are constants being conditioned on $\gF_{k+1}$, and by the triangle inequality,  
\begin{eqnarray*}
     \ExpCond{\norm{v^{k+1}_i- g^{k+1}_i}}{\gF_{k+1}}   % &\leq & \sqrt{1-\alpha}\norm{v^{k+1}_i- g^{k}_i}\\
        &\leq&\sqrt{1-\alpha}\norm{v^{k}_i- g^{k}_i} + \sqrt{1-\alpha}\norm{v^{k+1}_i- v^{k}_i}\\
        &=&\sqrt{1-\alpha}\norm{v^{k}_i- g^{k}_i} + \sqrt{1-\alpha}\eta_{k+1}\norm{\nabla f(x^{k+1};\xi^{k+1}_i) - v^k_i}. 
 \end{eqnarray*}
Here, the equality comes from the definition of $v_i^{k+1}$ in Algorithm~\ref{alg:normalized_ef21_sgdm}. Next, by the triangle inequality,
 \begin{eqnarray*}
    \ExpCond{\norm{v^{k+1}_i- g^{k+1}_i}}{\gF_{k+1}}        &\leq&\sqrt{1-\alpha}\norm{v^{k}_i- g^{k}_i} + \sqrt{1-\alpha} \eta_{k+1} \|v^k_i -\nabla f_i(x^k)\|\\
        && + \sqrt{1-\alpha}\eta_{k+1}\norm{\nabla f_i(x^{k}) - \nabla f_i(x^{k+1})}\\
        && + \sqrt{1-\alpha}\eta_{k+1}\norm{\nabla f_i(x^{k+1};\xi^{k+1}_i) - \nabla f_i(x^{k+1})}\\
        &\overset{\eqref{eqn:LzeroLoneSmooth_ineq_0}}{\leq}& \sqrt{1-\alpha}\norm{v^{k}_i- g^{k}_i} + \sqrt{1-\alpha} \eta_{k+1} \|v^k_i -\nabla f_i(x^k)\|\\
        && + \sqrt{1-\alpha}\eta_{k+1}\left(L_0 +L_1\norm{\nabla f_i(x^k)}\right)\exp\left(L_1\norm{x^{k+1}-x^k}\right)\norm{x^{k+1}-x^k}\\
        && + \sqrt{1-\alpha}\eta_{k+1}\norm{\nabla f(x^{k+1};\xi^{k+1}_i) - \nabla f(x^{k+1})}.
    \end{eqnarray*}
    Next, using $x^{k+1} -x^k = -\gamma_k \frac{g^k}{\|g^k\|}$, and taking the expectation, we obtain
    \begin{eqnarray*}
        \Exp{\norm{v^{k+1}_i- g^{k+1}_i}} &\leq& \sqrt{1-\alpha}\Exp{\norm{v^{k}_i- g^{k}_i}} + \sqrt{1-\alpha}\eta_{k+1}\Exp{\norm{v^k_i - \nabla f_i(x^k)}}\\
        && +\sqrt{1-\alpha}\eta_{k+1}\gamma_{k}\exp\left(\gamma_k L_1\right) \left( L_0 +  L_1\Exp{\norm{\nabla f_i(x^k)}}\right)\\
        && + \sqrt{1-\alpha}\eta_{k+1} \Exp{\norm{\nabla f_i(x^{k+1};\xi^{k+1}_i) - \nabla f_i(x^{k+1})}}.
    \end{eqnarray*}
    Finally, since 
    \begin{eqnarray*}
        \Exp{\norm{\nabla f_i(x^{k+1};\xi^{k+1}_i) - \nabla f_i(x^{k+1})}} 
        &\leq& \sqrt{ \Exp{\sqnorm{\nabla f_i(x^{k+1};\xi^{k+1}_i) - \nabla f_i(x^{k+1})}}} \\
        & \overset{\eqref{eqn:bounded_variance}}{\leq} & \sigma, 
    \end{eqnarray*}
    % we can obtain the upper bound for $\frac{1}{n}\sum_{i=1}^n   \Exp{\norm{v^{k+1}_i- g^{k+1}_i}}$.
    %To finish the proof, we apply \eqref{eqn:bounded_variance} and sum up by $i$. 
    \revision{we derive
    \begin{eqnarray*}
        \frac{1}{n}\sum^n_{i=1}\Exp{\norm{v^{k+1}_i - g^{k+1}_i}} &\leq& \frac{\sqrt{1-\alpha}}{n}\sum^n_{i=1}\Exp{\norm{v^{k}_i - g^{k}_i}}\\
        && + \frac{\sqrt{1-\alpha}\eta_{k+1}}{n}\sum^n_{i=1}\Exp{\norm{v^{k}_i -\nabla f_i(x^k)}}\\
        && +\sqrt{1-\alpha}\eta_{k+1}\gamma_k\exp\left(\gamma_k L_1\right) \left(L_0 +  L_1\frac{1}{n}\sum^n_{i=1}\Exp{\norm{\nabla f_i(x^k)}} \right) \\
        && + \sqrt{1-\alpha} \eta_{k+1} \sigma\\
        &\overset{\eqref{eq:sum_of_grad_norms_bound}}{\leq}& \frac{\sqrt{1-\alpha}}{n}\sum^n_{i=1}\Exp{\norm{v^{k}_i - g^{k}_i}}\\
        && + \frac{\sqrt{1-\alpha}\eta_{k+1}}{n}\sum^n_{i=1}\Exp{\norm{v^{k}_i -\nabla f_i(x^k)}}\\
        && +8L_1^2\sqrt{1-\alpha}\eta_{k+1}\gamma_k\exp\left(\gamma_k L_1\right) \Exp{f(x^k) - f^{\inf}} \\
        && +\frac{8L_1^2\sqrt{1-\alpha}\eta_{k+1}\gamma_k\exp\left(\gamma_k L_1\right)}{n} \sum_{i=1}^n (f^{\inf}-f^{\inf}_{i}) \\
        &&  + 2L_0\sqrt{1-\alpha}\eta_{k+1}\gamma_k\exp\left(\gamma_k L_1\right) + \sqrt{1-\alpha} \eta_{k+1} \sigma.
    \end{eqnarray*}
    This concludes the proof.}
\end{proof}

\begin{lemma}
\label{lem:momentum_error_ef21_sgdm}
Consider the iterates $\{x^k\}$ generated by Algorithm~\ref{alg:normalized_ef21_sgdm}. If  Assumptions~\ref{assum:LzeroLoneSmooth}, and~\ref{assum:bounded_variance} hold, then for any $\gamma_k \equiv \gamma >0$, $\eta_k \equiv \eta $, and  $k\geq 0$,
    \begin{eqnarray}
       \Exp{\norm{v^{k}-\nabla f(x^{k})}} &\leq& (1-\eta)^{k}\Exp{\norm{v^0 - \nabla f(x^0) }} +\frac{\sqrt{\eta}\sigma}{\sqrt{n}} + \revision{\frac{2L_0 \gamma \exp\left(\gamma L_1\right)}{\eta}}\notag\\
    && \revision{+ 8L_1^2 \gamma \exp\left(\gamma L_1\right)\sum^{k-1}_{t=0}(1-\eta)^{k-t} \Exp{f(x^t) - f^{\inf}}}\notag\\
    && \revision{+  \frac{8L_1^2 \gamma \exp\left(\gamma L_1\right)}{\eta n}\sum_{i=1}^n\left(f^{\inf} - f_i^{\inf}\right).} \label{eq:vk_nablafxk_bound}
    \end{eqnarray}
    % \begin{eqnarray}
    %     \Exp{\norm{v^{k}-\nabla f(x^{k})}} &\leq& (1-\eta)^{k}\Exp{\norm{v^0 - \nabla f(x^0) }} +\frac{\sqrt{\eta}\sigma}{\sqrt{n}}+ \frac{\gamma}{\eta}L_0\exp\left(\gamma L_1\right) \notag\\
    %     && +\exp\left(\gamma L_1\right)\frac{\gamma L_1}{n}\sum^{k-1}_{t=0}(1-\eta)^{k-t}\sum^n_{i=1}\Exp{\norm{\nabla f_i(x^t)}}. \notag
    % \end{eqnarray}
In addition,  for any $k\geq 0$,
    \revision{\begin{eqnarray}
        \frac{1}{n}\sum^n_{i=1}\Exp{\norm{v_i^{k+1}-\nabla f_i(x^{k+1})}} &\leq& \frac{1-\eta}{n}\sum^n_{i=1}\Exp{\norm{v^k_i - \nabla f_i(x^k) }} + \eta\sigma + 2L_0\gamma\exp\left(\gamma L_1\right)\notag\\
        && + 8L_1^2\gamma\exp(\gamma L_1)\Exp{f(x^k) - f^{\inf}}\notag \\
    && + \frac{8L_1^2\gamma\exp(\gamma L_1)}{n}\sum_{i=1}^n\left(f^{\inf} - f_i^{\inf}\right).\label{eq:sum_vik_nablafixk_bound}
    \end{eqnarray}}
    % \begin{eqnarray}
    %     \frac{1}{n}\sum^n_{i=1}\Exp{\norm{v^{k}_i-\nabla f_i(x^{k})}} &\leq& \frac{(1-\eta)^{k}}{n}\sum^n_{i=1}\Exp{\norm{v^0_i - \nabla f_i(x^0) }} + \sqrt{\eta}\sigma + \frac{\gamma}{\eta}L_0\exp\left(\gamma L_1\right)\notag\\
    %     && +\exp\left(\gamma L_1\right)\frac{\gamma L_1}{n}\sum^k_{t=0}(1-\eta)^{k-t}\sum^n_{i=1}\Exp{\norm{\nabla f_i(x^t)}},\notag
    % \end{eqnarray}
\end{lemma}
\begin{proof}
\revision{We prove the result \revision{using the} arguments similar to those \revision{given in the proof of} Theorem 1 \revision{from} \citet{cutkosky2020momentum}.
From the definition of $v_i^{k+1}$, we have the following recursion for any $k \geq 0$:
\begin{eqnarray*}
    v^{k+1}_i &=& (1-\eta)v^k_i +\eta\nabla f_i(x^{k+1};\xi^{k+1}_i) \\
    &=&\nabla f_i(x^{k+1}) + (1-\eta)(v^k_i -\nabla f_i(x^k)) + (1-\eta)(\nabla f_i(x^k)-\nabla f_i(x^{k+1}))\\
    && + \eta(\nabla f_i(x^{k+1};\xi^{k+1}_i) - \nabla f_i(x^{k+1})).
\end{eqnarray*}
Next, from the recursion of  $v^{k+1}_i$, we obtain the following recursion for $k \geq 0$:  
\begin{eqnarray}
    H^{k+1}_i &=& (1-\eta)H^{k}_i + (1-\eta)G^k_i + \eta U^{k+1}_i, \label{eq:H_i^k_recursion}
\end{eqnarray}    
where
\begin{equation*}
    U^{k+1}_i =\nabla f_i(x^{k+1};\xi^{k+1}_i) - \nabla f_i(x^{k+1}),\quad G^k_i = \nabla f_i(x^k)-\nabla f_i(x^{k+1}),\quad H^{k}_i= v^k_i -\nabla f_i(x^k),
\end{equation*}
\begin{equation*}
    U^{k+1} =\frac{1}{n}\sum^n_{i=1}U^{k+1}_i,\quad G^k =\frac{1}{n}\sum^n_{i=1}G^{k}_i ,\quad \text{and} \quad  H^{k}=\frac{1}{n}\sum^n_{i=1}H^{k}_i.
\end{equation*}
\revision{Unrolling the recursion for $H_i^k$, we derive}
\begin{align*}
    H^{k+1}_i 
    =& (1-\eta)^{k+1} H^0_i + \sum^{k}_{t=0}(1-\eta)^{k-t+1}G^t_i + \eta\sum^{k}_{t=0}(1-\eta)^{k-t}U^{t+1}_i.
\end{align*}
\revision{Averaging the above inequality, we get}
\begin{eqnarray*}
    H^{k+1} &=& (1-\eta)^{k+1} H^0 + \sum^{k}_{t=0}(1-\eta)^{k-t+1}G^t + \eta\sum^{k}_{t=0}(1-\eta)^{k-t}U^{t+1}.
\end{eqnarray*}
Next, taking the Euclidean norm, using the triangle inequality, and then taking the expectation, we obtain
\begin{eqnarray}
    \Exp{\norm{H^{k+1}}}  &\leq& (1-\eta)^{k+1} \Exp{\norm{H^0}} + \underbrace{\sum^{k}_{t=0}(1-\eta)^{k-t+1}\Exp{\norm{G^t}}}_{=: \gA_1} \notag \\
    && + \eta\underbrace{\Exp{\norm{\sum^{k}_{t=0}(1-\eta)^{k-t}U^{t+1}}}}_{=:\gA_2}.\label{eqn:ef21sgdm:exp_norm_H_k}
\end{eqnarray}
To bound $\Exp{\norm{H^{k+1}}}$, we need to bound the expectation of the last two terms. First, we bound term $\gA_1$. \revision{Using} the fact that $\norm{G^t} \leq \frac{1}{n}\sum_{i=1}^n \norm{G_i^t}$, and the definition of $G_i^t$, \revision{we obtain}
\begin{eqnarray*}
    \gA_1 
    &\leq& \frac{1}{n}\sum^n_{i=1}\sum^{k}_{t=0}(1-\eta)^{k-t+1}\Exp{\norm{\nabla f_i(x^t)-\nabla f_i(x^{t+1})}}\\
    &\overset{\eqref{eqn:LzeroLoneSmooth_ineq_0}}{\leq}& \frac{1}{n}\sum^n_{i=1}\sum^{k}_{t=0}(1-\eta)^{k-t+1}\Exp{L_0\exp\left(L_1\norm{x^{t+1}-x^t}\right)\norm{x^{t+1}-x^t}}\\
    && + \frac{1}{n}\sum^n_{i=1}\sum^{k}_{t=0}(1-\eta)^{k-t+1}\Exp{L_1\norm{\nabla f_i(x^t)}\exp\left(L_1\norm{x^{t+1}-x^t}\right)\norm{x^{t+1}-x^t}}\\
    &=& \sum^{k}_{t=0}(1-\eta)^{k-t+1} \gamma {\exp(\gamma L_1)} L_0 + \frac{L_1}{n}\sum^n_{i=1}\sum^{k}_{t=0}(1-\eta)^{k-t+1} \gamma \exp\left(\gamma L_1\right) \Exp{\norm{\nabla f_i(x^t)}}\\
    &\overset{\eqref{eq:sum_of_grad_norms_bound}}{\leq}& 2L_0 \gamma \exp\left(\gamma L_1\right)\sum^{k}_{t=0}(1-\eta)^{k-t+1} + 8L_1^2 \gamma \exp\left(\gamma L_1\right)\sum^{k}_{t=0}(1-\eta)^{k-t+1} \Exp{f(x^t) - f^{\inf}}\\
    && +  \frac{8L_1^2 \gamma \exp\left(\gamma L_1\right)}{n}\sum_{i=1}^n\sum^{k}_{t=0}(1-\eta)^{k-t+1} \left(f^{\inf} - f_i^{\inf}\right)\\
    &\leq& 2L_0 \gamma \exp\left(\gamma L_1\right)\sum^{\infty}_{t=0}(1-\eta)^{t} + 8L_1^2 \gamma \exp\left(\gamma L_1\right)\sum^{k}_{t=0}(1-\eta)^{k-t+1} \Exp{f(x^t) - f^{\inf}}\\
    && +  \frac{8L_1^2 \gamma \exp\left(\gamma L_1\right)}{n}\sum_{i=1}^n\left(f^{\inf} - f_i^{\inf}\right) \sum^{\infty}_{t=0}(1-\eta)^{t}\\
    &=& \frac{2L_0 \gamma \exp\left(\gamma L_1\right)}{\eta} + 8L_1^2 \gamma \exp\left(\gamma L_1\right)\sum^{k}_{t=0}(1-\eta)^{k-t+1} \Exp{f(x^t) - f^{\inf}}\\
    && +  \frac{8L_1^2 \gamma \exp\left(\gamma L_1\right)}{\eta n}\sum_{i=1}^n\left(f^{\inf} - f_i^{\inf}\right).
\end{eqnarray*}
Next, we bound term $\gA_2$. Jensen's inequality and the tower property of the conditional expectation imply
\begin{eqnarray*}
    \gA_2 &\leq& \sqrt{\Exp{\sqnorm{\sum^{k}_{t=0}(1-\eta)^{k-t}U^{t+1}}}} =  \sqrt{\sum^{k}_{t=0}(1-\eta)^{2(k-t)}\Exp{\sqnorm{U^{t+1}}}}.
\end{eqnarray*}
Moreover, due to independence of $\{\xi^t_i\}_{i=1}^n$, we have
\begin{eqnarray*}    
 \gA_2  &\leq& \sqrt{\sum^{k}_{t=0}\frac{(1-\eta)^{2(k-t)}}{n^2}\sum_{i=1}^n\Exp{\sqnorm{U_i^{t+1}}}} \overset{\eqref{eqn:bounded_variance}}{\leq} \sqrt{\sum^{k}_{t=0}(1-\eta)^{2(k-t)}\frac{\sigma^2}{n}}\\
    &\leq&\frac{\sigma}{\sqrt{n}}\sqrt{\sum^{\infty}_{t=0}(1-\eta)^{2t}} = \frac{\sigma}{\sqrt{n\eta(2-\eta)}} \overset{\eta\in[0,1]}{\leq} \frac{\sigma}{\sqrt{n\eta}}.
\end{eqnarray*}
Therefore, plugging the derived upper-bounds for $\gA_1$, and for $\gA_2$ into \eqref{eqn:ef21sgdm:exp_norm_H_k}, we obtain 
\begin{eqnarray*}
    \Exp{\norm{H^{k+1}}} &\leq& (1-\eta)^{k+1} \Exp{\norm{H^0}} + \frac{2L_0 \gamma \exp\left(\gamma L_1\right)}{\eta}\\
    && + 8L_1^2 \gamma \exp\left(\gamma L_1\right)\sum^{k}_{t=0}(1-\eta)^{k-t+1} \Exp{f(x^t) - f^{\inf}}\\
    && +  \frac{8L_1^2 \gamma \exp\left(\gamma L_1\right)}{\eta n}\sum_{i=1}^n\left(f^{\inf} - f_i^{\inf}\right) + \frac{\sqrt{\eta}\sigma}{\sqrt{n}},
\end{eqnarray*}
which is equivalent to \eqref{eq:vk_nablafxk_bound}.

To derive \eqref{eq:sum_vik_nablafixk_bound}, we make a step back to the recursion from \eqref{eq:H_i^k_recursion}, which implies
\begin{eqnarray}
    \frac{1}{n}\sum_{i=1}^n \Exp{\norm{H_i^{k+1}}} &\leq& \frac{1-\eta}{n}\sum_{i=1}^n \Exp{\norm{H_i^{k}}} + \underbrace{\frac{1-\eta}{n}\sum_{i=1}^n \Exp{\norm{G_i^{k}}}}_{=: \cB_1} \notag\\
    && + \underbrace{\frac{\eta}{n}\sum_{i=1}^n \Exp{\norm{U_i^{k+1}}}}_{=: \cB_2}. \label{eq:bound_for_sum_tecnical_ineq}
\end{eqnarray}
Next, we derive the upper bounds for $\cB_1$ and $\cB_2$. For $\cB_1$, we have
\begin{eqnarray*}
    \cB_1 &=& \frac{1-\eta}{n}\sum_{i=1}^n \Exp{\norm{\nabla f_i(x^k) - \nabla f_i(x^{k+1})}}\\
    &\overset{\eqref{eqn:LzeroLoneSmooth_ineq_0}}{\leq}& \frac{1-\eta}{n}\sum_{i=1}^n\Exp{(L_0 + L_1\norm{\nabla f_i(x^k)})\exp\left(L_1\norm{x^k - x^{k+1}}\right)\norm{x^k - x^{k+1}}}\\
    &=& (1-\eta)L_0\gamma\exp(\gamma L_1) +  \frac{(1-\eta)L_1\gamma\exp(\gamma L_1)}{n}\sum_{i=1}^n\Exp{\norm{\nabla f_i(x^k)}}\\
    &\overset{\eqref{eq:sum_of_grad_norms_bound}}{\leq}& 2(1-\eta)L_0\gamma\exp(\gamma L_1) + 8(1-\eta)L_1^2\gamma\exp(\gamma L_1)\Exp{f(x^k) - f^{\inf}}\\
    && + \frac{8(1-\eta)L_1^2\gamma\exp(\gamma L_1)}{n}\sum_{i=1}^n\left(f^{\inf} - f_i^{\inf}\right),
\end{eqnarray*}
and for $\cB_2$, we obtain
\begin{eqnarray*}
    \cB_2 &=& \frac{\eta}{n}\sum_{i=1}^n \Exp{\norm{\nabla f_i(x^{k+1};\xi_i^{k+1}) - \nabla f_i(x^{k+1})}} \overset{\eqref{eqn:bounded_variance}}{\leq} \eta\sigma.
\end{eqnarray*}
Plugging the derived upper bounds for $\cB_1$ and $\cB_2$ into \eqref{eq:bound_for_sum_tecnical_ineq} and using $1-\eta \leq 1$, we get
\begin{eqnarray*}
    \frac{1}{n}\sum_{i=1}^n \Exp{\norm{H_i^{k+1}}} &\leq& \frac{1-\eta}{n}\sum_{i=1}^n \Exp{\norm{H_i^{k}}} + 2L_0\gamma\exp(\gamma L_1)\\
    && + 8L_1^2\gamma\exp(\gamma L_1)\Exp{f(x^k) - f^{\inf}}\\
    && + \frac{8L_1^2\gamma\exp(\gamma L_1)}{n}\sum_{i=1}^n\left(f^{\inf} - f_i^{\inf}\right) +  \eta\sigma,
\end{eqnarray*}
which is equivalent to \eqref{eq:sum_vik_nablafixk_bound}.}
\end{proof}

\subsection{Proof of Theorem~\ref{thm:ef21_sgdm}}\label{app:sec:ef21_sgdm}

Now, we are ready to prove Theorem~\ref{thm:ef21_sgdm}. \revision{For convenience, we introduce new notation: 
\begin{gather}
    \delta^k \eqdef \Exp{f(x^k) - f^{\inf}},\quad A_{k} \eqdef \frac{1}{n}\sum^n_{i=1}\Exp{\norm{v^{k}_i - g^{k}_i}},\quad B_k \eqdef \Exp{\norm{v^{k} - \nabla f(x^k)}}, \notag\\
    C_k \eqdef \frac{1}{n}\sum_{i=1}^n \Exp{\norm{v_i^{k} - \nabla f_i(x^k)}},\quad \delta^{\inf} \eqdef \frac{1}{n}\sum_{i=1}^n(f^{\inf} - f_i^{\inf}). \notag
\end{gather}
Using the new notation and noticing that $\Exp{\|v^k - g^k\|} \leq A_k$, we rewrite the results of Lemmas~\ref{lem:descent_lemma_ef21_sgdm}, \ref{lem:compression_error_ef21_sgdm}, and \ref{lem:momentum_error_ef21_sgdm} as
\begin{eqnarray*}
    \delta^{k+1} &\leq& \left(1 + 4L_1^2\gamma^2\exp(L_1\gamma)\right)\delta^k + 2\gamma A_k + 2\gamma B_k - \gamma\Exp{\norm{\nabla f(x^k)}}\\
    &&+ \gamma^2 \exp(L_1\gamma)\left(L_0 + 4L_1^2 \delta^{\inf}\right),\\
    A_{k+1} &\leq& \sqrt{1 - \alpha} A_k + \eta\sqrt{1-\alpha}C_k + 8L_1^2\sqrt{1-\alpha}\eta\gamma\exp\left(\gamma L_1\right) \delta^k\\
    &&+ 2\sqrt{1-\alpha} \eta\gamma \exp\left(\gamma L_1\right)\left(L_0 + 4L_1^2 \delta^{\inf}\right) + \sqrt{1-\alpha}\eta\sigma,\\
    B_k &\leq& (1-\eta)^kB_0 + \frac{\sqrt{\eta}\sigma}{\sqrt{n}} + \frac{2 \gamma \exp\left(L_1\gamma\right)}{\eta}\left(L_0 + 4L_1^2 \delta^{\inf}\right)\\
    &&+ 8L_1^2\gamma \exp\left(L_1\gamma\right)\sum\limits_{t=0}^{k-1}(1-\eta)^{k-t}\delta^t,\\
    C_{k+1} &\leq& (1-\eta)C_k + 8L_1^2\gamma \exp\left(L_1\gamma\right)\delta^k + \eta\sigma + 2\gamma\exp\left(\gamma L_1\right)\left(L_0 + 4L_1^2 \delta^{\inf}\right).
\end{eqnarray*}
Moreover, since $\gamma = \frac{\gamma_0}{(K+1)^{\nicefrac{3}{4}}}$ with $\gamma_0 \leq \frac{1}{2L_1}$, we have $\exp(L_1\gamma) \leq \exp(L_1\gamma_0) \leq 2$ and the above inequalities can be further simplified as
\begin{eqnarray}
    \delta^{k+1} &\leq& \left(1 + 8L_1^2\gamma^2\right)\delta^k + 2\gamma A_k + 2\gamma B_k - \gamma\Exp{\norm{\nabla f(x^k)}} + 2\gamma^2 \left(L_0 + 4L_1^2 \delta^{\inf}\right), \label{eq:ef21_sgdm_main_ineq1}\\
    A_{k+1} &\leq& \sqrt{1 - \alpha} A_k + \eta\sqrt{1-\alpha}C_k + 16L_1^2\sqrt{1-\alpha}\eta\gamma \delta^k\notag \\
    &&+ 4\sqrt{1-\alpha} \eta\gamma \left(L_0 + 4L_1^2 \delta^{\inf}\right) + \sqrt{1-\alpha}\eta\sigma,\label{eq:ef21_sgdm_main_ineq2}\\
    B_k &\leq& (1-\eta)^kB_0 + \frac{\sqrt{\eta}\sigma}{\sqrt{n}} + \frac{4 \gamma }{\eta}\left(L_0 + 4L_1^2 \delta^{\inf}\right)+ 16L_1^2\gamma \sum\limits_{t=0}^{k-1}(1-\eta)^{k-t}\delta^t,\label{eq:ef21_sgdm_main_ineq3}\\
    C_{k+1} &\leq& (1-\eta)C_k + 16L_1^2\gamma \delta^k + \eta\sigma + 4\gamma\left(L_0 + 4L_1^2 \delta^{\inf}\right).\label{eq:ef21_sgdm_main_ineq4}
\end{eqnarray}
Next, we introduce the Lyapunov function $V_k$ defined for any $k \geq 0$ as
\begin{equation*}
    V_{k} = \delta^k + aA_k + cC_k,
\end{equation*}
where $a \eqdef \frac{2\gamma}{1 - \sqrt{1-\alpha}}$ and $c\eqdef a\sqrt{1-\alpha}$. Then, using \eqref{eq:ef21_sgdm_main_ineq1}, \eqref{eq:ef21_sgdm_main_ineq2}, \eqref{eq:ef21_sgdm_main_ineq4}, we get
\begin{eqnarray*}
    V_{k+1} &\leq& \left(1 + 8L_1^2\gamma^2\right)\delta^k + 2\gamma A_k + 2\gamma B_k - \gamma\Exp{\norm{\nabla f(x^k)}} + 2\gamma^2 \left(L_0 + 4L_1^2 \delta^{\inf}\right)\\
    &&+ a\left(\sqrt{1 - \alpha} A_k + \eta\sqrt{1-\alpha}C_k + 16L_1^2\sqrt{1-\alpha}\eta\gamma \delta^k\right)\\
    &&+ a\left(4\sqrt{1-\alpha} \eta\gamma \left(L_0 + 4L_1^2 \delta^{\inf}\right) + \sqrt{1-\alpha}\eta\sigma\right)\\
    &&+ c\left((1-\eta)C_k + 16L_1^2\gamma \delta^k + \eta\sigma + 4\gamma\left(L_0 + 4L_1^2 \delta^{\inf}\right)\right).
\end{eqnarray*}
To proceed, we rearrange the terms:
\begin{eqnarray*}
    V_{k+1} &\leq& \left(1 + 8L_1^2\gamma^2 + 16aL_1^2\sqrt{1-\alpha}\eta\gamma + 16cL_1^2\gamma\right)\delta^k + \left(\frac{2\gamma}{a} + \sqrt{1-\alpha} \right)aA_k\\
    && + \left(\frac{a\eta\sqrt{1-\alpha}}{c} + 1-\eta\right)cC_k + 2\gamma B_k - \gamma\Exp{\norm{\nabla f(x^k)}}\\
    && + \left(2\gamma^2 +  4a\sqrt{1-\alpha} \eta\gamma + 4c\gamma\right) \left(L_0 + 4L_1^2 \delta^{\inf}\right) + \eta\left(a\sqrt{1-\alpha} + c\right)\sigma\\
    &\overset{\overset{c = a\sqrt{1-\alpha},}{\eta\leq 1}}{\leq}& \left(1 + 8L_1^2\gamma^2 + 32aL_1^2\sqrt{1-\alpha}\gamma\right)\delta^k + \left(\frac{2\gamma}{a} + \sqrt{1-\alpha} \right)aA_k + cC_k\\
    && + 2\gamma B_k - \gamma\Exp{\norm{\nabla f(x^k)}}\\
    && + \left(2\gamma^2 +  8a\sqrt{1-\alpha} \gamma \right) \left(L_0 + 4L_1^2 \delta^{\inf}\right) + 2\eta a\sqrt{1-\alpha}\sigma.
\end{eqnarray*}
Since $a = \frac{2\gamma}{1 - \sqrt{1-\alpha}}$, we have $\frac{2\gamma}{a} + \sqrt{1-\alpha} = 1$ and
\begin{eqnarray*}
    V_{k+1} &\leq& \left(1 + 8L_1^2\gamma^2 + \frac{64L_1^2\gamma^2\sqrt{1-\alpha}}{1-\sqrt{1-\alpha}}\right)\delta^k + aA_k + cC_k + 2\gamma B_k - \gamma\Exp{\norm{\nabla f(x^k)}}\\
    && + \left(2\gamma^2 +  \frac{16\gamma^2\sqrt{1-\alpha}}{1-\sqrt{1-\alpha}} \right) \left(L_0 + 4L_1^2 \delta^{\inf}\right) + \frac{4\gamma\eta \sqrt{1-\alpha}\sigma}{1-\sqrt{1-\alpha}}\\
    &\leq& \left(1 + \frac{64L_1^2\gamma^2}{1-\sqrt{1-\alpha}}\right)V_k + 2\gamma B_k - \gamma\Exp{\norm{\nabla f(x^k)}}\\
    && + \frac{16\gamma^2\left(L_0 + 4L_1^2 \delta^{\inf}\right)}{1-\sqrt{1-\alpha}}  + \frac{4\gamma\eta \sqrt{1-\alpha}\sigma}{1-\sqrt{1-\alpha}}.
\end{eqnarray*}
Next, we bound $B_k$ using \eqref{eq:ef21_sgdm_main_ineq3} and $\delta^k \leq V_k$:
\begin{eqnarray*}
    V_{k+1} &\leq&  \left(1 + \frac{64L_1^2\gamma^2}{1-\sqrt{1-\alpha}}\right)V_k + 32L_1^2\gamma^2\sum\limits_{t=0}^{k-1}(1-\eta)^{k-t}V_t + 2\gamma (1-\eta)^k B_0 - \gamma\Exp{\norm{\nabla f(x^k)}}\\
    && + \left(\frac{16\gamma^2}{1-\sqrt{1-\alpha}} + \frac{8\gamma^2}{\eta}\right)\left(L_0 + 4L_1^2 \delta^{\inf}\right)  + \left(\frac{4\gamma\eta \sqrt{1-\alpha}}{1-\sqrt{1-\alpha}} + \frac{2\gamma\sqrt{\eta}}{\sqrt{n}}\right)\sigma.
\end{eqnarray*}
Summing up the above inequality with weights $\beta_k \eqdef \left(1 + \frac{64L_1^2\gamma^2}{1-\sqrt{1-\alpha}} + \frac{32L_1^2\gamma^2}{\eta}\right)^{-(k+1)}$ for $k = 0,\ldots,K$ and denoting $S_K \eqdef \sum_{k=0}^K\beta_k$ and $\beta_{-1}\eqdef 1$, we get
\begin{eqnarray*}
    \sum\limits_{k=0}^K\beta_k V_{k+1} &\leq& \sum\limits_{k=0}^K\left(1 + \frac{64L_1^2\gamma^2}{1-\sqrt{1-\alpha}}\right)\beta_k V_k + 32L_1^2\gamma^2\sum\limits_{k=0}^K\beta_k\sum\limits_{t=0}^{k-1}(1-\eta)^{k-t}V_t\\
    && + 2\gamma B_0\sum\limits_{k=0}^K(1-\eta)^k \beta_k - \gamma\sum\limits_{k=0}^K\beta_k \Exp{\norm{\nabla f(x^k)}}\\
    && + S_K\left(\frac{16\gamma^2}{1-\sqrt{1-\alpha}} + \frac{8\gamma^2}{\eta}\right)\left(L_0 + 4L_1^2 \delta^{\inf}\right)  + S_K\left(\frac{4\gamma\eta \sqrt{1-\alpha}}{1-\sqrt{1-\alpha}} + \frac{2\gamma\sqrt{\eta}}{\sqrt{n}}\right)\sigma.
\end{eqnarray*}
By definition of $\beta_k$, we have $\beta_k \leq \beta_{k-1}$ and, in particular, $\beta_k \leq 1$ for all $k\geq 0$. Using these inequalities, we continue the derivation as follows:
\begin{eqnarray*}
    \sum\limits_{k=0}^K\beta_k V_{k+1} &\leq& \sum\limits_{k=0}^K\left(1 + \frac{64L_1^2\gamma^2}{1-\sqrt{1-\alpha}}\right)\beta_k V_k + 32L_1^2\gamma^2\sum\limits_{k=0}^K\sum\limits_{t=0}^{k-1}(1-\eta)^{k-t}\beta_tV_t\\
    && + 2\gamma B_0\sum\limits_{k=0}^K(1-\eta)^k - \gamma\sum\limits_{k=0}^K\beta_k \Exp{\norm{\nabla f(x^k)}}\\
    && + S_K\left(\frac{16\gamma^2}{1-\sqrt{1-\alpha}} + \frac{8\gamma^2}{\eta}\right)\left(L_0 + 4L_1^2 \delta^{\inf}\right)  + S_K\left(\frac{4\gamma\eta \sqrt{1-\alpha}}{1-\sqrt{1-\alpha}} + \frac{2\gamma\sqrt{\eta}}{\sqrt{n}}\right)\sigma\\
    &\leq& \sum\limits_{k=0}^K\left(1 + \frac{64L_1^2\gamma^2}{1-\sqrt{1-\alpha}}\right)\beta_k V_k + 32L_1^2\gamma^2\left(\sum\limits_{t=0}^{\infty} (1-\eta)^t\right)\left(\sum\limits_{k=0}^K\beta_k V_k\right)\\
    && + 2\gamma B_0\sum\limits_{k=0}^{\infty}(1-\eta)^k - \gamma S_K \min\limits_{k=0,\ldots,K} \Exp{\norm{\nabla f(x^k)}}\\
    && + S_K\left(\frac{16\gamma^2}{1-\sqrt{1-\alpha}} + \frac{8\gamma^2}{\eta}\right)\left(L_0 + 4L_1^2 \delta^{\inf}\right)  + S_K\left(\frac{4\gamma\eta \sqrt{1-\alpha}}{1-\sqrt{1-\alpha}} + \frac{2\gamma\sqrt{\eta}}{\sqrt{n}}\right)\sigma\\
    &=& \sum\limits_{k=0}^K\underbrace{\left(1 + \frac{64L_1^2\gamma^2}{1-\sqrt{1-\alpha}} + \frac{32L_1^2\gamma^2}{\eta}\right)\beta_k}_{=\beta_{k-1}} V_k  + \frac{2\gamma B_0}{\eta} - \gamma S_K \min\limits_{k=0,\ldots,K} \Exp{\norm{\nabla f(x^k)}}\\
    && + S_K\left(\frac{16\gamma^2}{1-\sqrt{1-\alpha}} + \frac{8\gamma^2}{\eta}\right)\left(L_0 + 4L_1^2 \delta^{\inf}\right)  + S_K\left(\frac{4\gamma\eta \sqrt{1-\alpha}}{1-\sqrt{1-\alpha}} + \frac{2\gamma\sqrt{\eta}}{\sqrt{n}}\right)\sigma.
\end{eqnarray*}
Rearranging the terms and dividing both sides of the above inequality by $\gamma S_K$, we obtain
\begin{eqnarray}
    \min\limits_{k=0,\ldots,K} \Exp{\norm{\nabla f(x^k)}} &\leq& \frac{1}{\gamma S_K}\sum\limits_{k=0}^K\left(\beta_{k-1}V_k - \beta_kV_{k+1}\right) + \frac{2B_0}{\eta S_K}\notag\\
    && + \left(\frac{16\gamma}{1-\sqrt{1-\alpha}} + \frac{8\gamma}{\eta}\right)\left(L_0 + 4L_1^2 \delta^{\inf}\right)  + \left(\frac{4\eta \sqrt{1-\alpha}}{1-\sqrt{1-\alpha}} + \frac{2\sqrt{\eta}}{\sqrt{n}}\right)\sigma\notag\\
    &\leq& \frac{V_0}{\gamma S_K} + \frac{2B_0}{\eta S_K} + \left(\frac{16\gamma}{1-\sqrt{1-\alpha}} + \frac{8\gamma}{\eta}\right)\left(L_0 + 4L_1^2 \delta^{\inf}\right)\notag\\
    &&  + \left(\frac{4\eta \sqrt{1-\alpha}}{1-\sqrt{1-\alpha}} + \frac{2\sqrt{\eta}}{\sqrt{n}}\right)\sigma, \label{eq:pre_final_bound_EF21_SGDM}
\end{eqnarray}
where in the last inequality we use $V_{K+1} \geq 0$ and $\beta_{-1} = 1$. Next, we estimate $S_K$:
\begin{eqnarray}
    S_K &=& \sum_{k=0}^K \beta_k \geq (K+1)\beta_K = \frac{K+1}{\left(1 + \frac{64L_1^2\gamma^2}{1-\sqrt{1-\alpha}} + \frac{32L_1^2\gamma^2}{\eta}\right)^{K+1}}\notag\\
    &\geq& \frac{K+1}{\exp\left(\frac{64L_1^2\gamma^2 (K+1)}{1-\sqrt{1-\alpha}} + \frac{32L_1^2\gamma^2(K+1)}{\eta}\right)}. \label{eq:S_K_lower_bound}
\end{eqnarray}
Since $\eta = \frac{1}{(K+1)^{\nicefrac{1}{2}}}$ and $\gamma = \frac{\gamma_0}{(K+1)^{\nicefrac{3}{4}}}$ with $\gamma_0 \leq \frac{1}{16L_1}\min\left\{(K+1)^{\nicefrac{1}{2}}(1 - \sqrt{1-\alpha}), 1 \right\}$, we have $\frac{32L_1^2\gamma^2(K+1)}{\eta} \leq \frac{1}{4}$ and $\frac{64L_1^2\gamma^2 (K+1)}{1-\sqrt{1-\alpha}} \leq \frac{1}{4}$. Plugging these inequalities into \eqref{eq:S_K_lower_bound}, we get $S_K \geq \nicefrac{(K+1)}{\exp(\nicefrac{1}{2})} \geq \nicefrac{(K+1)}{2}$. Using this lower bound for $S_K$ and $\eta = \frac{1}{(K+1)^{\nicefrac{1}{2}}}$, $\gamma = \frac{\gamma_0}{(K+1)^{\nicefrac{3}{4}}}$ in \eqref{eq:pre_final_bound_EF21_SGDM} , we get
\begin{eqnarray}
    \min\limits_{k=0,\ldots,K} \Exp{\norm{\nabla f(x^k)}} &\leq& \frac{2V_0}{\gamma_0 (K+1)^{\nicefrac{1}{4}}} + \frac{4B_0}{(K+1)^{\nicefrac{1}{2}}}\notag\\
    && + \left(\frac{16\gamma_0}{(1-\sqrt{1-\alpha})(K+1)^{\nicefrac{3}{4}}} + \frac{8\gamma_0}{(K+1)^{\nicefrac{1}{4}}}\right)\left(L_0 + 4L_1^2 \delta^{\inf}\right)\notag\\
    &&  + \left(\frac{4\sqrt{1-\alpha}}{(1-\sqrt{1-\alpha})(K+1)^{\nicefrac{1}{2}}} + \frac{2}{\sqrt{n}(K+1)^{\nicefrac{1}{4}}}\right)\sigma. \notag
\end{eqnarray}
For the convenience, we define $C_\alpha \eqdef 1 - \sqrt{1-\alpha}$. Then, by definition of $V_0$, we have
\begin{equation*}
    \frac{2V_0}{\gamma_0 (K+1)^{\nicefrac{1}{4}}} = \frac{2\delta^0}{\gamma_0 (K+1)^{\nicefrac{1}{4}}} + \frac{2A_0}{C_\alpha(K+1)} + \frac{2(1-C_\alpha)C_0}{C_\alpha (K+1)}.
\end{equation*}
Moreover, since $g_i^{-1} = 0$ and $v_i^{-1} = \nabla f_i(x_i^0;\xi_{i}^0)$ for all $i = 1,\ldots, n$ with independent $\{\xi_i^0\}_{i=1}^n$, we have $v_i^0 = \nabla f_i(x_i^0;\xi_{i}^0)$ and $g_i^0 = \cC^0(\nabla f_i(x_i^0;\xi_{i}^0))$ for all $i = 1,\ldots,n$ and
\begin{eqnarray*}
    A_0 &=& \frac{1}{n}\sum_{i=1}^n\Exp{\norm{\nabla f_i(x_i^0;\xi_{i}^0) - \cC^0(\nabla f_i(x_i^0;\xi_{i}^0))}} \\
    &\overset{\eqref{eqn:contractive_comp}}{\leq}& \frac{\sqrt{1-\alpha}}{n}\sum_{i=1}^n\Exp{\norm{\nabla f_i(x_i^0;\xi_{i}^0) - \nabla f_i(x_i^0)}} \overset{\eqref{eqn:bounded_variance}}{\leq} (1-C_\alpha)\sigma,\\
    C_0 &=& \frac{1}{n}\sum_{i=1}^n\Exp{\norm{\nabla f_i(x_i^0;\xi_{i}^0) - \nabla f_i(x_i^0)}} \overset{\eqref{eqn:bounded_variance}}{\leq} \sigma,\\
    B_0 &=& \Exp{\norm{\frac{1}{n}\sum_{i=1}^n\left(\nabla f_i(x_i^0;\xi_{i}^0) - \nabla f_i(x_i^0)\right)}}\\
    &=&\sqrt{\frac{1}{n^2}\sum_{i=1}^n\Exp{\norm{\nabla f_i(x_i^0;\xi_{i}^0) - \nabla f_i(x_i^0)}^2}} \overset{\eqref{eqn:bounded_variance}}{\leq} \frac{\sigma}{\sqrt{n}}.
\end{eqnarray*}
Using these inequalities, we get
\begin{eqnarray}
    \min\limits_{k=0,\ldots,K} \Exp{\norm{\nabla f(x^k)}} &\leq& \frac{2\delta^0}{\gamma_0 (K+1)^{\nicefrac{1}{4}}} + \frac{2A_0}{C_\alpha(K+1)} + \frac{2(1-C_\alpha)C_0}{C_\alpha (K+1)} + \frac{4B_0}{(K+1)^{\nicefrac{1}{2}}}\notag\\
    && + \left(\frac{16\gamma_0}{C_\alpha(K+1)^{\nicefrac{3}{4}}} + \frac{8\gamma_0}{(K+1)^{\nicefrac{1}{4}}}\right)\left(L_0 + 4L_1^2 \delta^{\inf}\right)\notag\\
    && + \frac{4(1-C_\alpha)\sigma}{C_\alpha(K+1)^{\nicefrac{1}{2}}} + \frac{2\sigma}{\sqrt{n}(K+1)^{\nicefrac{1}{4}}} \notag\\
    &\leq& \frac{2\delta^0}{\gamma_0 (K+1)^{\nicefrac{1}{4}}} + \frac{4(1-C_\alpha)\sigma}{C_\alpha(K+1)} + \frac{4\sigma}{\sqrt{n}(K+1)^{\nicefrac{1}{2}}}\notag\\
    && + \left(\frac{16\gamma_0}{C_\alpha(K+1)^{\nicefrac{3}{4}}} + \frac{8\gamma_0}{(K+1)^{\nicefrac{1}{4}}}\right)\left(L_0 + 4L_1^2 \delta^{\inf}\right)\notag\\
    && + \frac{4(1-C_\alpha)\sigma}{C_\alpha(K+1)^{\nicefrac{1}{2}}} + \frac{2\sigma}{\sqrt{n}(K+1)^{\nicefrac{1}{4}}} \notag\\
    &\leq& \frac{2\delta^0}{\gamma_0 (K+1)^{\nicefrac{1}{4}}}  + \left(\frac{16\gamma_0}{C_\alpha(K+1)^{\nicefrac{3}{4}}} + \frac{8\gamma_0}{(K+1)^{\nicefrac{1}{4}}}\right)\left(L_0 + 4L_1^2 \delta^{\inf}\right)\notag\\
    && + \frac{8(1-C_\alpha)\sigma}{C_\alpha(K+1)^{\nicefrac{1}{2}}} + \frac{6\sigma}{\sqrt{n}(K+1)^{\nicefrac{1}{4}}}, \notag
\end{eqnarray}
which concludes the proof since $\frac{1-C_\alpha}{C_\alpha} \leq \frac{2\sqrt{1-\alpha}}{\alpha}$ and $\frac{1}{C_\alpha} \leq \frac{1}{\alpha}$}

\clearpage

\section{Additional Experimental Results}

In this section, we provide additional results for minimizing nonconvex polynomial functions, and for training the ResNet-20 model over the CIFAR-10 dataset.

\subsection{Minimization of Nonconvex Polynomial Functions}\label{app:sub:minimize_nonconvex_poly}

We ran \algname{||EF21||} and \algname{EF21} in a single-node setting ($n=1$) for solving the following problem:
%minimizing nonconvex $(L_0, L_1)$-smooth problems. To address this, we consider the following problem:
\begin{align}\label{eqn:simple_polynomials}
\min _{x \in \mathbb{R}^d}\biggl\{f(x)\eqdef\underbrace{\sum_{i=1}^d a_i x_i^4}_{=: g(x)}+\underbrace{\lambda \sum_{i=1}^d \frac{x_i^2}{1+x_i^2}}_{=: h(x)}\biggl\},
\end{align}
where $a_i > 0$, $i = 1, \ldots, d$, $\lambda > 0$.

Let us show that $f(x)$ is non-convex (for the specific choice of $a_i$) and $(L_0, L_1)$-smooth. First, we  prove that $f(x)$ is non-convex. Indeed,
\begin{align*}
\nabla^{2} f(x) & =\nabla^{2} g(x) + \nabla^{2} h(x) \\
& =12 \operatorname{diag}\left\{a_1 x_{1}^{2}, \ldots, a_d x_{d}^{2}\right\} + 2 \lambda \operatorname{diag}\biggl\{\frac{1-3 x_{1}^{2}}{\left(1+x_{1}^{2}\right)^{3}}, \ldots, \frac{1-3 x_{d}^{2}}{\left(1+x_{d}^{2}\right)^{3}}\biggl\},
\end{align*} 
is not positive definite matrix if we choose $a_i = \frac{\lambda}{24}$, $x_i = \pm 1$ for $i = 1, \ldots, d$. 

Second, we  find $L_0, L_1 > 0$ such that $$\norm{\nabla^{2} f(x)} \leq L_{0}+L_{1}\norm{\nabla f(x)}, \quad \forall x\in\mathbb{R}^d.$$ This condition is equivalent to Assumption~\ref{assum:LzeroLoneSmooth} (generalized smoothness) with $L_0,L_1$ \citep[Theorem 1]{chen2023generalized}.
Let us fix some $L_1 > 0$ and choose $L_0 = \frac{9 \lambda d^2}{2 L_1^2} + 2 \lambda$. Since $\nabla^2 h(x) \preccurlyeq 2 \lambda I$,
\begin{align*}
\norm{\nabla^{2} f(x)}=\norm{\nabla^{2} g(x)+\nabla^{2} h(x)} & \leq \norm{\nabla^{2} g(x)}+\norm{\nabla^{2} h(x)} \\
&\leq 12 \sqrt{a_1^2 x_{1}^{4}+\ldots+a_d^2 x_{d}^{4}}+2 \lambda\\
&\leq 12\left(a_1 x_{1}^{2}+\ldots+a_d x_{d}^{2}\right) + 2\lambda.
\end{align*}
Also, notice that
\begin{align*}
\norm{\nabla f(x)}=\norm{\nabla g(x)+\nabla h(x)}&=\sqrt{\left(4a_1x_{1}^{2} + \frac{2 \lambda}{(1+x_1^2)^2}\right)^2 x_1^2 +\ldots+\left(4a_dx_d^{2} + \frac{2 \lambda}{(1+x_d^2)^2}\right)^2 x_d^2}\\
& \geq 4\sqrt{a_1^2 x_1^6+\ldots+a_d^2 x_d^6} \\
& \overset{(*)}{\geq} \frac{4}{\sqrt{d}}\left(a_1 \left|x_1\right|^3+\ldots+a_d \left|x_d\right|^{3}\right),
\end{align*}
where (*) results from the fact that $\norm{x}_1 \leq \sqrt{d}\norm{x}$ for $x\in\R^d$. 
Our goal is to show that 
\begin{align*}
12\left(a_1 x_{1}^{2}+\ldots+a_d x_{d}^{2}\right) \leq \tilde{L}_{0}+\frac{4 L_1}{\sqrt{d}}\left(a_1 \left|x_{1}\right|^{3}+\ldots+a_d \left|x_d\right|^{3}\right), \quad \tilde{L}_{0}=L_{0}-2 \lambda.
\end{align*}
To show this, we consider two cases: if $\left|x_{i}\right| \leq \frac{3 \sqrt{d}}{L_{1}}$, and otherwise. 

\begin{enumerate}
    \item If $\left|x_{i}\right| \leq \frac{3 \sqrt{d}}{L_{1}}$ for all $i=1, \ldots, d$, then $12a_i x_i^2 \leq \frac{108 a_i d}{L_{1}^2}$. Thus,  $12\left(a_1 x_{1}^{2}+\ldots+a_d x_d^{2}\right) \leq \frac{108 \lambda d^{2}}{24 L_{1}^2}=\tilde{L}_{0}$.
    \item  If $\left|x_{j}\right| > \frac{3 \sqrt{d}}{L_{1}}$ for some $j = 1, \ldots, d$, then $12 a_j x_j^2 < \frac{4 L_1}{\sqrt{d}}a_j \left|x_j\right|^3$, and the sum of the remaining terms (such that $\left|x_{i}\right| \leq \frac{3 \sqrt{d}}{L_{1}}$) in $12\left(a_1 x_{1}^{2}+\ldots+a_d x_{d}^{2}\right)$ can be upper bounded by $\tilde{L}_{0}$.
\end{enumerate}
In conclusion, $f(x)$ is $(L_0, L_1)$-smooth, where $L_1$ is any positive constant and $L_0 = \frac{9 \lambda d^2}{2 L_1^2} + 2 \lambda$.

Additionally, we can show that under certain additional constraints, $f(x)$ is $L$-smooth with $L=\frac{ \lambda \sqrt{d} D^2}{2} + 2 \lambda$. If $\left|x_i\right| \leq D$ for all $i=1, \ldots, d$, then 
\begin{align*}
\left\|\nabla^{2} f(x)\right\| \leq 12 \sqrt{a_1^2 x_{1}^{4}+\ldots+a_d^2 x_{d}^{4}}+2 \lambda \leq \frac{ \lambda \sqrt{d} D^2}{2} + 2 \lambda = L,
\end{align*}
 In the experiments, we estimate $D$ based on the initial point $x^0\in\R^d$.

%\sarit{Fix the issue by adding theory for single-node EF21.}

In the following experiments, we used a top-$k$ sparsifier with $k = 1$ and $\alpha = k/d$, setting $d = 4$, $L_1 = \{1, 4, 8\}$, and $L_0 = 4$ (adjusting $\lambda$ to maintain a constant $L_0$). The initial values $x^0$ were drawn from a normal distribution, $x_i^0 \sim \mathcal{N}(20, 1)$ for $i = 1, \ldots, d$, with $D$ estimated as 20. 
For EF21, we set $\gamma_k = \frac{1}{L + L \sqrt{\frac{\beta}{\theta}}}$, using $\theta = 1 - \sqrt{1 - \alpha}$ and $\beta = \frac{1 - \alpha}{1 - \sqrt{1 - \alpha}}$, according to Theorem 1 of \cite{richtarik2021ef21}. 
For \algname{||EF21||}, we chose $\gamma_k = \frac{1}{2c_1}$ with $c_1 = \frac{L_1}{2} + 2\frac{\sqrt{1 - \alpha} L_1}{1 - \sqrt{1 - \alpha}}$ from Theorem~\ref{thm:normalized_ef21_singlenode}, and $\gamma_k = \frac{\gamma_0}{\sqrt{K + 1}}$ with $\gamma_0 > 0$, as specified in Theorem~\ref{thm:ef21} with $n=1$.

%for EF21,  

%For EF21, we select $\gamma = \frac{1}{L + L \sqrt{\frac{\beta}{\theta}}}$ (Theorem 1 of \cite{richtarik2021ef21}), where $\theta=1-\sqrt{1-\alpha}$, $\beta=\frac{1-\alpha}{1-\sqrt{1-\alpha}}$, and for  \algname{||EF21||} we choose $\gamma_k = \gamma = \frac{1}{2c_1}$ (Theorem~\ref{thm:single_ef21_v2_tradL}) and $\gamma_k = \gamma = \frac{\gamma_0}{\sqrt{K+1}}$ (Theorem~\ref{thm:ef21_v2}), where $c_1 = \frac{L_1}{2} + 2\frac{\sqrt{1-\alpha} L_1}{1-\sqrt{1-\alpha}}$, $\gamma_0>0$.

\paragraph{The impact of $\gamma_0$ and $K$ on the convergence of \algname{||EF21||}.}
First, we investigate the impact of $\gamma_0$ and $K$ on the convergence of \algname{||EF21||}.
%Before we move on to the comparison between \algname{EF21} and  \algname{||EF21||}, let us explore how to choose $\gamma_0$ and $K$ parameters for  \algname{||EF21||}.  To address this, 
We evaluated $\gamma_0$ from the set $\{0.1, 1, 10\}$, and plotted the histogram representing the number of iterations required to achieve the target accuracy of $\norm{\nabla f(x)}^2 < \epsilon$ with $\epsilon = 10^{-4}$, using the stepsize rule $\gamma = \frac{\gamma_0}{\sqrt{K+1}}$. 
For each $\gamma_0$, we determined $K$ as the minimum number of iterations required to achieve the desired accuracy, found through a grid search with step sizes of 500 for $\gamma_0 = {1, 10}$ and $5000$ for $\gamma_0 = 0.1$.
%In all the experiments, we choose $K$ as the smallest number of iterations required to achieve the desired accuracy (we perform a grid search with the stepsize of $500$ for $\gamma_0=\{1, 10\}$ and $5000$ for $\gamma_0=0.1$ to find the minimum number $K$).
\begin{figure}[h]
    \centering
    \includegraphics[width=0.5\textwidth]{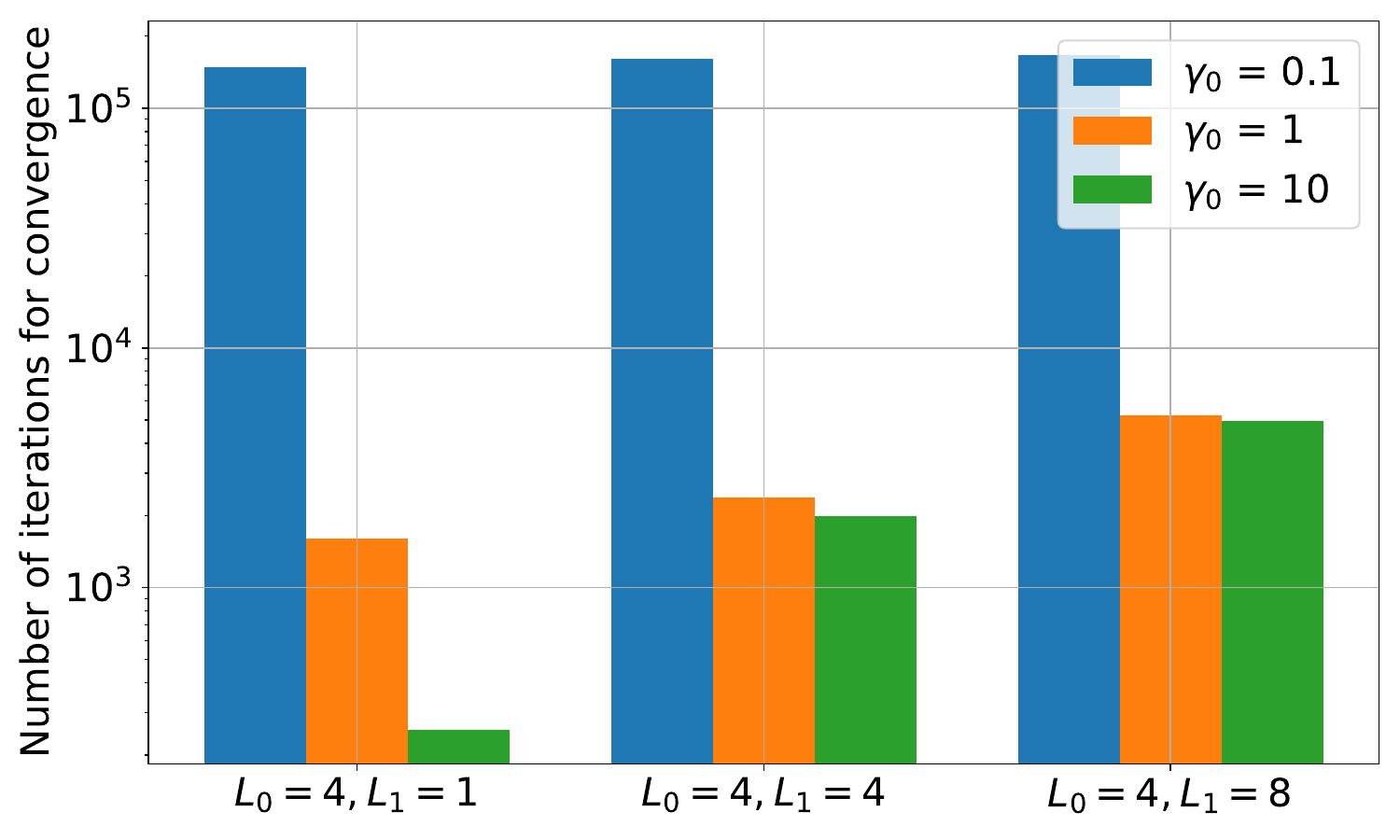}
    \caption{Number of iterations required to achieve the desired accuracy, $\norm{\nabla f(x)}^2 < \epsilon$, $\epsilon = 10^{-4}$, using \algname{||EF21||}  with $\gamma = \frac{\gamma_0}{\sqrt{K+1}}$ for different values of $L_0$ and $L_1$.}
    \label{fig:histogram}
\end{figure}
From Figure~\ref{fig:histogram}, for small values of $\gamma_0$, such as $0.1$,  significantly more iterations are required to reach convergence compared to $\gamma_0$ values of $1$ and $10$, which show similar performance (with the exception of the $L_0=4$, $L_1=1$ case, where  $\gamma_0=10$ converges faster). Based on this observation, we use $\gamma_0 = 1$ in all subsequent experiments and adjust only $K$ to achieve convergence,  identifying the minimum number of iterations needed to reach the target accuracy through a grid search with a step size of $500$. 
%(we select $K$ as the smallest number of iterations required to achieve the desired accuracy by performing a grid search with the stepsize of $500$).

\paragraph{Comparisons between \algname{EF21} and \algname{||EF21||}.}
Next, we evaluate the performance of \algname{EF21} and \algname{||EF21||} for a fixed \( L_0 = 4 \) and varying \( L_1 \) values of \{1, 4, 8\}. From Figure~\ref{fig:exp1}, \algname{||EF21||}, regardless of the chosen stepsize \(\gamma\), achieves the desired accuracy \(\norm{\nabla f(x)}^2 < \epsilon\) with \(\epsilon = 10^{-4}\) faster than \algname{EF21}. Initially, however, \algname{EF21} converges more quickly, likely because \algname{||EF21||} employs normalized gradients, which can be slower at the start due to the large gradients when the initial point is far from the stationary point. Moreover, as \( L_1 \) increases, both methods show slower convergence.

\subsection{ResNet20 Training over CIFAR-10}

% \sarit{Add plots that are not in the main pages. }

We included additional experimental results from running \algname{EF21} and \algname{||EF21||} for training the ResNet20 model over the CIFAR-10 dataset. The parameter details were set to be the same as those in Section~\ref{subsec:resnet_20}, with the exception that we vary $k=0.01d, 0.5d$ for a top-$k$ sparsifier. From Figures~\ref{fig:plot_exp_DL_2} and~\ref{fig:plot_exp_DL_3}, \algname{||EF21||} attains a higher accuracy improvement than \algname{EF21}, across different sparsification levels $k$.

% \begin{itemize}
%     \item We can say quickly that we have additional results when running the experiments under different $k$ for top-$k$ sparisifiers. But we maintain the same parameter settings, as described in the main text.
%     \item Key observations:  \algname{||EF21||} produces higher training and validation accuracies than EF21, under the same stepsizes. $10$\% improvement for $k=0.01d$, just like for $k=0.1d$. Furthermore,  \algname{||EF21||} also achieves the faster convergence in the gradient norm than \algname{EF21} for $k=0.05d$.  
% \end{itemize}

\begin{figure}[h]
    \centering
    \includegraphics[width=\textwidth]{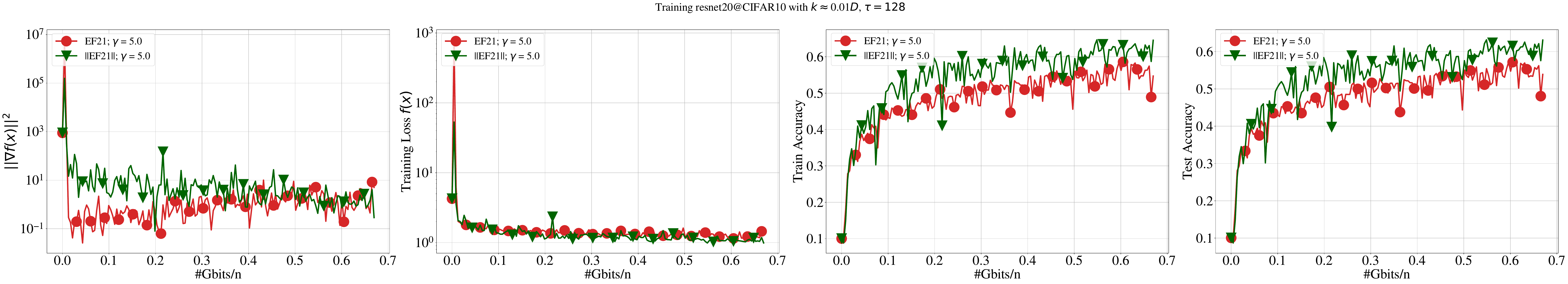}
    \caption{ResNet20 training on CIFAR-10 by using \algname{EF21} and \algname{||EF21||} under the same stepsize $\gamma=5$ and $k=0.01d$ for a  top-$k$ sparsifier.}
    %Image Classification experiment: Comparing \algname{EF21} and  \algname{||EF21||} under the same stepsize $\gamma=5$ during training ResNet20 on CIFAR-10. Using top-k compressor with $k = 0.01d$.}
    \label{fig:plot_exp_DL_2}
\end{figure}

\begin{figure}[h]
    \centering
    \includegraphics[width=\textwidth]{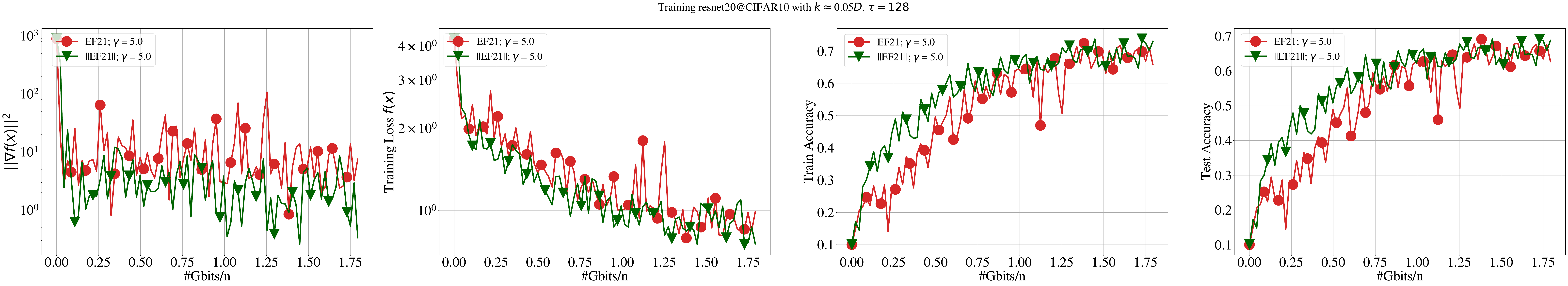}
    \caption{ResNet20 training on CIFAR-10 by using \algname{EF21} and \algname{||EF21||} under the same stepsize $\gamma=5$ and $k=0.05d$ for a  top-$k$ sparsifier.}
    %Image Classification experiment: Comparing \algname{EF21} and  \algname{||EF21||} under the same stepsize $\gamma=5$ during training ResNet20 on CIFAR-10. Using top-k compressor with $k = 0.01d$.}
    \label{fig:plot_exp_DL_3}
\end{figure}

%\subsection{ResNet-20 Training}\label{subsec:app:resnet20}

%{
%\color{red}
%All the experiments were run from our source code, to which the link can be found by \href{https://anonymous.4open.science/r/error-feedback-generalized-smoothness-9DDC}{error-feedback-generalized-smooth}.
%}

\clearpage

\section{Omitted Proof for Smoothness Parameters of Logistic Regression}\label{app:smooth_logistic}

In this section, we prove the generalized smoothness parameters $L_0,L_1$ for logistic regression problems with a nonconvex regularizer, which are the following problems 
\begin{align*}
\min _{x \in \mathbb{R}^d}\biggl\{f(x)\eqdef\frac{1}{n}\sum_{i=1}^n f_i(x)\eqdef \frac{1}{n}\sum_{i=1}^n \underbrace{\log(1+\exp(-b_i a_i^T x))}_{=:\tilde f_i(x)} + \underbrace{\lambda \sum_{j=1}^d \frac{x_j^2}{1+x_j^2}}_{=:h(x)}\biggl\},
\end{align*}
where $a_i \in \R^d$ is the $i^{\rm th}$ feature vector of matrix $A$ with its class label $b_i \in \{-1,1\}$, $\lambda>0$. 

First, we can prove that $f(x)$ is $L$-smooth with $L= \frac{1}{4n} \|A\|^2 + 2 \lambda$, and that each $f_i(x)$ is $\hat L_i$-smooth with $\hat L_i=\frac{1}{4} \|a_i\|^2 + 2 \lambda$.

Next, we show that each $f_i(x)$ is  generalized smooth with $L_0=2\lambda+\lambda \sqrt{d} \max_i \norm{a_i}$ and $L_1 = \max_i \norm{a_i}$, when the Hessian exists. 
%Notice that $f(x)$ is nonconvex, $L$-smooth for $L=\frac{1}{4n} \lambda_{\max }\left(A^{\top} A\right) + 2 \lambda = \frac{1}{4n} \|A\|^2 + 2 \lambda$, and each $f_i(x)$ is $\tilde L_i$-smooth for $\tilde L_i=\frac{1}{4} \lambda_{\max }\left(a_i a_i^{\top}\right) + 2 \lambda=\frac{1}{4} \|a_i\|^2 + 2 \lambda$. We can show that each $f_i$ is $(L_0,L_1)$-smooth for $L_0=2\lambda+\lambda \sqrt{d} \max_i \sqnorm{a_i}$ and $L_1 = \max_i \norm{a_i}$. Indeed, b
By the fact that
\begin{eqnarray*}
    \nabla \tilde f_i(x) = - \frac{\exp(-b_i a_i^Tx)}{1+\exp(-b_i a_i^T x)} b_i a_i, \quad \text{and} \quad \nabla^2 \tilde f_i(x) = \frac{\exp(-b_i a_i^Tx)}{(1+\exp(-b_i a_i^T x))^2} b_i^2 a_i a_i^T, 
\end{eqnarray*}
we have 
\begin{eqnarray}
    \norm{\nabla^2 \tilde f_i(x)} 
    & \overset{b_i \in \{-1,1\} }{=} & \frac{\exp(-b_i a_i^Tx)}{(1+\exp(-b_i a_i^T x))^2} \lambda_{\max}(a_i a_i^T)  \nonumber \\
    & = & \frac{\exp(-b_i a_i^Tx)}{(1+\exp(-b_i a_i^T x))^2} \sqnorm{a_i}  \nonumber \\
    & = & \frac{\norm{a_i}}{1+\exp(-b_ia_i^Tx)} \norm{\nabla \tilde f_i(x)} \nonumber \\
    & \leq & \norm{a_i}\norm{\nabla \tilde f_i(x)} \label{eq:hession_bound}.
\end{eqnarray}
After adding the nonconvex regularizer $h(x)$, we can show the following inequalities:
\begin{eqnarray}
    \norm{\nabla^2 f_i(x)} 
    &\leq& \norm{\nabla^2 \tilde f_i(x)}+\norm{\nabla^2 h(x)} \nonumber\\
    &\leq& \norm{\nabla^2 \tilde f_i(x)} + 2\lambda \label{eq:norm_ineq},
\end{eqnarray}
and
\begin{eqnarray}
    \norm{\nabla f_i(x)}
    \geq \norm{\nabla \tilde f_i(x)} - \norm{\nabla h(x)} & = & \norm{\nabla \tilde f_i(x)} - \sqrt{ \left( \frac{2\lambda x_1}{(1+x_1^2)^2} \right)^2 + \ldots + \left( \frac{2\lambda x_d}{(1+x_d^2)^2} \right)^2 } \nonumber \\
    & \geq & \norm{\nabla \tilde f_i(x)}  - \sqrt{ \lambda ^2 + \ldots + \lambda ^2 } \nonumber \\
    &=& \norm{\nabla \tilde f_i(x)} - \lambda \sqrt{d}. \label{eq:triangular_ineq} 
\end{eqnarray}

%\sarit{Artem 4. fixed. We should have the inequality.}
By combining inequalities \eqref{eq:hession_bound}, \eqref{eq:norm_ineq}, and \eqref{eq:triangular_ineq}, we obtain
\begin{eqnarray*}
    \norm{\nabla^2 f_i(x)}
    & \leq & \norm{\nabla^2 \tilde f_i(x)} + 2\lambda \\
    & \leq & \norm{a_i} \norm{\nabla \tilde f_i(x)} + 2\lambda \\ 
    & \leq & 2\lambda + \lambda \sqrt{d}\revision{\|a_i\|} + \norm{a_i} \norm{\nabla f_i(x)}. 
    %\leq \underbrace{2\lambda + \lambda \sqrt{d} \norm{a_i}}_{\leq L_0} + \underbrace{\norm{a_i}}_{\leq L_1} \norm{\nabla f_i(x)} \leq L_0 + L_1 \norm{\nabla f_i(x)},
\end{eqnarray*}
In conclusion, $\norm{\nabla^2 f_i(x)} \leq L_0 + L_1 \norm{\nabla f_i(x)}$ with $L_0 \leq 2\lambda + \lambda \sqrt{d} \revision{\|a_i\|}$, and $L_1 \leq \norm{a_i}$. This condition is equivalent to Assumption~\ref{assum:LzeroLoneSmooth} (generalized smoothness) with $L_0,L_1$ \citep[Theorem 1]{chen2023generalized}.

\end{document}

%% file: math_commands.tex
\usepackage{amsmath,amsfonts,bm}
\usepackage{nicefrac}
\usepackage{wrapfig}

\def\eqref#1{(\ref{#1})}
\def\1{\bm{1}}

\DeclareMathAlphabet{\mathsfit}{\encodingdefault}{\sfdefault}{m}{sl}
\SetMathAlphabet{\mathsfit}{bold}{\encodingdefault}{\sfdefault}{bx}{n}

\def\gA{{\mathcal{A}}}

\def\gC{{\mathcal{C}}}

\def\gF{{\mathcal{F}}}

\def\gO{{\mathcal{O}}}

%% file: macros_arXiv.tex
\usepackage{multirow}
\usepackage[cp1250]{inputenc}
\usepackage[T1]{fontenc}
\usepackage{calligra}
\usepackage{layout}

\usepackage{amsmath}
\usepackage{amsthm}
\usepackage{amssymb}
\usepackage{amsfonts}
\usepackage{mathtools}

\usepackage{url}
\usepackage{color}
\usepackage{graphicx}
\usepackage{algorithmic}
\usepackage{algorithm}
\usepackage{verbatim}
\usepackage{appendix}

\newcommand{\norm}[1]{\left\| #1 \right\|}
\newcommand{\sqnorm}[1]{\left\| #1 \right\|^2}
\newcommand{\inp}[2]{\left\langle#1,#2\right\rangle} % inner product

\newcommand{\cB}{\mathcal{B}}
\newcommand{\cC}{\mathcal{C}}
\newcommand{\cD}{\mathcal{D}}

\newcommand{\cO}{\mathcal{O}}

\usepackage[dvipsnames]{xcolor}
\usepackage{xspace}
\newcommand{\algname}[1]{{\color{ForestGreen}\small\sf#1}\xspace}
\newcommand{\algnamesmall}[1]{{\color{ForestGreen}\scriptsize\sf#1}\xspace}

\newcommand{\del}[1]{}

\newcommand{\R}{\mathbb{R}} % reals
\newcommand{\eqdef}{:=} %\vcentcolon

\newcommand{\Exp}[1]{{\rm E}\left[#1\right]}
\newcommand{\ExpCond}[2]{{\rm E}\left[\left.#1\right\vert#2\right]}

\newtheorem{assumption}{Assumption}
\newtheorem{lemma}{Lemma}

\newtheorem{theorem}{Theorem}

\theoremstyle{plain}

\theoremstyle{definition}

\usepackage[many]{tcolorbox}

\tcolorboxenvironment{theorem}{
	colback=gray!10,
	boxrule=0pt,
	boxsep=1pt,
	left=2pt,right=2pt,top=2pt,bottom=2pt,
	oversize=2pt,
	sharp corners,
	before skip=\topsep,
	after skip=\topsep,
}

\tcolorboxenvironment{lemma}{
	colback=gray!10,
	boxrule=0pt,
	boxsep=1pt,
	left=2pt,right=2pt,top=2pt,bottom=2pt,
	oversize=2pt,
	sharp corners,
	before skip=\topsep,
	after skip=\topsep,
}

\tcolorboxenvironment{algorithms}{
	colback=gray!10,
	boxrule=0pt,
	boxsep=1pt,
	left=2pt,right=2pt,top=2pt,bottom=2pt,
	oversize=2pt,
	sharp corners,
	before skip=\topsep,
	after skip=\topsep,
}

\definecolor{BrickRed}{RGB}{143,20,2}
\definecolor{PineGreen}{RGB}{0,110,51}

\newcommand{\revision}[1]{{\color{black}#1}}